\documentclass[reqno,twoside,11pt]{article}

\usepackage{jmlr2e}

\usepackage{amsthm}
\usepackage{mathtools}
\usepackage{amsmath}

\usepackage{hyperref}
\usepackage{pdfsync}
\usepackage{dsfont}
\usepackage{color}
\usepackage{esint}
\usepackage{enumerate}

\usepackage{braket}
\usepackage{bm}
\usepackage{graphicx}
\usepackage{caption}
\usepackage{xcolor}
\usepackage{ulem}
\numberwithin{equation}{section}
\newtheorem{assumption}{Assumption}
\usepackage{tikz}
\usepackage{graphicx}

\usepackage{subfig}

\usepackage{algorithm}
\usepackage{algpseudocode}
\usepackage{algorithmicx}

\usepackage{gensymb}

\definecolor{Blue}{rgb}{0,0,1}
\definecolor{Red}{rgb}{1,0,0}
\def\blue{\color{black}} 

\newcommand{\new}[1]{{\blue #1}}

\newcommand{\veps}{\varepsilon}

\newcommand{\K}{\mathcal{K}}
\newcommand{\X}{\mathcal{X}}

\newcommand{\vol}{\mathrm{vol}}
\newcommand{\divergence}{\mathrm{div}}

\newcommand{\dist}{\mathrm{dist}}

\newcommand{\gL}{\mathcal{L}^{\veps_+,\veps_-}}

\newcommand{\gbk}{{b}^{\veps_+,\veps_-}_k}

\newcommand{\T}{\mathcal{T}}

\newcommand{\tT}{\widetilde{T}}
\newcommand{\I}{\mathcal{I}}

\newcommand{\R}{\mathbb{R}}

\renewcommand{\L}{\Delta}

\newcommand{\M}{\mathcal{M}}

\definecolor{mygreen}{rgb}{0.1,0.75,0.2}

\newcommand{\nc}{\normalcolor}

\newcommand{\C}{\mathcal{C}}
\newcommand{\eps}{\varepsilon}

\newcommand{\N}{\mathbb{N}}

\newtheorem{theorem}{Theorem}[section]

\newtheorem{corollary}[theorem]{Corollary}

\newtheorem{definition}[theorem]{Definition}

\newtheorem{lemma}[theorem]{Lemma}

\newtheorem{proposition}[theorem]{Proposition}
\newtheorem{remark}[theorem]{Remark}

\renewenvironment{proof}{\par\noindent{\bf Proof\ }}{\hfill\BlackBox\\[2mm]}

\definecolor{darkred}{rgb}{0.6,0.1,0.1}
\definecolor{darkgreen}{rgb}{0.1,0.6,0.1}
\definecolor{darkblue}{rgb}{0.1,0.1,0.6}

\ShortHeadings{Spectral Analysis for Graph-based MMC}{Garc\'ia Trillos, He and Li}
\firstpageno{1}

\begin{document}
    \title{Large sample spectral analysis of graph-based multi-manifold clustering}

    \author{\name Nicol\'as Garc\'ia Trillos \email garciatrillo@wisc.edu \\
       \addr Department of Statistics\\
       University of Wisconsin\\
       Madison, Wisconsin, USA
       \AND
       \name Pengfei He \email hepengf1@msu.edu \\
       \addr Department of Statistics and Probability\\
       Michigan State University\\
       East Lansing, MI, USA
       \AND
       \name Chenghui Li \email cli539@wisc.edu \\
       \addr Department of Statistics\\
       University of Wisconsin\\
       Madison, Wisconsin, USA}
    
    \editor{}
  
	\maketitle
	\begin{abstract}
		In this work we study statistical properties of graph-based algorithms for multi-manifold clustering (MMC). In MMC the goal is to retrieve the multi-manifold structure underlying a given Euclidean data set when this one is assumed to be obtained by sampling a distribution on a union of manifolds $\M = \M_1 \cup\dots  \cup \M_N$ that may intersect with each other and that may have different dimensions. We investigate sufficient conditions that similarity graphs on data sets must satisfy in order for their corresponding graph Laplacians to capture the right geometric information to solve the MMC problem. Precisely, we provide high probability error bounds for the spectral approximation of a tensorized Laplacian on $\M$ with a suitable graph Laplacian built from the observations; the recovered tensorized Laplacian contains all geometric information of all the individual underlying manifolds. We provide an example of a family of similarity graphs, which we call annular proximity graphs with angle constraints, satisfying these sufficient conditions. We contrast our family of graphs with other constructions in the literature based on the alignment of tangent planes. Extensive numerical experiments expand the insights that our theory provides on the MMC problem.    
	\end{abstract}
	\begin{keywords}
  multi-manifold clustering, graph Laplacian, spectral convergence, manifold learning, discrete to continuum limit.
    \end{keywords}
	\section{Introduction}
	
	\footnotetext[1]{All authors contributed equally to this work. Their names are listed in alphabetical order by last name.}

	In this work we study the problem of \textit{multi-manifold clustering} (MMC) from the perspective of spectral geometry. Multi-manifold clustering is the task of identifying the structure of multiple manifolds that underlie an observed data set $X=\{ x_1, \dots, x_n\}$, its main challenge being that in general the underlying manifolds may be non-linear, may intersect with each other, and may have different dimensions (see Figures \ref{fig: two spheres}-\ref{fig: Path Algorithm with epsilon+,epsilon- graph} and Figures \ref{fig: 2 Clusters}-\ref{fig: 5 Clusters} for some illustrations). While spectral methods for learning have been analyzed by several authors throughout the past two decades in settings as varied as unsupervised, semi-supervised, and supervised learning, less is known about their theoretical guarantees for the specific multi-manifold clustering problem. We analyze MMC algorithms that are based on the construction of suitable similarity graph representations for the data and in turn on the spectra of their associated graph Laplacians. We provide statistical error guarantees  for the identification of the underlying manifolds as well as for the recovery of their individual geometry.


	\begin{figure}[htbp]
		\par\medskip
		\centering
		\begin{minipage}[t]{0.30\textwidth}
			\centering
			\includegraphics[width=5.5cm]{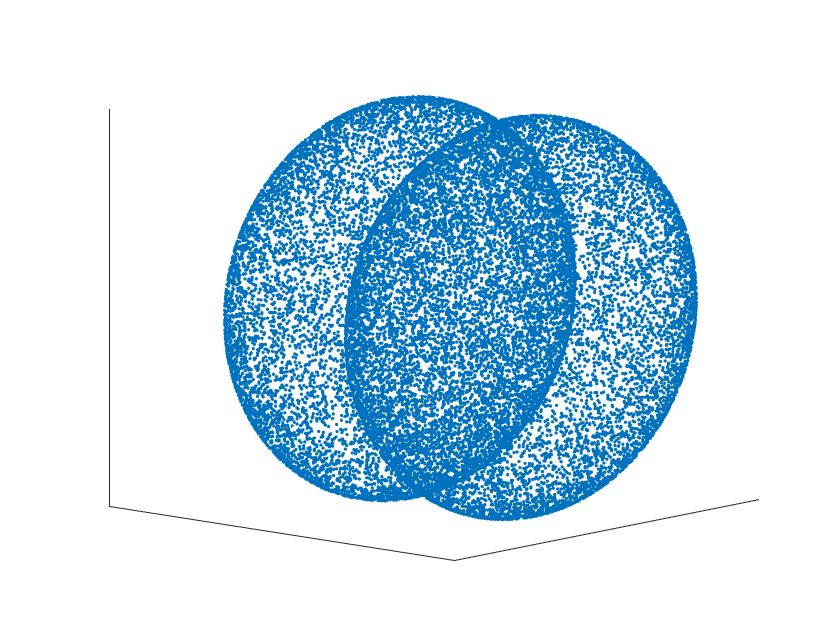}
			\caption{}
			\label{fig: two spheres}
		\end{minipage}
		\begin{minipage}[t]{0.30\textwidth}
			\centering
			\includegraphics[width=5.5cm]{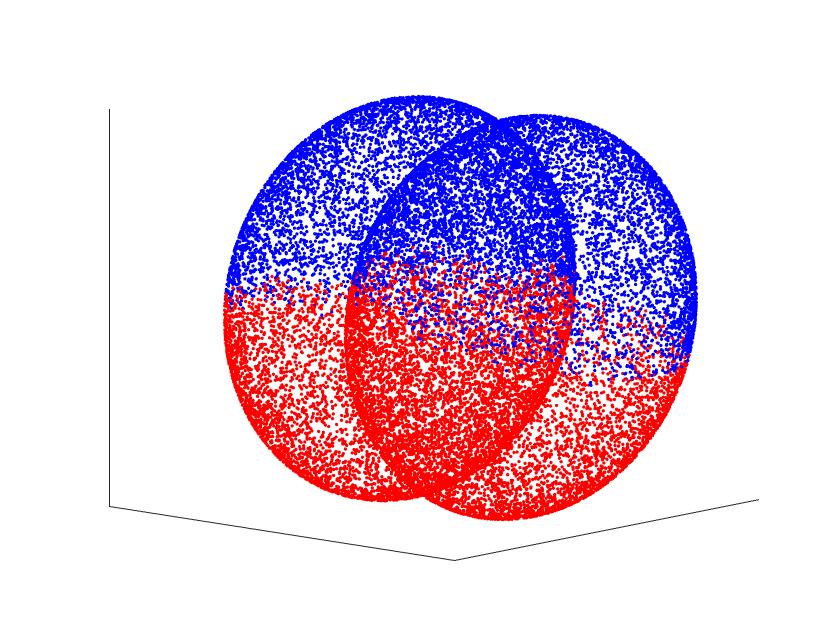}
			\caption{}
			\label{fig: Path Algorithm with epsilon graph}
		\end{minipage}
		\begin{minipage}[t]{0.30\textwidth}
			\centering
			\includegraphics[width=5.5cm]{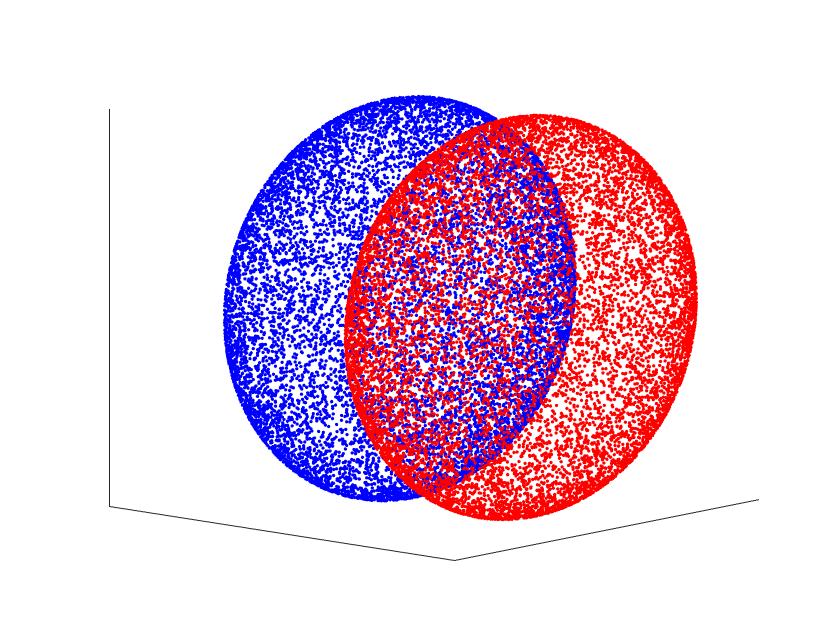}
			\caption{}
			\label{fig: Path Algorithm with epsilon+,epsilon- graph}
		\end{minipage}	
		\caption*{Figure \ref{fig: two spheres} illustrates two intersecting ellipsoids (two dimensional). A \textit{good} multi-manifold clustering algorithm must identify the two underlying ellipsoids. Figure \ref{fig: Path Algorithm with epsilon graph} and Figure \ref{fig: Path Algorithm with epsilon+,epsilon- graph} show the spectral clustering with $k$-NN graph and annular proximity graph with angle constraint, respectively; see section \ref{sec:GraphConstruct}.}
		\label{fig:illustration}
	\end{figure}

	
	As for most spectral approaches to clustering, we are interested in studying spectral properties of graph Laplacian operators of the form
	\begin{equation}
	\Delta_n u(x_i) := \sum \omega_{ij} ( u(x_i) - u(x_j)), \quad x_i \in X.
	\label{eqn:GraphLaplacian1}
	\end{equation}
	Here, the $\omega_{ij}$ are appropriately defined symmetric weights that in general depend on the proximity of points $x_i, x_j$, and, importantly,  on a mechanism that detects when points belong to different manifolds even if lying close to each other. Once the graph Laplacian is constructed, we follow the spectral clustering algorithm: the first $N$ eigenvectors  of $\Delta_n$ (denoted $\psi_1, \dots, \psi_N$) are used to build an embedding of the data set $X$ into $\R^N$:
	\[ x_i \in X \longmapsto  \left( \begin{matrix} \psi_1(x_i) \\ \vdots \\ \psi_N(x_i) \end{matrix}  \right) \in \R^N.\]
	In turn, with the aid of a simple clustering algorithm such as $k$-means the embedded data set is clustered. A successful algorithm will produce clusters that are in agreement with the different manifolds underlying the data set.
	
	As can be imagined, the success of spectral clustering when applied to MMC problems relies strongly on the specific similarity weights $\omega_{ij}$ that determine the graph Laplacian, its eigenvectors, and ultimately the partitioning of the data. In the literature, authors have considered different types of mechanisms to discriminate points that lie on different manifolds. Some strategies include the use of local tangent planes from data 
	\citet{arias2017spectral,goh2007segmenting,elhamifar2011sparse,wang2011spectral} (whose angles are compared), and the construction of paths between different points (e.g. geodesics) that are  considered admissible if they do not exhibit sudden turns  (effectively imposing a curvature constraint) \citet{BABAEIAN2015118}. All these methods are inspired by heuristics that are meaningful at the continuum level (i.e. the infinite data setting) and use second order geometric information to detect the different intersecting manifolds. While these heuristics provide practical insights, in general they do not guarantee the success of the employed methodologies for MMC at the finite sample level. Part of the motivation for this work is precisely to establish a more concrete and mathematically precise link between the heuristic motivation at the continuum level and the actual methodologies that are used in practice. It is worth highlighting that widely known graph constructions such as $\veps$-proximity graphs or $k$-NN graphs  used for standard data clustering tasks (typically aimed at detecting bottle-necks in data sets) are in general not suitable for MMC. To illustrate this, take for example Figure \ref{fig: Path Algorithm with epsilon graph}. There, we have used a $k$-NN graph to build a graph Laplacian whose first non-trivial eigenvector has been used to obtain the partition illustrated in the figure; as can be observed, from the geometry induced by the $k$-NN graph we are unable to distinguish the two underlying ellipsoids.  \nc

	\nc
	
	\medskip
	
	To start making the results presented in this paper more precise, let us suppose that the data set $X$ is obtained by sampling a distribution $\mu$ supported on a set $\M$ of the form
	\begin{equation}
	\M = \M_1 \cup \dots \cup \M_N,
	\label{eqn:M}
	\end{equation}
	where the $\M_l$ are smooth compact connected manifolds with no boundary that for the moment are assumed to have the same dimension $m$; the manifolds $\M_l$ may have nonempty pairwise intersections, but these  are assumed to have measure zero relative to the volume forms of each of the manifolds involved. The distribution $\mu$ is assumed to be a mixture model taking the form
	\[d \mu= w_1\rho_1 d\vol_{\M_1}   + \dots +w_N \rho_N d \vol_{\M_N},\]
	for smooth density functions $\rho_l: \M_l \rightarrow \R$ and positive weights $w_i$ that add to one; henceforth we use $d\vol_{\M_l}$ to denote integration with respect to the Riemannian volume form associated to $\M_l$. A \textit{tensorized Laplacian} $\Delta_\M$ acting on functions $f$ on $\M$ (which will be written as $f= (f_1, \dots, f_N)$, where $f_l: \M_l \rightarrow \R$) can be defined according to
	\begin{equation}
	\Delta_\M f := ( w_1\Delta_{\M_1} f_1, \dots, w_N \Delta_{\M_N} f_{N}  ), 
	\label{eqn:LaplacianContinuum}
	\end{equation}  
	where $\Delta_{\M_l}$ is a Laplacian operator mapping regular enough functions $f_l : \M_l \rightarrow \R$ into functions  $\Delta_\M f_l : \M_l \rightarrow \R$ according to
	\[\Delta_{\M_l} f_l = - \frac{1}{\rho_l} \divergence_{\M_l}\left(   \rho_l^2 \nabla_{\M_l} f_l  \right).\]
	In other words, the operator $\Delta_\M$ acts in a coordinatewise fashion, effectively treating each manifold $\M_i$ independently.  It is then straightforward to show that eigenfunctions of $\Delta_\M$ are spanned by functions of the form
	\[(0, \dots, f_l, \dots, 0) \]
	for some $l$, where $f_l$ is an eigenfunction of $\Delta_{\M_l}$. This means that the spectrum of $\Delta_\M$ splits the geometries of the $\M_l$. In particular, the different $\M_l$ can be detected by retrieving the eigenfunctions with zero eigenvalue.

	Our first main results (\textbf{Theorem \ref{Rate of convergence for eigenvalues} and Theorem \ref{convergence rate for eigenvectors}}) state that, provided that the weights $\omega_{ij}$ defining the graph Laplacian operator $\Delta_n$ in \eqref{eqn:GraphLaplacian1} satisfy two conditions referred to as \textit{full inner connectivity} and \textit{sparse outer connectivity}, the eigenvalues (appropriately scaled) and eigenvectors of $\Delta_n$ approximate the eigenvalues and eigenfunctions of the tensorized Laplacian $\Delta_\M$; we obtain high probability quantitative bounds for the error of this approximation. The bottom line is that our results imply that the spectral methods studied here are guaranteed, at least for large enough $n$, to recover the underlying multi-manifold structure of the data; see Figure \ref{fig: Path Algorithm with epsilon+,epsilon- graph} for an illustration. Our work extends the growing literature of works that study the connection between graph Laplacians on data sets and their continuum analogues. This literature, which we review in section \ref{sec:specclust}, has mostly focused on the smooth setting where multiple intersecting manifolds are not allowed.

	In our second main result (\textbf{Theorem \ref{thm:MixedDimensions}}), we present some results for the case when the dimensions of the manifolds $\M_i$ do not agree. In this more general setting, the spectrum of the graph Laplacian $\Delta_n$ does not recover the tensorized geometry captured by $\Delta_\M$ as introduced earlier, but rather, only the tensorized geometry of the manifolds with the \textit{largest} dimension, effectively quotienting out the geometric information of manifolds with dimension strictly smaller than the maximum dimension. 
	

	\blue After presenting our general results, we move on to discussing specific examples of graph constructions that satisfy the full inner connectivity and sparse outer connectivity conditions. In particular, we discuss a family of annular proximity graphs with angle constraints (see section \ref{sec:GraphConstruct}) that we show satisfies the desired connectivity conditions. In the final section of the paper, we present some insights into the behavior of this graph construction and its ability to tackle the MMC problem in concrete numerical examples, as well as present a performance comparison with other existing spectral-based MMC approaches.

	\nc

	%
	%
	%
	%
	%
	%

	\subsection{ Related work}
	\label{sec:OtherWorks}

	In this section we provide an overview of some related works that study spectral clustering and its connection with manifold learning, as well as other works that study the specific multi-manifold clustering problem.

	\subsubsection{Spectral clustering and manifold learning}
	\label{sec:specclust}

	In the past two decades, several authors have attempted to establish precise connections between operators such as graph Laplacians built from random data and analogous differential operators defined at the continuum level. To make this connection mathematically precise, one can assume that the data are sampled from some distribution supported on a certain geometric object $\M$. In the setting where $\M$ is a smooth compact manifold embedded in $\R^d$ that has no boundary, several authors have studied the connection between $\veps$-graph-based Laplacians and weighted versions of Laplace Beltrami operators on $\M$. For pointwise consistency results we refer the reader to \citet{singer2006graph,hein2005graphs,hein2007graph,belkin2005towards,ting2010analysis,GK}). 
	Regarding \textit{spectral convergence} of graph Laplacians, a notion of convergence that is relevant for spectral clustering, the regime $n \rightarrow \infty$ and $\veps$ constant is studied in \citet{vLBeBo08} and also 
	in \citet{SinWu13}. The latter analyzes connection Laplacians, which are operators acting on vector fields as opposed to functions. Works that have studied regimes where $\veps$ is allowed to decay to zero include \citet{Shi2015,BIK,trillos2019error,Lu2019GraphAT,calder2019improved,DunsonWuWu,WormellReich}. The mathematical theory around graph Laplacians in the smooth manifold setting has developed considerably and even regularity estimates of graph Laplacian eigenvectors are now available  (see \citet{CalderGTLewicka}). 
	
	In the setting of a smooth compact manifold $\M$ \textit{with} boundary, graph Laplacians are seen to behave differently around the manifold's boundary than in their interior. This has been observed in works like \citet{BerryVaughn,BoundaryWuWu}, which study this setting and obtain expansions for graph Laplacians that hold all the way up to the boundary. Earlier works such as \citet{trillos2018variational} use variational methods to provide spectral asymptotic consistency results in this setting but don't obtain convergence rates nor describe the behavior of graph Laplacians close to the boundary. The work \citet{Lu2019GraphAT} provides rates for spectral convergence in the setting of manifolds with boundary and also considers the case where $\M$ is of the form \eqref{eqn:M}. However, in contrast to what we do here, the aim in \citet{Lu2019GraphAT} is not to analyze graph constructions that guarantee the recovery of the multi-manifold structure of the data, focusing instead on analyzing intrinsic proximity graphs to the union of the intersecting manifolds. Our analysis shares aspects and ideas with this and some of the other works previously mentioned, but to fulfill our goals we must introduce new constructions and estimates not currently available. In addition, to the best of our knowledge, we are the first to present an analysis of the full spectrum of graph Laplacians when data points are supported on a union of intersecting manifolds that have \textit{different} dimensions. \blue Previous work \citet{arias2011clustering} had analyzed the null space of a graph Laplacian when the generators (manifolds), although potentially of different dimensions, were assumed to be separated from each other.  \nc
	
	From a methodological perspective, it is also worth highlighting several other works that have studied the use of metrics different from the Euclidean one to build proximity graphs for clustering and other unsupervised learning tasks. The idea in those papers is to use the modified metrics to improve the performance of spectral clustering when applied to data sets with some special geometric structure. Examples include:  \citet{10.1145/2729977,10.1145/321356.321357,10.1007/978-3-642-19867-0_17,10.5555/646596.698112,10.1016/j.patcog.2007.04.010,little2017path,McKenzie2022}. In a sense, our approach in this paper is in line with the general perspective taken in the previously mentioned works, only that in our case we have a different geometric structure in mind, i.e., we consider multiple intersecting manifolds.

	%
	%

	\subsubsection{Multi-manifold clustering}
	In contrast to the graph constructions which are analyzed in most of the works mentioned in section \ref{sec:specclust} (i.e. standard $\veps$-graphs and $k$-NN graphs), graph constructions for multi-manifold clustering must incorporate a mechanism to discriminate between points that lie on different manifolds.  One such mechanism relies on the approximation of approximate tangent planes around every point. Pairs of nearby points are then endowed high weights whenever their corresponding tangent planes are aligned, as proposed in \citet{arias2017spectral}. The recovery of tangent planes from data is a problem that has been studied theoretically in papers such as \citet{aamari2018stability} (see also references within).
	The  methodology proposed in \citet{SinWu13}, which uses a connection Laplacian, can be considered as a MMC algorithm since it also uses tangent plane information to inform the affinity between points.  The LLMC algorithm from \citet{goh2007segmenting} is also based on locally fitting planes to points and their nearest neighbors. Sparse Manifold Clustering and Embedding(SMCE) in \citet{elhamifar2011sparse} implicitly attempts to recover tangent planes too; a sparse representation of points in a neighborhood is sought via a local $l^1$ optimization problem.  Another work that considers affinities based on local tangent planes is \citet{wang2011spectral}. In section \ref{Section: Plane Algorithm with epsilon_+,epsilon_- graph setting} we will discuss some properties of the tangent plane based graphs and their effect on spectral clustering for MMC.
	\blue 
\citet{wang2014riemannian} also consider estimating tangent planes to solve the MMC problem, but now in a generalized setting where the ambient space is a curved manifold and not $\R^d$.
	\nc
	
	At a high level, all multi-manifold clustering algorithms use curvature information to detect pairs of points that, while close to each other, lie on different manifolds. Measuring the difference of tangent planes is one way to capture curvature, but there are alternative ways. For example, works like \citet{chen2009spectral,chen2009foundations} use the notion of polar curvature between collections of points to define an algorithm known as spectral curvature clustering (SCC). In \citet{chen2009foundations} the authors present some theoretical analysis of SCC in the setting where the data are  sampled from multiple flats with the same dimension. \blue In \citet{arias2011spectral}, a localized spectral curvature clustering algorithm is proposed to find local curvature information by constructing similarity graphs that are obtained by aggregating certain alignment score for a collection of data tuples of high enough order. This method is computationally too intensive given that it requires to consider tuples of order larger than the dimension of the manifolds. Besides, from a theoretical perspective, the method scales very poorly with the dimensionality of the underlying manifolds, and the authors indicate that it can only solve the MMC problem in the setting of intersecting curves, i.e. 1d manifolds.
	\nc
	
	%
	Curvature can also be captured by measuring how quickly paths turn as proposed in \citet{BABAEIAN2015118}. Our graph construction from section \ref{sec:GraphConstruct} is inspired by the one proposed in \citet{BABAEIAN2015118}, but with some important differences that we will motivate and explain throughout the paper. These differences, in particular, allow us to provide a comprehensive theoretical analysis and provide theoretical guarantees for the success of our algorithms.
	
	%
	%
	%
	%
	%
	%

	%
	
	\blue 
	To wrap up this brief literature review, it is worth mentioning a special setting where the manifolds $\M_l$ are linear subspaces of the ambient space. In that case, the multi-manifold clustering problem reduces to \textit{subspace clustering} (SubC), a problem that has received considerable attention in the past decades due to its multiple applications in tasks such as image segmentation, motion segmentation, and image representation (see \citet{Vidal2010ATO}). Many algorithms in SubC rely strongly on the assumed global flat structure of the data and on the fact that the origin is known to lie on the intersection of the spaces. Unfortunately, these approaches can not be used directly for a general multi-manifold clustering task, so we will not discuss them in more detail. Conversely, while it is possible to use general MMC approaches to solve SubC problems, it is clear that the performance of general MMC methods will in general be far from satisfactory when compared to the performance of SubC approaches, which actively target the subspace structure of the manifolds. We refer the reader interested in the SubC problem to the following list of papers and their references: \citet{boult1991factorization,10.1007/11744085_8,DBLP:journals/corr/abs-1010-3460,park2014greedy,DBLP:journals/corr/abs-1202-4002,990968,DBLP:journals/corr/abs-1202-4002,990968,elhamifar2009sparse,elhamifar2010clustering,oswal2018scalable,10.5555/3104322.3104407}.
	\nc

	\subsection{Contributions and outline}

	
	
	
	
	

	We summarize our contributions as follows:

	\begin{itemize}
		\item We analyze graph Laplacians on families of proximity graphs when the nodes of the graphs are random data points that are supported on a union of unknown \textit{intersecting} manifolds. The manifolds may all have \textit{different} dimensions.  
		\item We introduce two sufficient conditions that similarity graphs must satisfy in order to recover, from a graph Laplacian operator, the geometric information (as contained in the spectrum of weighted Laplace-Beltrami operators) of the individual smooth manifolds underlying the data set. These conditions are referred to as \textit{full inner connectivity} and \textit{sparse outer connectivity}.

		\item We introduce and analyze \textit{annular} proximity graphs and their effect on multi-manifold clustering. These are simple extensions of $\veps$-proximity graphs that nonetheless can be shown to be, theoretically and numerically, better than the vanilla $\veps$-graphs for multi-manifold clustering.
		
		
		\item We analyze a family of \textit{annular proximity graphs with angle constraints}. This family is shown to satisfy the full inner connectivity and sparse outer connectivity conditions when their parameters are tuned appropriately. We contrast this construction with other constructions such as those based on local PCA, which in general do not satisfy the full inner connectivity condition.
		
		\item Through numerical examples and some heuristic computations, we provide further insights into the use of spectral methods for multi-manifold clustering.
	\end{itemize}

	The rest of the paper is organized as follows. Our theoretical framework is presented in section \ref{section: setup}, where we formalize the setting for the multi-manifold clustering problem, introduce the definitions of sparsely outer connected and fully inner connected similarity graphs, and state our main theoretical results. In our first results, the ones in section \ref{sec:SameDim}, we assume that all underlying manifolds have the same dimension, and in section \ref{subsection: Different dimensions} we extend to settings where the dimensions of the underlying manifolds can be different. In section \ref{sec:GraphConstruct} we discuss an example of a graph construction that satisfies the full inner connectivity and sparse outer connectivity conditions. In section \ref{section: discussion} we present a series of numerical experiments whose goal is to illustrate the theory developed throughout the paper and highlight some drawbacks of the MMC methods discussed in the paper. In Appendix \ref{Section: main results} we present the proofs of all the results from sections \ref{sec:SameDim} and \ref{subsection: Different dimensions}. 
	
	\section{Set up and main results}\label{section: setup}
	
	Let $ \{ \M_l \}_{l=1}^N$ be a collection of $N$ smooth, compact manifolds without boundary embedded in $\R^d$. We denote by $m_l$ the dimension of manifold $\M_l$ and $m= \max_{l=1, \dots, N} \{  m_l \}$. Let $\M$ be the union:
	\[ \M := \M_1 \cup \dots \cup \M_N. \]
	Let $X= \{x_1, \dots, x_n\}$ be i.i.d. samples from a distribution $\mu$ on $\M$ of the form:
	\begin{equation}
	d\mu = \sum_{l=1}^N w_l \rho_l(x)d \vol_{\M_l}(x),\ \  \text{where}\   w_l >0 \nc , \quad \sum_{l=1}^N w_l =1. 
	\label{eqn:DefMu}
	\end{equation} 
	In the above, for each $l$, $d\vol_{\M_l}$ is used to denote integration with respect to the Riemannian volume form associated to the manifold $\M_l$, and the probability density $\rho_l: \M_l \rightarrow \R$ is assumed to be $C^2(\M_l)$ and satisfy 
	\[    \frac{1}{c_\rho} \leq \rho_l(x) \leq c_\rho, \quad \forall l=1, \dots, N\]
	for some positive constant $c_\rho>1$. We use $\mu_l$ to denote the probability measure $\rho_l d\vol_{\M_l}$. Notice that from \eqref{eqn:DefMu} it follows that the number of data points $n_l$ in manifold $\M_l$ is with very high probability within \new{the interval $[w_l n- t , w_l n + t]$} for some tolerance level $t$ at least in the order of $\sqrt{n}$.

	While we will not require the manifolds $\M_l$ to be separated from each other in a distance sense (i.e., we allow manifolds to intersect with each other), we will assume that they are sufficiently ``well separated" in an angular sense that we specify below and \blue that we illustrate in Figure \ref{fig:assumption1}.\nc

	\begin{assumption}
		\label{assump:WellSeparated}
		For every $l,k$ we assume:
		
		\begin{enumerate}
			\item The intersection $\M_{lk}:=\mathcal{M}_l \cap \mathcal{M}_k$ is either the empty set or a smooth manifold of dimension $m_{kl}$ satisfying $0 \leq m_{kl } <\min \{m_l, m_k\}$. In particular, $\M_{lk}$ is of measure zero according to $\vol_{\M_l}$ and $\vol_{\M_k}$. 
			
			\item For every point $x$ in $\mathcal{M}_l\cap \mathcal{M}_k$ we have:
			\begin{equation}
			\label{eqn:AngleConstraint}
			\sup_{v\in \mathcal{T}_x\mathcal{M}_{lk}^{\perp_l}, \widetilde{v}\in \mathcal{T}_x\mathcal{M}_{lk}^{\perp_k}} |\angle(v,\widetilde{v})-\frac{\pi}{2}|\leq\beta,
			\end{equation}
			for some fixed $\beta$ strictly smaller than $\frac{\pi}{2}$. In the above, $\angle(v,\widetilde{v})$ denotes the angle between vectors $v, \tilde v$ (recall that all manifolds are embedded in the ambient space $\R^d$), and $\T_{x}\M_{lk}^{\perp_{l}}$ denotes the orthogonal complement of $\T_{x}\M_{lk}$ in $\T_x \M_l$, and $\T_{x}\M_{lk}^{\perp_{k}}$ is defined analogously.
		\end{enumerate}
	\end{assumption}
	
		\begin{figure}
		\centering
		\includegraphics[scale=0.4]{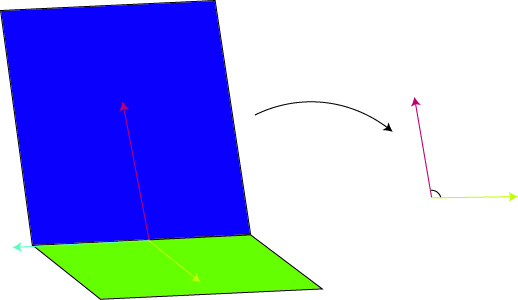}
		\put(-80,10){$\T\M_{l}$}
		\put(-235,20){$\T\M_{kl}$}
		\put(-240,90){$\T\M_{k}$}
		\put(10,40){$\T\M_{kl}^{\perp l}$}
		\put(-65,102){$\T\M_{kl}^{\perp k}$}
		\caption{\blue $\T\M_{kl}^{\perp k}$ is the orthogonal complement of $\T\M_{kl}$ in $\T\M_k$. $\T\M_{kl}^{\perp l}$ is defined analogously. The second part in Assumption \ref{assump:WellSeparated} is better satisfied when the angle between these spaces is close to ninety degrees.\nc  }		\label{fig:assumption1}
	    \end{figure}

	In the above, and in the remainder, we use $\mathcal{T}_x{\mathcal{M}_l}$ to denote the tangent plane to $\M_l$ at the point $x\in \M_l$; also, we use $\T\M_l$ to denote the tangent bundle associated to $\M_l$. Notice that the second condition in Assumption \ref{assump:WellSeparated} states that if two manifolds $\M_l$ and $\M_k$ do intersect, they do so in a non-tangential way; see Figure \ref{fig:assumption1} below.

	%

	\medskip

	\nc

	\subsection{Fully inner connected, and sparsely outer connected graphs}\label{Section:Conditions on algorithms for the recovery of tensorized spectrum from data clouds}


	We endow the data set $X$ with a weighted graph structure $(X, \omega)$, where the weights $\omega$ are specified by the data. In this section we present the definitions of fully inner connected and sparsely outer connected graphs. The notion of full inner connectivity depends on a prespecified family of base proximity graphs that we introduce next. 
	\begin{definition}
		Given $0\leq \veps_- < \veps_+$ and data points $x_i,x_j$, we define their $\veps_+, \veps_-$-weight as:
		\[  \omega_{ij}^{\veps_+, \veps_-}:= \begin{cases}  1 \quad \text{ if } \veps_- \leq |x_i-x_j| \leq \veps_+ \\ 0 \quad  \text{otherwise}.   \end{cases}  \]	
		We use $\omega_{ij}^{\veps}$ as shorthand notation for $\omega_{ij}^{\veps,0}$. Notice that with this definition we have the identity:
		\begin{equation*}
		\omega^{\veps_+,\veps_-}_{yx}=\omega^{\veps_+}_{yx}-\omega^{\veps_-}_{yx}.
		\end{equation*}
	\end{definition}
	\begin{remark}
	\label{rem:RemarkAnnular}
	The above definition extends the notion of $\veps$-proximity graph and in principle allows pairs of points that are too close to each other to have zero weight. While in the literature this annular proximity graphs have not been given any attention, we will see later on that this extended notion is convenient from  qualitative and quantitative points of view for the MMC problem; \blue see the discussion surrounding Lemmas \ref{lemma: path algorithm, different manifolds |x-y|<epsilon_+} and \ref{Lemma: path algorithm, different manifolds |x-y|>epsilon_-}. \nc
	\end{remark}
	
	\begin{definition}[Fully inner Connected graphs]\label{Inner Fully Connected}
		Let $X={x_1, \dots, x_n}$ be samples from $\mu$ as defined in \eqref{eqn:DefMu}. A weighted graph $(X,\omega)$ is said to be \textit{fully inner connected} relative to the $\veps_+,\veps_-$ weights as $n \rightarrow \infty$, if with probability $1-C_1(n)$, where $C_1(n) \rightarrow 0$ as $n \rightarrow \infty$, for any pair of points $x_i, x_j$ belonging to the same manifold $\M_k$ we have $\omega_{x_i,x_j}=\omega^{\veps_+,\veps_-}_{x_i,x_j}$.

		%
	\end{definition}

	\begin{remark}
	\blue  It is possible to generalize the definition of full inner connectivity considered here and adapt it to other base proximity graph constructions like weighted proximity graphs. Indeed, the essential requirement that the full inner connectivity condition imposes on $(X, \omega)$ is that it should behave like a graph that can capture the underlying geometry of each individual manifold. We have chosen annular proximity graphs here because 1) they can be used to capture the geometry of the underlying manifolds (as shown by our analysis: just take Theorems \ref{Rate of convergence for eigenvalues} and \ref{convergence rate for eigenvectors} and consider the case where the number of manifolds is equal to one), and 2) because they introduce an extra degree of flexibility that, as we show later on, allows us to prove better recovery guarantees for MMC than what we can show for standard $\veps$-graphs; see Remark \ref{rem:RemarkAnnular}. 
	
	\end{remark}

	Next, we introduce the notion of spare outer connectivity.
	
	\begin{definition}[Sparsely Outer Connected graphs]\label{Outer Partially Connected} Let $X={x_1, \dots, x_n}$ be samples from $\mu$ as defined in \eqref{eqn:DefMu}, and let $(X,\omega)$ be a weighted graph. Let $N_{sl}$ be the number of connections between $x_i\in \M_s$ and $x_j \in \M_l$ such that $\omega_{ij}>0$, and let  \[N_0:=\max_{l\not = s} \{ N_{ls} \}.\]
		The graph is said to be \textit{sparsely outer connected} relative to $\veps_+$ and $\veps_-$ converging to zero as $n \rightarrow \infty$ if with probability one $\frac{N_0}{n^2(\veps_+^{m+2}-\veps_-^{m+2})}\rightarrow 0$ as $n \rightarrow \infty$. We recall that $m= \max_{l=1, \dots, N} m_l$.
	\end{definition}

	The above notions will capture the intuitive desire of giving high weights to pairs of points that are close to each other when they belong to the same manifold (full inner connectivity condition) and to give low weights to pairs of points when they lie on different manifolds (sparse outer connectivity condition). \blue These notions are geometric adaptations to the setting of interest of general notions explored in the literature to describe the feasibility of a clustering problem. We elaborate on this next.
 
 In the setting considered in \citet{ng2001spectral}, for example, given a network $(X, W)$, four conditions on the network (that depend on the Laplacian and degree function of the network) are proposed to ensure that a certain spectral embedding constructed from $(X, W)$ maps the original set of nodes $X$ to points that are close to a set of orthogonal vectors in Euclidean space; notice that the conditions in \citet{ng2001spectral} do not rely on any specific modelling assumption on the generative process that produces the data set $X$. Works like \citet{schiebinger2015geometry} and \citet{trillos2021geometric} have taken a different perspective and introduced sufficient geometric conditions on certain families of \textit{generative models} that, at the ``ground-truth" level, guarantee the feasibility of a certain ground-truth level clustering problem. These works then show that, when their proposed modelling assumptions hold, the conditions in \citet{ng2001spectral} are satisfied with high probability by certain network constructions $(X, W)$, where $X$ is a set of samples from the generative model, and $W$ is a suitable weight matrix over $X$. The notions of inner connectivity and outer connectivity considered here can be interpreted as conditions that a given weight matrix $W$ built over a data set $X$ sampled from a model like \eqref{eqn:DefMu} must satisfy in order for the network $(X, W)$ to: 1) satisfy the conditions in \citet{ng2001spectral} and 2) have clusters that are consistent with the underlying manifolds in the generative model \eqref{eqn:DefMu}. In this sense, the results that we present in this paper are analogous to those in \citet{schiebinger2015geometry,trillos2021geometric}, except that the geometric structure of the generative models in our paper is substantially different from the ones in those works. Other works in the literature such as \citet{vu_2018} have considered other types of ``clusterability" conditions, requiring data points within each cluster to be close to each other and points from different clusters to be far away from each other. These conditions are certainly not satisfied in this paper, mainly because the separation between manifolds can in fact be equal to zero.

	\blue 
	
	\medskip
	
Before we finish this section, we remark that the discussion in the upcoming sections \ref{sec:TensorizedLap}-\ref{subsection: Different dimensions} will not be restricted to any particular graph construction. In the results presented there, we quantify the error of approximation of the spectra of tensorized Laplacians at the continuum level from the graph Laplacian associated to $(X, \omega)$. This approximation error will naturally depend on the quantities $C_1(n)$ and $N_0$ appearing in the definition of the inner and outer connectivity conditions. In section \ref{sec:GraphConstruct}, we provide one example of a family of graphs that satisfies the inner and outer connectivity conditions. The graphs in that family are obtained by pruning an $\veps_+, \veps_-$ graph, removing edges between points for which there is no almost straight path connecting them. In section \ref{Section: Plane Algorithm with epsilon_+,epsilon_- graph setting}, we discuss other popular choices of weights $\omega$ that are based on the comparison of local tangent planes, but that, as we will discuss, do not, in general, satisfy the full inner connectivity condition. 
	
	\nc 

	\subsection{Basic properties of the spectrum of the operator $\Delta_\M$}
	\label{sec:TensorizedLap}
	
	In order to state our main theoretical results in sections \ref{sec:SameDim} and \ref{subsection: Different dimensions} we first discuss some basic properties of the spectrum of the operator $\Delta_\M$ in \eqref{eqn:LaplacianContinuum} and its relation to the MMC problem.

	Let $L^2(\mu)$ be the space of $N$ tuples $(f_1, \dots, f_N)$ where each $f_l \in L^2(\mu_l)$. We endow the space $L^2(\mu)$ with the tensorized inner product:
	\[ \langle  f , g \rangle_{L^2(\mu)}:= \sum_{l=1}^N w_l \langle f_l, g_l \rangle_{L^2(\mu_l)}    = \sum_{l=1}^N w_l \int_{\M_l} f_l(x) g_l(x) d\mu_l(x),    \]
	where $f=(f_1, \dots, f_N) \in L^2(\mu)$ and $g=(g_1, \dots, g_N) \in L^2(\mu)$. A tensorized Sobolev space $H^1(\mu)$ is defined as the space of $f=(f_1, \dots, f_N) \in L^2(\mu)$ for which $f_l \in H^1(\M_l)$ for each $l=1, \dots, N$. In particular, for elements $f \in H^1(\mu)$ the quantity
	\[ \sum_{l=1}^N w_l^2\int_{\M_l}  |\nabla f_l(x)|^2 \rho^2_l(x) dx  \]
	is finite. We then define the weighted Dirichlet energy:
	\begin{equation}\label{Equ: dirichlet energy}
	D(f):= \begin{cases} \sum_{l=1}^N w_l^2\int_{\M_l} |\nabla f_l(x)|^2 \rho^2_l(x) d \vol_{\M_l}(x), \quad \text{ if } f \in H^1(\mu), \\  +\infty, \quad \text{ if } f \in L^2(\mu) \setminus H^1(\mu) . \end{cases}
	\end{equation}
	
	Now, notice that the operator $\Delta_\M$ is self-adjoint with respect to the inner product $\langle  \cdot , \cdot \rangle_{L^2(\mu)}$ simply because each of the operators $\Delta_{\M_l}$ is self-adjoint w.r.t. $\langle \cdot ,\cdot  \rangle_{L^2(\mu_l)}$ (e.g. see \citet{trillos2018variational}). Given that each $L^2(\M_l, \rho_l)$ admits an orthonormal basis $\{ f_l^k \}_{k \in \N}$ of eigenvectors of $\Delta_{\M_l}$, we can see that the set of $f^k \in L^2(\mu)$ of the form
	\begin{equation}
	\label{eqn:FormEigen}
	f^k=(0, \dots, \frac{1}{\sqrt{w_l} } f_l^k, \dots, 0)
	\end{equation}
	for $k \in \N$ and $l=1, \dots, N$ is an orthonormal basis for $L^2(\mu)$. In addition, for such $f^k$ we have 
	\[\Delta_{\M} f^k = (0 , \dots,   \frac{w_l}{\sqrt{w_l}} \Delta_{\M_l} f_l^k ,\dots, 0) = w_l\lambda (0 , \dots,   \frac{1}{\sqrt{w_l}} f_l^k ,\dots, 0) =w_l\lambda f^k  \]
	for some eigenvalue $\lambda$ of $\Delta_{\M_l}$. In conclusion, we can build an orthonormal basis for $L^2(\mu)$ consisting of eigenfunctions of $\Delta_\M$ of the form \eqref{eqn:FormEigen}. From the above we can also conclude that the set of eigenvalues of $\Delta_\M$ is the set of numbers of the form $w_l\lambda$ for some $l$, where $\lambda$ is an eigenvalue of $\Delta_{\M_l}$. In terms of the Dirichlet energy defined in \eqref{Equ: dirichlet energy}, the eigenvalues of $\Delta_\M$, arranged in increasing order according to multiplicity, can be written as 
	\begin{equation}
	\label{minmax principle for laplacian}
	\lambda_{l}=\min _{S \in \mathfrak{S}_{l}} \max _{f \in S \backslash\{0\}} \frac{D(f)}{\|f\|_{L^{2}(\mu)}^{2}}.
	\end{equation}
	where $\mathfrak{S}_{l}$ denotes the set of all linear subspaces of $L^2(\mu)$ of dimension $l$.
	
	Regarding the zero eigenvalue of $\Delta_\M$, notice that since the manifolds $\M_l$ were assumed to be connected, the multiplicity of the zero eigenvalue for the operator $\Delta_\M$ is equal to $N$. Moreover, an orthonormal basis for this eigenspace is the set of functions of the form $(0, \dots, c_l\mathds{1}_{\M_l}, \dots, 0)$ where $c_l$ is a normalization constant. This observation is the key property that allows us to think of the multi-manifold clustering problem in terms of the spectrum of the operator $\Delta_\M$. However, it should be clear that the tensorized Laplacian has much more information than that needed to solve the MMC problem.

	For convenience, we also introduce Dirichlet energies associated to each manifold $\M_l$:
	\begin{equation}
	\label{Equ: dirichlet energyLocal}
	D_l(f_l):= \begin{cases} \int_{\M_l} |\nabla f_l(x)|^2 \rho^2_l(x) d \vol_{\M_l}(x), \quad \text{ if } f_l \in H^1(\mu_l), \\ \\ +\infty \quad \text{if} f \in L^2(\mu_l) \setminus H^1(\mu_l).\end{cases}
	\end{equation}

	\subsection{Convergence results in the $m_1=\dots= m_N$ case }
	\label{sec:SameDim}
	
	In this section we establish high probability error bounds between the spectrum of a rescaled version of the graph Laplacian $\Delta_n$ defined in \eqref{eqn:GraphLaplacian1} and the spectrum of $\Delta_\M$ under the additional assumption that all manifolds $\M_k$ have the same dimension. The results presented in this section apply to generic weighted graphs $(X,\omega)$, but the error estimates are only meaningful when the quantities $N_0$, $C_1(n)$ and $\veps_+$ from section \ref{Section:Conditions on algorithms for the recovery of tensorized spectrum from data clouds} scale appropriately with the number of data points. In section \ref{sec:GraphConstruct} we present a specific construction for $(X, \omega)$ where we can make our error estimates concrete.

	In what follows we make the following assumptions on the parameters $\varepsilon_+,\varepsilon_-,\widetilde{\delta},\theta$. Here $\tilde{\delta}$ and $\theta$ are small parameters that we use to tune the probabilities of some random events defined in corollary \ref{cor:Densities}. 
	\begin{assumption}\label{assumption: main}
		We assume that the quantities $\veps_+, \veps_-, \tilde{\delta}, \theta $ satisfy:
		\begin{enumerate}[(1)]
			\item $\eps_+\leq \min\{1,\frac{R}{2},CK^{-1/2},i_0\}$, where $R,K$ are uniform upper bounds on the reach and on the absolute values of the sectional curvatures for all the manifolds, $i_0$ is a lower bound on the injectivity radius of all manifolds, and $C$ is a constant no larger than 1.\label{Assump2:1}
			\item $\varepsilon_-\le \frac{1}{4}\varepsilon_+$; the $\frac{1}{4}$ here is an arbitrary number smaller than $1$.\label{Assump2:2}
			\item $\frac{c}{n^{1/m}}<\widetilde{\delta}$\label{Assump2:3}
			\item $C(\widetilde{\delta}+\theta)\leq \frac{1}{2 c_\rho}$, where $\frac{1}{c_\rho}$ is the lower bound of $\rho$ .\label{Assump2:4}
		\end{enumerate}
	\end{assumption}
	With the above assumptions we can make sure that with the underlying $\veps_+, \veps_-$-weighted graph we can approximate the operators $\Delta_{\M_k}$ in each of the $\M_k$; this part does not rely on the assumption that all manifolds have the same dimension, and only depends on the full inner connectivity property. For the spectrum of the graph Laplacian associated to $(X,\omega)$ to successfully recover the spectrum of $\Delta_\M$ we need $(X,\omega)$ to be fully inner connected and sparsely outer connected as we will make explicit in our first theorem.

	\begin{theorem}[Convergence rate for eigenvalues]\label{Rate of convergence for eigenvalues}

		Let $\mu$ be a probability measure on $\M $ as in \eqref{eqn:DefMu}. Suppose that the $\M_k$ forming the $\M$ satisfy Assumptions \ref{assump:WellSeparated}, and assume also that $ m_1=\dots=m_N$. Let $X=\{ x_1, \dots, x_n \}$ be i.i.d. samples from $\mu$. Let $(X,\omega)$ be a symmetric weighted graph and let $\mathcal{L}$ be the rescaled graph Laplacian:
		\begin{equation}
		\mathcal{L}u(x):=\frac{1}{n^2(\veps_+^{m+2}-\veps_-^{m+2})}\sum_{y\in  X }\omega_{xy}(u(x)-u(y)), \quad x \in X , u : X \rightarrow \R. 
		\label{eqn:GraphLaplacian}
		\end{equation}
		Suppose that the quantities $\widetilde{\delta}, \theta, \varepsilon_+,\varepsilon_-$ satisfy Assumptions \ref{assumption: main}. Let $\lambda_k^{\veps_+, \veps_-}$ be the $k$-th eigenvalue of $\mathcal{L}$ and let $\lambda_k$ be the $k$-th eigenvalue of $\Delta_\M$, where $\Delta_\M$ is the tensorized Laplacian from \eqref{eqn:LaplacianContinuum}. \blue Finally, let $t:= \frac{n \omega_{\min}}{2}$. \nc  Then there exists a constant $C$ (independent of $k$) such that, with probability at least $1-\sum_{l=1}^N (n w_l+t) \exp \left(-{C}(n w_l-t) \theta^{2} \widetilde{\delta}^{m}\right)-2N\exp \left(\frac{-2 t^{2}}{n}\right)-C_1(n)$, for every $k \in \N$ for which
		\[
		C\widetilde{\delta}\sqrt{\lambda_k}+C(\theta+\widetilde{\delta})  <   \frac{1}{k} \nc,
		\]
		we have:
		\[
		\left|\lambda_{k}^{\varepsilon_+,\varepsilon_-}-\sigma_\eta \lambda_{k}\right| \leq e_k+C\left(\veps_+(\sqrt{\lambda_k}+1)+\theta+\frac{\widetilde{\delta}}{\veps_+}\right)\lambda_k.
		\]
		In the above, $e_k=\frac{CN_0}{n^2(\veps_+^{m+2}-\veps_-^{m+2})}\left( 1+C'(\lambda_k^{m/2+1}+\widetilde{\delta}\sqrt{\lambda_k}+\theta+\widetilde{\delta}) \right)$, and $C_1(n)$ and $N_0$ are introduced in Definition \ref{Inner Fully Connected} and Definition \ref{Outer Partially Connected}, respectively. The constant $\sigma_\eta$ is given by:
		\begin{equation}
		\sigma_{\eta}:=\int_{\R^m}|y_1|^2\eta(|y|)dy,
		\label{eqn:sigmaeta}
		\end{equation}
		where $\eta=\mathds{1}_{r\le 1}$ \new{and $y_1$ is the first coordinate of $y$}. 
	\end{theorem}
	
	
	\begin{remark} 

		\begin{enumerate}

			\item In general, we should expect a trade-off between the quantities $C_1(n)$ and $N_0$. That is, in general, an attempt at making $N_0$ smaller (i.e., erase connections between different manifolds) will typically result in a smaller probability of having a graph that is well connected within each manifold $\M_l$.
			
			\item The benefits that come from taking $\veps_- >0$ for the MMC problem are not explicit in the error bounds from Theorem \ref{Rate of convergence for eigenvalues}. However, as we will see later on, by tuning  $\veps_-$ appropriately, one can substantially eliminate connections between data points in different manifolds when one considers $\veps_- \sim \veps_+$. This means a substantial decrease in $N_0$.  We explain this in Remark \ref{rem:Eps-equalzero} for the specific annular graph construction with angle constraints. The fact that we can improve the performance of MMC algorithms by introducing $\veps_+,\veps_-$-graphs motivates the theoretical analysis that we present in the Appendix.
			
			%
			
			
			\item The proof of the estimates in Theorem \ref{Rate of convergence for eigenvalues} relies on a variational approach that compares Dirichlet energies at discrete and continuum levels. This approach has been used before in \citet{BIK,trillos2019error,Lu2019GraphAT}. However, the structure of the $\veps_+, \veps_-$-graph that we consider here forces us to modify the analysis and present new proofs. Even for a single manifold $\M=\M_1$, the analysis of graph Laplacians on $\veps_+, \veps_-$-graphs is a technical contribution of this work. The actual proof of Theorem \ref{Rate of convergence for eigenvalues} appears in section \ref{Proof:ThmEigen} in the Appendix. Several technical preliminary results are established in the preceding sections. 
			
			\item The scaling factor relating $\mathcal{L}$ and $\Delta_n$ in \eqref{eqn:GraphLaplacian} is irrelevant in practice because the eigenvectors of $\Delta_n$ are the same as those for $\mathcal{L}$, and the ratio between eigenvalues of $\Delta_n$ coincides with the ratio of eigenvalues of $\mathcal{L}$. In other words, in practice we can work directly with $\Delta_n$ without having to compute the rescaling factor.\nc
			
			\item If we choose $ 1\gg \veps_+ \gg \left(\frac{\log(n)}{n} \right)^{1/m}$, then, with high probability, the error of approximation of eigenvalues scales like:
			\[  \frac{N_0}{n^2\veps_+^{m+2}} + \frac{\left(\frac{\log(n)}{n} \right)^{1/m}}{\veps_+} + \veps_+. \]
			This result is analogous to results in \citet{BIK} and \citet{trillos2019error}, except that now we have the extra $\frac{N_0}{n^2\veps_+^{m+2}}$ term. In order for this error estimate to converge to zero in the large data limit we thus need to require the graph to satisfy the sparse outer connectivity condition.
			
		\end{enumerate}
		
		\label{rem:MainTheorem}	
	\end{remark}
	
	\begin{theorem}[Convergence rate for eigenvectors]\label{convergence rate for eigenvectors}
		Under the same setting and assumptions as in Theorem \ref{Rate of convergence for eigenvalues},
		for every $k \in \mathbb{N}$ there is a constant $c_{k}=c_k(\M)$ such that if
		\begin{equation*}
		e_k + C\left(\veps_+\sqrt{\lambda_k}+\veps_+^2+\theta+\frac{\widetilde{\delta}}{\veps_+}\right)\leq c_k,
		\end{equation*}
		then, with probability at least $1-\sum_{l=1}^N (n w_l+t) \exp \left(-{C}(n w_l-t) \theta^{2} \widetilde{\delta}^{m}\right)-2N\exp \left(\frac{-2 t^{2}}{n}\right)-C_1(n)$, for every $v_k$ normalized eigenvector of $\mathcal{L}$ with eigenvalue $\lambda_{k}$, there is a normalized eigenfunction $f_{k}$ of $\Delta_{\M}$ with eigenvalue $\lambda_{k}$ such that
		\begin{equation*}
		\lVert f_k-v_k\lVert_{L^2(\mu^n)}\leq \left[Ce_k+C\left(\veps_+\sqrt{\lambda_k}+\veps_+^2+\theta+\frac{\widetilde{\delta}}{\veps_+}\right)\right]^{1/2}+C_{\mathcal{M},\lambda}\widetilde{\delta},
		\end{equation*}
		where $e_k$ is the same as in Theorem \ref{Rate of convergence for eigenvalues}.
	\end{theorem}
	
	The proof of this theorem is presented in section \ref{Proof:Eignvector} in the Appendix.
	
	\begin{remark}\label{rem:relaxation of outer condition}
	
	    \new{The sparse outer connectivity condition $\frac{N_0}{n^2(\eps_+^{m+2}-\eps_-^{m+2})}\to 0$ is imposed to guarantee the recovery of the full spectrum of the tensorized Laplacian in the large data limit. However, we highlight that our error estimates continue to be meaningful even if we only impose $\frac{N_0}{n^2(\eps_+^{m+2}-\eps_-^{m+2})}$ to be asymptotically smaller than some small tolerance level $c$.    
}
	\end{remark}
	
	%
	%

	\subsection{Mixed dimensions.}\label{subsection: Different dimensions}
	We generalize our results from section \ref{sec:SameDim} to a setting where the manifolds $\M_k$ may have different dimensions. For convenience, we introduce some notation first.
	
	Without the loss of generality we can assume that the manifolds $\M_k$ are indexed in decreasing order of dimension, i.e. $m=m_1 \geq m_2 \geq \dots \geq m_N$. We let $N_{\max}$ be the number of manifolds with the maximum dimension $m$, i.e. $m_1 = \dots= m_{N_{\max}}> m_{N_{\max} +1}$. We set $\M_{\max}:= \M_1\cup \dots \cup \M_{N_{\max}}$ and write $\langle  f , g \rangle_{L^2(\M_{\max})}$ to represent:
	\[ \langle  f , g \rangle_{L^2(\M_{\max})} =   \sum_{i=1}^{N_{\max}} w_i \langle f_i, g_i \rangle_{L^2(\mu_i)}    = \sum_{i=1}^{N_{\max}} w_i \int_{\M_i} f_i(x) g_i(x) d\mu_i(x).   \]
	We also use $\lVert f \rVert_{L^2(\M_{\max})}^2 =  \langle  f , f \rangle_{L^2(\M_{\max})}$. 
	
	Notice that with the above inner product we can identify (isometrically) elements in $ L^2(\M_{\max})$ with elements in $L^2(\mu)$ that are zero outside of $\M_{\max}$; throughout section \ref{sec:ProofSiffDim} in the Appendix we may use this identification without any further explanation. Finally, we use $\Delta_{\M_{\max}}$ to denote the tensorized Laplacian \eqref{eqn:LaplacianContinuum} for $\M_{\max}$ (i.e. just as in \eqref{eqn:LaplacianContinuum} but with only the first $N_{\max}$ coordinates); we use $D_{\max}$ to denote the corresponding Dirichlet energy defined for $L^2(\M_{\max})$ functions.

	\begin{theorem}
		\label{thm:MixedDimensions}
		Let $\mu$ be a probability measure on $\M $ as in \eqref{eqn:DefMu}. Suppose that the $\M_k$ forming $\M$ satisfy 
		Assumptions \ref{assump:WellSeparated}, and let $N_{\max},\M_{\max}, \Delta_{\M_{\max}}$ be defined as before. Set $\lambda_1, \dots, \lambda_{N}=0$ and let $\lambda_{N+1} \leq \lambda_{N+2} \leq \dots  $ be the list of non-zero eigenvalues of $\Delta_{\M_{\max}}$ repeated according to multiplicity.
		
		Let $X=\{ x_1, \dots, x_n \}$ be i.i.d. samples from $\mu$, let $(X,\omega)$ be a symmetric weighted graph, and let $\mathcal{L}$ be the rescaled graph Laplacian from \eqref{eqn:GraphLaplacian}. Finally, suppose that the quantities $\widetilde{\delta}, \theta, \varepsilon_+,\varepsilon_-$ satisfy Assumptions \ref{assumption: main}.
		
		Then, for some constant $C=C(\M,\mu)$, with probability at least 
		\[1-\sum_{l=1}^N (n w_l+t) \exp \left(-\mathrm{C}(n w_l-t) \theta^{2} \widetilde{\delta}^{m}\right)-2N\exp \left(\frac{-2 t^{2}}{n}\right)-C_1(n),\] for every $k \in \N$ for which
		\[
		C\widetilde{\delta}\sqrt{\lambda_k}+C(\theta+\widetilde{\delta})  <   \frac{1}{k} ,
		\]
		we have:
		\[
		\left|\lambda_{k}^{\varepsilon_+,\varepsilon_-}-\sigma_\eta \lambda_{k}\right| \leq e_k+C\left(\veps_+(\sqrt{\lambda_k}+1)+\theta+\frac{\widetilde{\delta}}{\veps_+}\right)\lambda_k.
		\]
		In the above, $e_k=\frac{CN_0}{n^2(\veps_+^{m+2}-\veps_-^{m+2})}\left( 1+C'(\lambda_k^{m/2+1}+\widetilde{\delta}\sqrt{\lambda_k}+\theta+\widetilde{\delta}) \right)$, and $C_1(n)$ and $N_0$ are introduced in Definition \ref{Inner Fully Connected} and Definition \ref{Outer Partially Connected}, respectively. 
		
		In addition, there is a constant $c_{k}=c_k(\M)$ such that if
		\begin{equation*}
		e_k + C\left(\veps_+\sqrt{\lambda_k}+\veps_+^2+\theta+\frac{\widetilde{\delta}}{\veps_+}\right)\leq c_k,
		\end{equation*}
		then, with probability at least $1-\sum_{l=1}^N (n w_l+t) \exp \left(-{C}(n w_l-t) \theta^{2} \widetilde{\delta}^{m}\right)-2N\exp \left(\frac{-2 t^{2}}{n}\right)-C_1(n)$, for every $u_k$ normalized eigenvector of $\mathcal{L}$ with eigenvalue $\lambda_{k}$, there is a normalized eigenfunction $f_{k}$ of $\Delta_{\M_{\max}}$ with eigenvalue $\lambda_{k}$ such that
		\begin{equation*}
		\lVert f_k-u_k\lVert_{L^2(\mu^n)}\leq \left[Ce_k+C\left(\veps_+\sqrt{\lambda_k}+\veps_+^2+\theta+\frac{\widetilde{\delta}}{\veps_+}\right)\right]^{1/2}+C_{\mathcal{M},\lambda}\widetilde{\delta}.
		\end{equation*}
		In the above, we interpret the functions $f_1, \dots, f_N$ as an orthonormal basis for $Span \{ \mathds{1}_{\M_1}, \dots, \mathds{1}_{\M_N} \}$.
		\nc
	
	\end{theorem}

	\begin{remark}
		
		\begin{enumerate}
			\item The proof of this theorem appears in section \ref{sec:ProofSiffDim} in the Appendix. We remark that the proof of Theorem \ref{thm:MixedDimensions} is only based on Theorems \ref{Rate of convergence for eigenvalues} and \ref{convergence rate for eigenvectors} and on a few associated preliminary results. 
			
			\item \new{When manifolds have different dimensions, making sure that the sparse outer connectivity condition is satisfied is more difficult because, in general, $N_{kl}$ is much larger when the dimensions of the manifolds $\M_k$ and $\M_l$ are small than when they are large. Indeed, since the number of points in each manifold is in the order of $n$, the number of points in a neighborhood of size $\veps$ around a point on a manifold with small dimension will be larger than when the manifold has larger dimension.}

			\new{\item Notice that when manifolds do not intersect the outer sparse connectivity condition is trivially satisfied. If in addition we assume the full inner connectivity condition, then we can conclude that the eigenvectors of the graph Laplacian corresponding to non-zero eigenvalues will only recover the spectra of the manifolds with dimension $m$.}
		\end{enumerate}

	\end{remark}

	\nc

	\section{Annular proximity graphs with angle constraints}\label{sec:GraphConstruct}

	In this section we introduce a graph construction that is both fully inner connected and sparsely outer connected. We start with a definition. 
	
	\begin{definition}\label{def:angle constraint}
		Let $\alpha \in (0,\pi/2)$ and $r>0$. We call an ordered sequence of data points \[({x}_{i_1},{x}_{i_2},...,{x}_{i_m})\] an $(\alpha,r)$-constrained path between $x_{i_1}$ and $x_{i_m}$ if the following two conditions hold:
		\begin{enumerate}
			\item $\angle( {x}_{i_{m}} -{x}_{i_{1}},{x}_{i_{j+1}}-{x}_{i_{j}})<\alpha, \forall j=1,2,...,m-1.$
			\item $|x_{i_j}-x_{i_{j+1}}|<r, \forall j=1,2,\dots,m-1.$
		\end{enumerate}
	\end{definition}
	\blue
		The first condition in the definition of an $(\alpha,r)$-constrained path requires the path to be almost straight, while the second condition requires consecutive points in the path to be close enough. The example in Figure \ref{fig:DEF3.1} shows that it is possible to have two points $x$ and $y$ on different manifolds for which there is an $(\alpha,r)$-constrained path between them. The intuition motivating this definition, however, is that the number of such pairs is small under Assumption \ref{assump:WellSeparated}.
	\nc
	\begin{figure}
		\centering
		\includegraphics[scale=0.5]{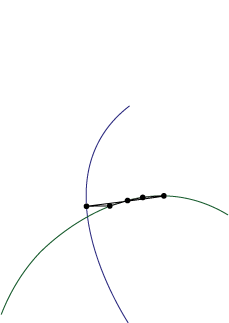}
		\put(0,50){$\M_l$}
		\put(-55,120){$\M_k$}
		\put(-35,70){$x_{i_1}$}
		\put(-46,50){$x_{i_2}$}
		\put(-57,68){$x_{i_3}$}
		\put(-65,48){$x_{i_4}$}
		\put(-88,58){$x_{i_5}$}
		\caption{\blue An example of an $(\alpha,r)$-constrained path where $x_{i_1}$ and $x_{i_5}$ are on different manifolds. To satisfy the constraints, the segments $x_{i_j}x_{i_{j+1}}$ and $x_{i_1}x_{i_m}$ must almost align. \nc }
		\label{fig:DEF3.1}
	   \end{figure}

	We now introduce the family of \textit{annular proximity graphs with angle constraints} $(X, \omega)$ that we study throughout the rest of this section.

	\begin{definition}[Annular proximity graphs with angle constraints]\label{def:Annular proximity graphs with angle constraints}
	Suppose that  $x_i, x_j$ are two data points such that $|x_i-x_j|\leq  \veps_-$ or $|x_i-x_j|\geq \veps_+$, then we set $\omega_{ij}=0$.  If $ \veps_-<|x_i-x_j|< \veps_+$ and there is an $(\alpha, r)$-constrained path between $x_i$ and $x_j$, then we set $\omega_{ij}=1$, otherwise we set $\omega_{ij}=0$. We refer to this type of graph as an \textit{annular proximity graph with angle constraints}.    
	\end{definition}

	\blue
	
	\begin{algorithm}[H]
	\begin{algorithmic}
	\State \textbf{Input: } source nodes $y_1,y_2$; data points $\{y_i\}_{i=3}^{\widetilde{n}}$ such that $|y_i-y_1|\le \epsilon_+$; parameters $\veps_+,\veps_-,r,\alpha$, where $\veps_+>\veps_->r>0$.
	\State \textbf{Output:} the shortest angle constrained path between $y_1$ and $y_2$.
	\\
	\hspace{1cm}
	\If{$|y_1-y_2|>\eps_+$ or $|y_1-y_2|<\eps_-$}
	\State{}Output $w=0$.
	\EndIf{}
	\State{} Construct $r$-graph $E$ on $\{y_i\}_{i=1}^{\widetilde{n}}$, that is $e_{ij}=\mathds{1}_{|y_i-y_j|\le r}$.
	\For{$e_{ij}=1$ and $\frac{\langle y_j-y_i,y_2-y_1 \rangle}{|y_j-y_i| \cdot |y_2-y_1|}<\cos\alpha$}
	\State{} Denote $e_{ij}=0$.
	\EndFor{}
	\State{}Apply Dijkstra algorithm to find the shortest path between $y_1$ and $y_2$ using $E$.
	\end{algorithmic}
	\caption{Annular proximity graph with angle constraints}
	\label{algorithm: Path Algorithm}
	\end{algorithm}
	
At the beginning of section \ref{section: discussion} we discuss the computational complexity of building annular graphs with angle constraints. As discussed there, $\omega$ can be constructed following a simple modification of the constrained Dijkstra's algorithm from \citet{BABAEIAN2015118}; see  Algorithm \ref{algorithm: Path Algorithm}. \nc On the other hand, from a theoretical perspective, we show that with the right choice of parameters ($\veps_+, \veps_-, \alpha, r$), these graphs satisfy the full inner connectivity and sparse outer connectivity conditions; the precise statements are contained in Theorem \ref{Proposition: path algorithm, inner fully connected} and Theorem \ref{theorem: path algorithm} below. In particular, in Theorem \ref{theorem: path algorithm} (\blue see also Remark \ref{rem:Eps-equalzero} and an illustration in Figure \ref{fig:epsminusisimportant}\nc) we quantify the benefits of considering annular graphs with $\veps_- \sim \veps_+$. In section \ref{sec:veps-}, we revisit the benefits of considering annular graphs for MMC, this time from a numerical perspective.

	\begin{theorem}
		\label{Proposition: path algorithm, inner fully connected}
		Let $n_k$ be the number of data points in $X \cap \M_k$, let $\alpha \in [0,\pi/4)$ and let $r\leq C \veps_+$. Then, with probability at least 
		\[ 1-\frac{C_k n^2\veps_+}{r} \exp\left(-C_kn_kr^{m_k}(\tan\alpha)^{m_k-1}\right), \] 
		for any two points $x_i,x_j\in\mathcal{M}_l$ such that $|x_i-x_j|<\veps_+$, there exists an $(\alpha,r)$-constrained path between $x_i$ and $x_j$. In the above, $C_k$ is a constant that depends on $\M_k$ and $\rho_k$.
	\end{theorem}

	\begin{theorem}\label{theorem: path algorithm} 
		
		Suppose that Assumption \ref{assump:WellSeparated} and Assumption \ref{assumption: main}.\ref{Assump2:1}, \ref{assumption: main}.\ref{Assump2:2} hold, and suppose that $\M_{kl} =\M_k \cap \M_l  \not = \emptyset$. 
		
		\begin{enumerate}[(1)]
			\item
			If $\alpha < \arcsin(\frac{C_{k,l} r}{\veps_+})$ , and $r \leq c \veps_+$, then, with probability no less than 
			\[ 1- C_{k,l} n\exp(-C_{k,l} n   \min\{ r^{m_k - m_{kl}}, r^{m_l - m_{kl}}, \veps_+^{m_k}, \veps_+^{m_l}   \}), \]
			$N_{kl}$, the number of connections between points in $X \cap \M_k$ and $X \cap \M_l$, satisfies
			\[ N_{kl} \leq C_{k,l} n^2\max \{ r^{m_k- m_{kl}} \veps_+^{m_l},   r^{m_l- m_{kl}} \veps_+^{m_k}   \}.   \]
		\end{enumerate}

		\begin{enumerate}[(2)]
			\item If in (1) we further assume the lower bound $\veps_- \geq c \veps_+$, then, with probability no less than 
			\[ 1- C_{k,l} n\exp(-C_{k,l} n   \min\{ (\frac{r^2}{\veps_+} \new{+r\eps_+ } )^{m_k - m_{kl}}, (\frac{r^2}{\veps_+} \new{+r\eps_+ } )^{m_k - m_{kl}} , \veps_+^{m_k}, \veps_+^{m_l}   \}), \]
			we have
			\[ N_{kl} \leq C_{k,l} n^2\max \{ (\frac{r^2}{\veps_+} \new{+r\eps_+ } )^{m_k - m_{kl}} \veps_+^{m_l},   (\frac{r^2}{\veps_+} \new{+r\eps_+ } )^{m_k - m_{kl}} \veps_+^{m_k}   \}.   \]
		\end{enumerate}
	\end{theorem}

	


	\begin{remark}\label{rem:rate}
		To illustrate our results and obtain concrete error estimates in Theorem \ref{Rate of convergence for eigenvalues}, for example, consider the case where all manifolds have the same dimension $m$. Going back to the last point in Remark \ref{rem:MainTheorem}, we need to tune the parameters $r$ and $\alpha$ so that $N_0 \ll n^2 \veps_+^{m+2}$ and also, following Theorem \ref{Proposition: path algorithm, inner fully connected}, we should have $n r^m \tan(\alpha)^{m-1} \gg 1$. Now, from Theorem \ref{theorem: path algorithm} it follows that we need $ r \ll \veps_+^{3/2}$ (assuming the worst case scenario where the dimension $m_{kl}= m-1$, and assuming we impose the lower bound on $\veps_-$, i.e. we treat $\veps_- \sim \veps_+$). On the other hand, if we set $\sin(\alpha)=C\frac{r}{\veps_+}$ we see from Theorem \ref{Proposition: path algorithm, inner fully connected} that we need $r^{2m-1}/\veps_+^{m-1}\gg \frac{1}{n} $ (omitting logarithmic terms). Thus, if we take $r \sim  \veps_+^{2}$ and we set $\veps_+= C\left( \frac{1}{n}\right)^{\frac{1}{3m-1}}$ for large enough constants $C$ we can satisfy all constraints \blue and get from Remark \ref{rem:MainTheorem} v) \nc the rate (omitting logarithmic terms):
		\[ O \left( \frac{1}{n^{1/(3m -1)}}  \right) \]
		for the convergence of the eigenvalues of the graph Laplacian on an annular path with angle constraints towards the eigenvalues of the tensorized Laplacian on $\M$. For comparison, recall that the convergence rate obtained in \citet{trillos2019error} for the regular convergence of graph Laplacians was $O \left( \frac{1}{n^{1/(2m)}}  \right)$ and the convergence rate in \citet{calder2019improved} is $O \left( \frac{1}{n^{1/(m+4)}}  \right)$. The extra sample complexity in our setting is induced by the additional mechanism that is needed to collect second order geometric information around the data in order to separate the underlying manifolds.

	\end{remark}
	
	\nc
	
	\begin{remark}
		\label{rem:Eps-equalzero}
		
		From Theorem \ref{theorem: path algorithm} we can see the quantitative effect of considering a non-zero $\veps_-$. Indeed, when $r$ is taken to be considerably smaller than $\veps_+$, the number of faulty connections in the $\veps_-\sim \veps_+$ setting is much smaller than when $\veps_-=0$ because \new{$\max\{\frac{r}{\veps_+},\eps_+\}$} is a small quantity. Intuitively, removing points from a base $\veps$-proximity graph should always reduce the number of faulty connections across different manifolds. However, what Theorem \ref{theorem: path algorithm} states is that, \blue when combined with the angle constraints \nc, the ratio of faulty connections erased by removing an inner ball with volume comparable to that of the outer ball is actually quite significant. This result motivates the theoretical analysis that we present in the Appendix. In our numerical experiments we illustrate further the superior performance of our MMC algorithm when we set $\veps_-\sim \veps_+$. \blue
	See an intuitive explanation in Figure \ref{fig:epsminusisimportant}.\nc
	\end{remark}

	\begin{remark}
		\label{rem:alphaSmall}
		It is not difficult to show that in general one can not relax the requirement that $\sin(\alpha) \ll 1$ in order to get sparsely outer connected graphs. Indeed, take for example two flats $\M_k$ and $\M_l$  with dimension $2$ that meet perpendicularly at a straight line $\ell$. If $\alpha \geq c >0$ for constant $c$, it is straightforward to see that there is a small enough constant $c_1$ (depending on the lower bound for $\alpha$) such that for all pairs of points $x\in \M_k$ and $y \in \M_l$ for which  
		$$dist(x, \M_{kl}) \leq c_1 \veps_+,\quad dist(y, \M_{kl}) \leq c_1 \veps_+,\quad |x-y| < \veps_+,$$
		\blue and for which the angle between $y-x$ and $\ell$ is smaller than $c_1$, there is an $(\alpha,r)$ constrained path between $x$ and $y$. This situation is illustrated on the left panel of \new{Figure \ref{fig:Rem3.6_1}}. \nc In turn, from this one can see that the number of connections that a point $x\in \M_{k}$ with $dist(x, \M_{kl}) \leq c_1\veps_+$ has with points in $\M_l$ is $O(n\veps_+^2)$ (i.e. the same order as with points in $\M_k$). In that case, $N_{0} \sim  n^2 \veps_+^{2+1} $ and thus $N_0 /(n^2 \veps_+^{2+2})  \rightarrow \infty$.

	    \begin{figure}[htbp]
		\par\medskip
		\centering
		\begin{minipage}[t]{0.45\textwidth}
			\centering
			\includegraphics[scale=0.3]{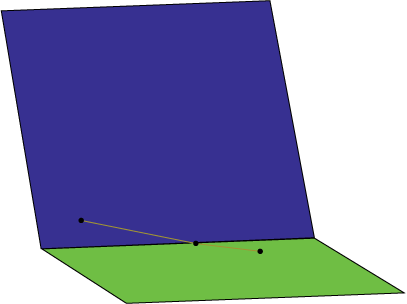}
		    \put(-45,5){$x$}
		    \put(-100,33){$y$}
		\end{minipage}
		\begin{minipage}[t]{0.45\textwidth}
			\centering
			\includegraphics[scale=0.3]{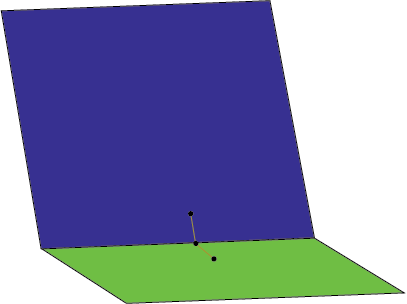}
		    \put(-54,5){$x$}
		    \put(-70,35){$y$}
		\end{minipage}
		\caption{\blue If $\alpha\ge c>0$, the number of pairs of points in different manifolds that can be connected by paths that are almost tangential to the intersection of the manifolds (as on the left panel) is of the same order as the number of connections between pairs of points on the same manifold that are within distance $\veps_+$ from the intersection of the manifolds. Notice that this situation may arise as soon as $m\ge 2$. For completeness, on the right panel we illustrate the benign situation of a pair of points on different manifolds that are close enough to each other and have no constrained path connecting them. \nc}


		
		 \label{fig:Rem3.6_1}
	\end{figure}
	    
		%
		\nc 
	\end{remark}

	\nc
	
	\subsection{Proofs of Theorems \ref{Proposition: path algorithm, inner fully connected} and \ref{theorem: path algorithm}}

	\begin{proof}[Proof of Theorem \ref{Proposition: path algorithm, inner fully connected}] Given that the manifold $\M_k$ is smooth and compact, and given that we only consider connecting two points $x,y \in X \cap \M_k$ when they are within a small distance $\veps_+$ from each other, we can (and will) assume for simplicity that $\M_k$ is a flat of dimension $m_k$. As will become clear from our argument, the reduction to the flat case is sufficient as all curvature effects only introduce lower order corrections to our estimates.

		Consider then the line segment connecting the points $x$ and $y$ and consider also the cylinder in $\M_k$ with axis given by the segment $y-x$ and circular base of radius $h_1$ centered at $x$ and orthogonal to the segment $xy$; see Figure \ref{fig:seriesPart}. We split this bigger cylinder into $l$ parallel smaller cylinders with height $h_2$. The smaller cylinders are labeled as $\mathcal{C}_1, \dots, \C_l$. By taking $h_2=cr$ for small enough constant $(1/4)> c>0$, $h_1=\frac{c}{2}\tan(\alpha)r $, and assuming that $|x-y| > 4cr$, we can guarantee that
		\begin{enumerate}
			\item $l$ is an odd number.
			\item $4h_1^2 +9h_2^2 \leq r^2$. 
			\item $\frac{2h_1}{h_2} \leq \tan(\alpha)$.
		\end{enumerate}
		With this construction it is clear that if $X \cap \C_s \not= \emptyset$ for every even $s$, then we can construct an $(\alpha, r)$-constrained path between $x$ and $y$; see Figure \ref{fig:seriesPart}. Notice that if on the other hand $|x-y|\leq 4cr \leq r$, then $x$ and $y$ can be connected directly. 
		
		In the more interesting case $\veps_+ \geq |x-y|>4cr$, the probability that there is at least one sample from the $n_k$ samples in $\M_k$ in all the cylinders $\C_s$ is no smaller than
		\[1- \frac{C_k \veps_+}{r}\left(1-C_kr^{m_k}(\tan\alpha)^{m_k-1}\right)^{n_k}= 1-\frac{C_k \veps_+}{r} \exp\left(-C_kn_kr^{m_k}(\tan\alpha)^{m_k-1}\right).\]
		By the discussion above, we conclude that the probability that there exists an $(\alpha,r)$-constrained path between $x$ and $y$ is no smaller than the above quantity.

		To bound from below the probability that \textit{all} pairs of points $x,y \in X \cap \M_k$ that are within distance $\veps_+$ are connected by $\left(\alpha,r\right)$-constrained paths it is sufficient to take a union bound using the above estimates.
	\end{proof}

	\nc

	Next we study the outer connectivity condition for annular proximity graphs with angle constraints. We start with a result that applies for all choices of $\veps_- \in [0, \frac{1}{4}\veps_+]$ and then refine the estimates for the case $\veps_- = c \veps_+$.

	
	
	
	\begin{figure}
		\centering
		\includegraphics[scale=1]{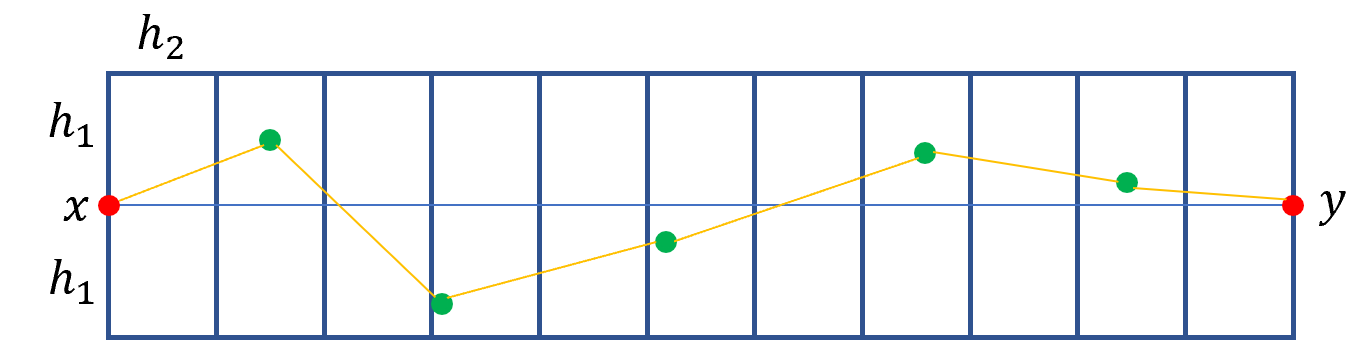}
		\caption{This is a valid path between $ x $ and $ y $ on the flat.}
		\label{fig:seriesPart}
	\end{figure}

	\new{
	
	\begin{lemma}\label{lemma: path algorithm, different manifolds |x-y|<epsilon_+}
		Suppose that $\mathcal{M}_l \cap \M_k\not = \emptyset$, and let $x\in \M_l$ and $y \in \M_k$ be such that $|x-y|<\veps_+$. Assume that $\dist(y, \M_{l})\ge\dist(x, \M_{k})> r$. If $y$ is such that
		\[\frac{C_{k,l}\dist(y, \M_{l})}{|x-y|} >  \sin(\alpha), \]
		then there is no $(\alpha,r)$-constrained path between $x$ and $y$; in the above, the constant $C_{k,l}$ depends on the manifolds $\M_k$ and $\M_l$ through their curvature,  the quantity $\beta$ from Assumption \ref{assump:WellSeparated}, and the curvature \nc of $\M_{kl}$. In particular, if we set $\alpha$ to be such that $\sin(\alpha)\leq \frac{C_{k,l} r}{\veps_+}$, then there is no $(\alpha,r)$-constrained path between points $x $ and $y$.
		
	In addition, suppose that $x,y$ are such that $\dist(y,\M_{l}) \ge r>\dist(x,\M_{k}) $, $\sin(\alpha) \leq \frac{C_{k,l}r }{\veps_+ }$ and $|x-y|\leq \veps_+$. Then the first point of any $(\alpha, r)$-constrained path connecting $x$ and $y$ starting from $x$ must belong to $\M_k$. 
	\end{lemma}
	
	\begin{proof}
		\begin{figure}
		\centering
		\includegraphics[scale=0.5]{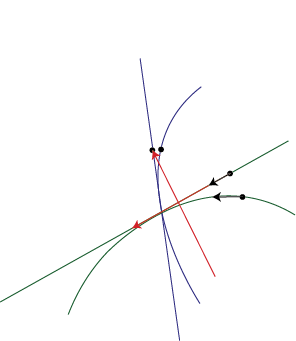}
		\put(-25,65){$x$}
		\put(-45,65){$x_1$}
		\put(-82,93){$y^*$}
		\put(-65,95){$y$}
		\put(-30,80){$x^*$}
		\put(-52,83){$x_1^*$}
		\put(0,50){$\M_l$}
		\put(0,100){$\T\M_l$}
		\put(-55,130){$\M_k$}
		\put(-90,150){$\T\M_k$}
		\put(-82,65){$v_l$}
		\put(-65,80){$v_l^\perp$}
		\caption{Illustration of proof for Lemma \ref{lemma: path algorithm, different manifolds |x-y|<epsilon_+}}
		\label{fig:proof3.7}
	    \end{figure}
		We denote the closest point in $\M_{kl}$ to $y$ as $O_y$ and the closest point in $\M_{kl}$ to $x$ as $O_x$, respectively; uniqueness of these closest points, provided that $\veps_+$ is small enough, is guaranteed by the discussion in Chapter 6 in \citet{lee2003introduction}. Let $\T\M_l$ and $\T\M_k$ be the tangent planes of $\M_l$ and $\M_k$ at $O_x$ and $O_y$, respectively. Let $x^*,y^*$ be the closest points from $x\in \M_l$ and $y\in \M_k$ to $\T\M_l$ and $\T\M_k$, respectively. See an illustration in Figure \ref{fig:proof3.7}. In the remainder of this proof, we assume for the sake of contradiction that there is an $(\alpha, r)$-constrained path between $x$ and $y$.
		
		We start by noticing that $y^*-x^*$ can be decomposed as
		\begin{align}\label{eq-3.7:decomposition:x*-y*}
		   y^*-x^*= y_l^\perp v_l^\perp + y_{l}v_{l}, 
		\end{align}
		for unit vectors $v_l^\perp$ and $v_l$, where $v_l^\perp$ is vertical to $\T\M_l$ and $v_l$ is tangent to $\T\M_{l}$; $y_{l}^\perp$ and $y_{l}$ are non-negative scalars. On the other hand, given the angle constraint between manifolds in the second part of Assumption \ref{assump:WellSeparated}, it is straightforward to show that
		\begin{align}
		    |x-O_x|= \dist(x, \M_{kl})\le C\dist(x,\M_k),
		\end{align}
		for a constant $C$ that depends on $\beta$ (this constant degenerates as $\beta$ approaches $\pi/2$). A similar relation holds for $\dist(y,O_y)$.


		If there exists an $(\alpha,r)$-constrained path connecting $x \in \M_l$ and $y\in \M_k$, then the first segment $\bar{x}=x_1-x$ in this constrained path is such that $x_1\in\M_l$, because $\dist(x,\M_{k})>r$. Denote the closest point to $x_1$ on $\T\M_l$ as $x_1^*$ and in turn define $\bar{x}^*:=x_1^*-x^*$. Notice that $\bar{x}$ satisfies
		\begin{align}\label{eq:xbar and xbarstar}
        \bar{x}=\bar{x}^*+s_x v_x,
		\end{align}
		where $v_x$ is a unit norm vector vertical to $\T \M_l$ and the scalar $s_x$ satisfies
		\begin{align}\label{eq:3.7-s_x}
		    s_x\le C|\bar{x}| \cdot|x-O_x| + C |\bar{x}|^2\le  C|\bar{x}|\cdot\dist(x,\M_k) + C |\bar{x}|^2
		\end{align}
		because locally around the point $O_x$ the manifold $\M_l$ can be represented by a quadratic function as illustrated in Figure \ref{fig:proof_quadratic}.\nc
		\begin{figure}
		\centering
		\includegraphics[scale=0.5]{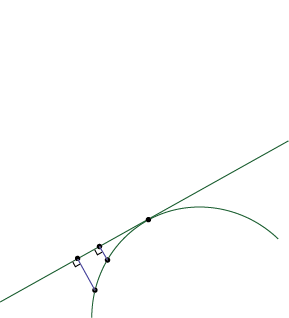}
		\put(-93,49){$z$}
		\put(-80,55){$O_x$}
		\put(-116,43){$\Delta z$}
		\caption{A manifold is locally the graph of a function $f$ which can be approximated by a quadratic function around $O_x$, thus we have $|f(z)-f(z+\Delta z)|\approx |(z-O_x)^2-(z+\Delta z-O_x)^2|\leq C |z- O_x|\cdot |\Delta z| + C |\Delta z |^2$ . Here we have illustrated the case where $f$ is a scalar function, and the general case follows by using the scalar case componentwise.}
		\label{fig:proof_quadratic}
	    \end{figure}
	    \blue
		On the other hand, we have
		\begin{align}\label{eq:y-x and ystar-xstar}
		 y-x=(y-y^*)+(y^*-x^*)+(x^*-x),
		\end{align}
		where 
		\begin{align}\label{eq:|y-y^*|}
		    |y-y^*|\le C|y-O_y|^2\le C \dist(y,\M_l)^2.
		\end{align}
		Indeed, the first inequality follows from the quadratic approximation of the manifold around $O_y$, the fact that $O_y$ is the closest point to $y$ in $\M_{kl}$, and the angle constraint between manifolds in Assumption \ref{assump:WellSeparated}, which allows us to bound $\dist(y, \M_{kl})$ by a constant times $\dist(y, \M_l)$. Notice that $x^*-x\perp \bar{x}^*$ and that $|x-x^*|$ satisfies a similar relationship to $|y-y^*|$, namely 
		\begin{align}\label{eq:|x-x^*|}
		    |x-x^*|\le C|x-O_x|^2\le C \dist(x,\M_k)^2.
		\end{align}
		Now, since $v_x$ is vertical to $\T\M_l$, we have
    	\begin{equation}\label{eq:3.7-y-x^*, v_x}
		\begin{split}
		   \langle y-x^*, v_x\rangle =\langle y-y^*, v_x\rangle+ \langle y^*-x^*, v_x\rangle\le |y-y^*|+y_l^\perp\le C \dist(y,\M_l)^2+y_l^\perp,
		\end{split}
		\end{equation}
		after using \eqref{eq-3.7:decomposition:x*-y*}, $\langle v_l ,v_x \rangle = 0 $, and \eqref{eq:|y-y^*|}.
		
		Since $\bar{x}$ is a segment in the $(\alpha,r)$-constrained path, we must have
		\begin{align}\label{eq:3.7-cosalpha}
		    \cos(\alpha) \le   \frac{\langle y-x , \bar{x} \rangle}{|y-x|\cdot |\bar{x}| }&= \frac{\langle y-x,\bar{x}^*\rangle}{|y-x|\cdot |\bar{x}| }+\frac{\langle y-x,s_x v_x\rangle}{|y-x|\cdot |\bar{x}| }.
		\end{align}
		Using the Cauchy-Schwartz inequality and combining \eqref{eq:xbar and xbarstar}, \eqref{eq:y-x and ystar-xstar}, and the fact that $\bar{x}$ is a segment in the $(\alpha,r)$-constrained path, we deduce
		
		\begin{align}\label{eq:3.7-first term}
		\begin{split}
		     \frac{\langle y-x,\bar{x}^*\rangle}{|y-x|\cdot |\bar{x}| }&=\frac{\langle y-y^*+(y^*-x^*)+(x^*-x) , \bar{x}^* \rangle}{|y-x|\cdot |\bar{x}| }\\
		    &=\frac{\langle y-y^* , \bar{x}^* \rangle }{|y-x|\cdot |\bar{x}|} + \frac{\langle y^* -x^* , \bar{x}^* \rangle }{|y-x|\cdot |\bar{x}|} \\
		    &=  \frac{\langle y-y^* , \bar{x}^* \rangle }{|y-x|\cdot |\bar{x}|} + \frac{\langle y_{l}v_{l} , \bar{x}^* \rangle}{|y-x|\cdot |\bar{x}| }\\
		    &\le  \frac{C\dist(y,\M_l)^2}{|y-x|} + \frac{\langle y_{l}v_{l} , \bar{x}^* \rangle}{|y-x|\cdot |\bar{x}| } \\
		    &\le \frac{C\dist(y,\M_l)^2}{|y-x|} + \frac{|y_{l} |}{|y-x|}
		\end{split}
		\end{align}
			where the second equality is from $\langle x^*-x,\bar{x}^*\rangle=0$, the third equality is from $\langle v_l^\perp, \bar{x}^*\rangle=0$, and the inequalities are from $|\bar{x}^*|\le |\bar{x}|$, Cauchy-Schwartz inequality, and \eqref{eq:|y-y^*|}.
		On the other hand, from Cauchy-Schwartz inequality,
		\begin{equation}\label{eq:3.7-second term}
		\begin{split}
		    \frac{\langle y-x,s_x v_x\rangle}{|y-x|\cdot |\bar{x}| }&=\frac{\langle y-x^*,s_x v_x\rangle}{|y-x|\cdot |\bar{x}| }+\frac{\langle x^*-x,s_x v_x\rangle}{|y-x|\cdot |\bar{x}| }\\
		   &\le\frac{C\left(\dist(y,\M_l)^2+y_l^\perp\right)\cdot (\dist(x,\M_k)+ |\overline {x}|)}{|y-x|}+\frac{C\dist(y,\M_l)^3}{|x-y|}\\
		    &\le \frac{C\dist(y,\M_l)^3}{|y-x|}+\frac{C y_l^\perp\cdot \dist(y,\M_l)}{|y-x|}		\end{split}
		\end{equation}
		where the first inequality is from \eqref{eq:3.7-y-x^*, v_x}, \eqref{eq:|x-x^*|}, the assumption $\dist(x,\M_k)\le \dist(y,\M_l)$, the fact that $|\overline x| \leq r \leq \dist(y, \M_l)$ , and \eqref{eq:3.7-s_x}; the second inequality is from $\dist(x,\M_k)\le \dist(y,\M_l)$ and the fact that $|\overline x| \leq r \leq \dist(y, \M_l)$.
		
		Therefore, combining \eqref{eq:3.7-cosalpha}, \eqref{eq:3.7-first term} and \eqref{eq:3.7-second term}, we obtain
		\begin{align}\label{eq-cosalpha}
		    \cos(\alpha)\le \frac{|y_{l} |}{|y-x|}+\frac{Cy_{l}^\perp }{|y-x|}\dist(y,\M_l)+\frac{C\dist(y,\M_l)^2}{|y-x|}.
		\end{align}
		Notice that we have used the fact that $\dist(y,\M_l)^2 $ dominates $\dist(y,\M_l)^3$. This is the case because $\dist(y,\M_l)\le\eps_+\ll 1$.
		
		From the Pythagorean theorem we have $y_{l}^2 +  (y^\perp_{l}) ^2 = |x^*-y^*|^2$. Therefore, we have
		\begin{equation}\label{eq:y_l}
		   \begin{split}
		    y_l^2&=|x^*-y^*|^2-(y_l^\perp)^2\\
		    &\le (|x-y|+ C\dist(y,\M_l)^2 )^2-(y_l^\perp)^2
		   \end{split}
		\end{equation}
		where the inequality is from \eqref{eq:y-x and ystar-xstar}, \eqref{eq:|y-y^*|}, \eqref{eq:|x-x^*|}, and the assumption that $\dist(y,\M_l)\ge \dist(x,\M_k)$. Also, 
		\begin{equation}\label{eq-3.7:ylperp}
		    \begin{split}
		        y_l^\perp&=\dist(y^*,\T\M_l)\ge \dist(y,\T\M_l)-|y-y^*|\\
		 &\ge\dist(y,\M_l)-|y-y^*|-C|x-y|^2\\
		 &\ge\dist(y,\M_l)-C|x-y|^2
		    \end{split}
		\end{equation}
		where the first inequality is from the triangle inequality; the third inequality is from \eqref{eq:|y-y^*|}; the second inequality is obtained as follows: denote the closest point to $y$ in $\T\M_l$ as $\mathcal{P}_{\T\M_l}(y)$, and let  $\mathcal{P}^{-1}_{\T\M_l}(\mathcal{P}_{\T\M_l}(y))$ be the point in $\M_l$ such that the closest point to $\mathcal{P}^{-1}_{\T\M_l}(\mathcal{P}_{\T\M_l}(y))$ in  $\T\M_l$ is  $\mathcal{P}_{\T\M_l}(y)$; the existence of the point $\mathcal{P}^{-1}_{\T\M_l}(\mathcal{P}_{\T\M_l}(y))$ follows from the local representation of $\M_l$ as a function of points in a neighborhood in $\T\M_l$ around $O_x$. Then,
		\begin{align}
		\begin{split}
		\dist(y,\M_l)-\dist(y,\T\M_l)&\le |y-\mathcal{P}^{-1}_{\T\M_l}(\mathcal{P}_{\T\M_l}(y))|- |y-\mathcal{P}_{\T\M_l}(y)| \\
		&\le|\mathcal{P}^{-1}_{\T\M_l}(\mathcal{P}_{\T\M_l}(y))-\mathcal{P}_{\T\M_l}(y)|\\
		&\le C| \mathcal{P}_{\T\M_l}(y) - O_x|^2
		\\&\le C| \mathcal{P}_{\T\M_l}(y) - \mathcal{P}_{\T\M_l}(x)|^2 + C| \mathcal{P}_{\T\M_l}(x) - O_x|^2
		\\& \leq C| y - x|^2 + C| x  - O_x|^2
		\\&
		\le C|x-y|^2.
		\end{split}
		\label{eq:PTML}
		\end{align}
		Similarly, one can derive an upper bound for $y_l^\perp$ of the form:
		\begin{equation}\label{eq:upper bound y_l^perp}
		    \begin{split}
		        y_l^\perp\le\dist(y,\M_l)+C|x-y|^2.
		    \end{split}
		\end{equation}
		
		Using the condition $1\gg \eps_+\ge \dist(y,\M_l)\ge \dist(x, \M_{k})$, we infer
		\begin{align*}
		    \sin^2(\alpha) &\geq 1-\left(\frac{|y_{l} |}{|y-x|}+\frac{Cy_{l}^\perp }{|y-x|}\dist(y,\M_l)+\frac{C\dist(y,\M_l)^2}{|y-x|}\right)^2 \\
		    &\ge 1- \frac{y_{l}^2 }{|y-x|^2}-\frac{Cy_l y_l^\perp}{|y-x|^2}\dist(y,\M_l) - \frac{C\dist(y,\M_l)^2}{|y-x|} \\
		    &\ge \frac{(y_l^\perp)^2}{|x-y|^2}	-\frac{Cy_l y_l^\perp}{|y-x|^2}\dist(y,\M_l) - \frac{C\dist(y,\M_l)^2}{|y-x|}    \\
		    &\ge \frac{C\dist(y,\M_{l})^2}{|y-x|^2}-C\frac{\dist(y,\M_{l})^2}{|y-x|},
		\end{align*}
		where the first inequality is from \eqref{eq-cosalpha}; we only keep the leading term in the second inequality; in the third inequality we use \eqref{eq:y_l} and $\dist(x,\M_k)\le \dist(y,\M_l)$, and in the last inequality we use \eqref{eq-3.7:ylperp}, \eqref{eq:upper bound y_l^perp}, and $\dist(y,\M_l)\le |x-y|\ll 1$. Noticing that $|y-x|\le \eps_+\ll 1$, we can further simplify the above inequality as
		\begin{align*}
		    \sin(\alpha) \geq \frac{C\dist(y,\M_{l})}{|y-x|}.
		\end{align*}
		As a consequence, if the above relationship is not satisfied, there can not exist an $(\alpha, r)$-constrained path between $x$ and $y$. This completes the proof of the first part.
		
		For the second part, we assume for the sake of contradiction that there is an $(\alpha, r)$-constrained path between $x$ and $y$ such that the first step (starting from $x$) in the path belongs to $\M_l$; we call this first step $x_1$. Since the only condition used for $\dist(x,\M_k)$ in the first part is that $\dist(x,\M_k)\le \dist(y,\M_l)$, and this also holds for $\dist(x,\M_k)<r\le \dist(y,\M_l) $, we can then repeat the same argument as above with $x'=x_1-x$ to conclude that if $\sin(\alpha) < \frac{C\dist(y,\M_{l})}{|y-x|}$, then we would reach a contradiction. 
		
	\end{proof}
	\nc








}

	\blue 
	The next lemma helps us justify why, for the multi-manifold clustering problem, choosing $\veps_-$ of the same order as $\veps_+$ is better than choosing $\veps_- =0$ (or in general $\veps_-$ much smaller than $\veps_+$). Intuitively, as illustrated in Figure \ref{fig:epsminusisimportant}, by directly omitting connections between points that are too close to each other we can remove edges between points on different manifolds that the angle constraint condition may not be able to remove. We quantify the gain of considering this step in the next lemma.

	

	\begin{figure}
		\centering
		\includegraphics[scale=0.5]{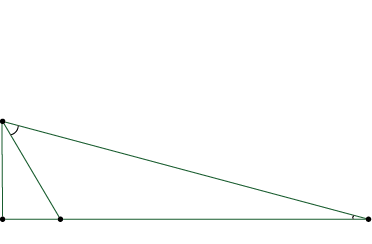}
		\put(-157,0){$x_1$}
		\put(-175,46){$\vartheta$}
		\put(10,13){$y$}
		\put(-200,70){$x$}
		\caption{\blue When $\dist(x,\M_k)\le\dist(y,\M_l)$ and the first step $x_1$ of a constrained path is on $\M_k$, the angle  $\angle x_1xy$ is larger than the angle $\angle x_1yx$. This means that the value of $\angle x_1xy$ dictates whether the path satisfies the angle constraints or not. In turn, we see that for points $x, y$ with a larger value of $|x-y|$ this angle will be larger than when $|x-y|$ is smaller, making it easier for the angle condition to detect that $x,y$ are in different manifolds when their distance is larger.\nc} 
			\label{fig:epsminusisimportant}
	    \end{figure}

	\new{
	\begin{lemma}\label{Lemma: path algorithm, different manifolds |x-y|>epsilon_-}
		
		Suppose that $\mathcal{M}_l \cap \M_k\not = \emptyset$ and let $x\in \M_l$ and $y \in \M_k$ be such that $\veps_- < |x-y|<\veps_+$, where $\veps_-= c \veps_+$. Let $\alpha>0$ be such that
		\[ \sin(\alpha) <  \frac{Cr}{\veps_+}.   \]
		If $\dist(x, \M_{k})>C(\frac{r^2}{\eps_+}+r\eps_+ )$ and $\dist(y, \M_{l})> r$, then there can not exist an $(\alpha,r)$-constrained path between $x$ and $y$.
	\end{lemma}
	
	    \begin{proof}
	    \begin{figure}
		\centering
		\includegraphics[scale=0.5]{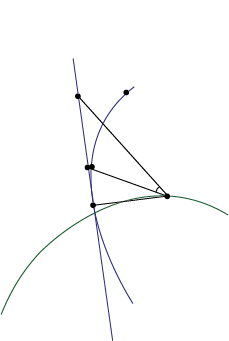}
		\put(-25,65){$x$}
		\put(-45,124){$y$}
		\put(-88,122){$y^*$}
		\put(-83,85){$x_1^*$}
		\put(-80,65){$z$}
		\put(-48,78){$\vartheta$}
		\put(5,55){$\M_l$}
		\put(-65,-15){$\T\M_k$}
		\put(-45,5){$\M_k$}
		\caption{Illustration of proof for Lemma \ref{Lemma: path algorithm, different manifolds |x-y|>epsilon_-}}
		\label{fig:proof3.8}
	    \end{figure}
	If $\dist(x, \M_{k}) >r$, then by Lemma \ref{lemma: path algorithm, different manifolds |x-y|<epsilon_+} there can not exist an $(\alpha, r)$-constrained path between $x,y$. Thus, without the loss of generality we can assume that $\dist(x, \M_{k}) \leq r $. 
	
	    Assume for the sake of contradiction that there is an $(\alpha,r)$-constrained path between $x$ and $y$. Denote the closest point to $x$ in $\M_{kl}$ as $O_x$, and let $\T\M_k$ be the tangent plane to $\M_k$ at $O_x$ (notice that this definition is different from the one in Lemma \ref{lemma: path algorithm, different manifolds |x-y|<epsilon_+}). Let $y^*$ be the closest point to $y$ in $\T\M_k$, and let $z$ be the closest point to $x$ in $\T\M_k$. Let $t:= |x-z|$ and $d:=|y^*-x|$; see an illustration in Figure \ref{fig:proof3.8}. Notice that we have
	    \begin{align}\label{eq-3.8:lower bound for d}
	        |y^*-x|\ge |y-x|-|y^*-y|>\veps_- - C\eps_+^2\ge C\eps_+,
	    \end{align}
	    due to the fact that
	    \begin{align}
	    \begin{split}
	   |y - y^*|&\leq C |y^* - O_x|^2\leq C |y^* - z|^2 + C | z- O_x|^2
	   \\& \leq C |x-y|^2  + C|x - O_x|^2 
	   \\&= C |x-y|^2  + C\dist(x, \M_{kl})^2
	   \\ & \leq C |x-y|^2  + C\dist(x, \M_k)^2
	   \\ & \leq C|x-y|^2 \leq C \veps_+^2.
	   \end{split}
	   \label{eq:yy*}
	    \end{align}
	     By the second part of Lemma \ref{lemma: path algorithm, different manifolds |x-y|<epsilon_+} we know that the first step in the constrained path (starting from $x$) must lie in $\M_k$; we denote by $x_1$ this first step and let  $x_1^*$ be the closest point to $x_1$ in $\T\M_k$. We have
	    \begin{align}\label{eq-3.8:t}
	    t = |x-z| =\dist(x, \T\M_{k}) \leq |x-x_1^*| \leq |x-x_1|+|x_1-x_1^*| \leq r+Cr^2\le C r.
	    \end{align}
	    In fact, $|x_1- x_1^*|$ can be bounded by $C|x_1- x|^2$, as it follows from the next computation:
	    \begin{equation}\label{eq-3.8:|x_1-x_1^*|<=Cr|x-x_1|+somthing}
	        \begin{split}
	            |x_1-x_1^*|\le C|x_1-O_x|^2&\le C|x-x_1|^2+C|x-O_x|^2\\
	            &= C |x-x_1|^2+C\dist(x, \M_{kl})^2\\
	            &\le C |x-x_1|^2+C \dist(x, \M_{k})^2\\
	            &\le C|x-x_1|^2.
	        \end{split}
	    \end{equation}
	  In particular, we also have
	    \begin{align}\label{eq-3.8:|x_1-x_1^*|<=Cr|x-x_1|}
	        |x_1-x_1^*|\le Cr|x-x_1|.
	    \end{align}
	    \blue 
	    We will also use the following inequality:
	    \begin{align}
	    \label{eq:320}
	       \begin{split}
	       |x-x_1^*|^2&=|x_1^* - z|^2 + |x-z|^2 \leq |x_1 -z|^2  + 2 |x-x_1|^2 + 2|x_1 - z|^2
	       \\ &\leq  2|x-x_1|^2 + 3|x_1 - x_1^* |^2 + 3|x_1^* -z|^2 \leq C |x-x_1|^2 +3|x_1-x_1^*|^2
	       \\&\leq C |x-x_1|^2 + C|x-x_1|^4 \leq C |x-x_1|^2,
 	       \end{split} 
	    \end{align}
	    where the second to last inequality follows from \eqref{eq-3.8:t} and \eqref{eq-3.8:|x_1-x_1^*|<=Cr|x-x_1|+somthing}. 
	  
	    The angle condition for the constrained path then gives
	    \begin{equation}\label{eq-cosalpha<cosvartheta}
	        \begin{split}
	            \cos\alpha&\le \frac{\langle y-x,x_1 - x \rangle}{|y-x|\cdot |x_1 - x|}\\
		        &=\frac{\langle y^*-x,x^*_1 - x \rangle}{|y-x|\cdot |x_1 - x|}+\frac{\langle y-y^*,x^*_1 - x \rangle}{|y-x|\cdot |x_1 - x|}+\frac{\langle y^*-x,x_1-x_1^* \rangle}{|y-x|\cdot |x_1 - x|}+\frac{\langle y-y^*,x_1-x_1^* \rangle}{|y-x|\cdot |x_1 - x|}\\
		        &\le \frac{|y^*-x|\cdot|x_1^*-x|}{|y-x|\cdot|x_1-x|}\cos\vartheta +\frac{\langle y-y^*,z - x \rangle}{|y-x|\cdot |x_1 - x|} +Cr\\
		        &\le \frac{|y^*-x|\cdot|x_1^*-x|}{|y-x|\cdot|x_1-x|}\cos\vartheta + C\frac{|y-x|  t}{ |x-x_1^*|}+Cr\\
		        &\le \frac{d}{|y-x|}\cos\vartheta  + C\frac{ |y-x| t}{ |x-x_1^*|}+Cr,
	        \end{split}
	    \end{equation}
	    where $\vartheta$ is the angle between the vectors $y^*-x$ and $x_1^*-x$ as illustrated in Figure \ref{fig:proof3.8}. The second inequality is from the fact that $\langle z-x_1^*,y-y^* \rangle=0$, \eqref{eq:yy*}, and \eqref{eq-3.8:|x_1-x_1^*|<=Cr|x-x_1|}. The third inequality is from \eqref{eq:yy*} and \eqref{eq-3.8:|x_1-x_1^*|<=Cr|x-x_1|}; the last inequality is from the following:
	    \begin{align*}
	        \left|\frac{|y^*-x|\cdot|x_1^*-x|}{|y-x|\cdot|x_1-x|}-\frac{d}{|x-y|}\right|&=\left|\frac{|y^*-x|(|x_1^*-x|-|x_1-x|)}{|x-y|\cdot|x_1-x|} \right|\\
	        &\le \frac{|y^*-x|\cdot|x_1-x_1^*|}{|x-y|\cdot|x_1-x|}\\
	        &\le \frac{\left(|y-x|+|y-y^*|\right)|x_1-x_1^*|}{|x-y|\cdot|x_1-x|}\\
	        &\le \frac{C\left(|y-x|+|y-y^*|\right)r}{|x-y|}\le Cr.
	    \end{align*}
	    
	    In turn, we have
	    \begin{align}\label{eq:d<t+}
	        d-|x-y|\le |z-y^*|+|z-x|-|x-y|\le t.
	    \end{align}
	    Combining \eqref{eq-cosalpha<cosvartheta} and \eqref{eq:d<t+}, we conclude that
	    \begin{align*}
	        \cos\alpha\le \cos\vartheta + \frac{Ct}{|y-x|}  + C\frac{ |y-x| \nc t}{|x-x_1^*|}+Cr.
	    \end{align*}
	    This implies
	    \begin{align}\label{eq:sin2alpha}
	        \sin^2\alpha \ge \sin^2\vartheta - \frac{Ct}{|y-x|}  - C\frac{ |y -x| \nc t}{|x-x_1^*|} - Cr,
	    \end{align}
	    where we drop some lower order terms.
	    
	    Since  $ s - Cs^3\leq \sin(s)\leq s $ when $s\geq 0$ is small enough, we can write \eqref{eq:sin2alpha} as,
	    \begin{align*}
	        \vartheta -  C \vartheta^3 \le \sqrt{\alpha^2 + \frac{Ct}{|y-x|}  + C\frac{|y-x|  t}{ |x-x_1^*|}+Cr} 
	    \end{align*}
	    Let us denote $\sqrt{\alpha^2 + \frac{Ct}{|y-x|}  + C\frac{ |y-x| t}{|x-x_1^*|}+Cr}  $ by $\alpha_0$.
	    
	    From a simple geometric observation (just consider the triangles $\triangle xy^* z $ and $\triangle x_1^* xz$), we have
		\[ \alpha_0 +  C \vartheta^3 > \vartheta  \geq  \arccos\frac{t}{d}-\arccos\frac{t}{|x-x_1^*|} =: f(d,t).\]
		
		When fixing $|x-x_1^*|$, the function $f(d,t)$ is strictly increasing in both coordinates because $\partial_d f(d,t)=\frac{t}{d^2}\cdot\frac{1}{\sqrt{1-\frac{t^2}{d^2}}}>0$ and $\partial_t f(d,t)=\frac{1}{\sqrt{|x-x_1^*|^2-t^2}}-\frac{1}{\sqrt{d^2-t^2}}>0$. In particular, $f(d, t )\geq f(C\veps_+ , t)$ because of \eqref{eq-3.8:lower bound for d}. We see that if $t > t_0$, where 
		\[ t_0:= \frac{C|x-x_1^*|\eps_+ \sin(\alpha_0 + C\vartheta ^3 )}{\sqrt{(\veps_- - C\eps_+^2)^2 - 2C |x-x_1^*|\eps_+\cos(\alpha_0+  C\vartheta ^3 \nc) + |x-x_1^*|^2}}, \]
		then
		\[ \alpha_0 +  C \vartheta^3  \geq  f(d,t)\geq f(C\eps_+, t)> f(C\eps_+ ,t_0) = \alpha_0 +  C\vartheta ^3,  \]
		and thus we would reach a contradiction. A simpler upper bound for $t_0$ is the following:
		\begin{equation*}
		t_0 <\frac{C|x-x_1^*|\veps_+\sin(\alpha_0)}{\veps_+\cos(\alpha_0)-|x-x_1^*|}\le C|x-x_1^*|\sin(\alpha_0).
		\end{equation*}
		In particular, if $t > C|x-x_1^*|\sin(\alpha_0) $, then there can not be an $(\alpha, r)$-constrained path between $y$ and $x$. We can rewrite this condition as
		\begin{align}\label{condition-t}
		    t> C|x-x_1^*|(\frac{r}{\eps_+} +\frac{\sqrt{t}}{\sqrt{|y-x|}} + \frac{ \sqrt{|y-x|}  \sqrt{t}}{\sqrt{|x-x_1^*|}} + \sqrt{r})
		\end{align}
		where we have used the assumption $\sin\alpha\le C\frac{r}{\eps_+}$. Notice that the right hand side of the inequality is an increasing function with respect to $|x-x_1^*|$, so we can replace $|x-x_1^*|$ with the upper bound $Cr$ as in \eqref{eq:320}. Also, using $\eps_-\le |y-x|\le \eps_+$ and $\eps_-=c\eps_+$, we can change \eqref{condition-t} to
		\begin{align}
		    t> Cr(\frac{r}{\eps_+} +\frac{\sqrt{t}}{\sqrt{\eps_+}} + \frac{ \sqrt{t\eps_+}}{\sqrt{r}} + \sqrt{r}).
		\end{align}
		This in turn can be changed to
		\begin{align}
		    t>C(\frac{r^2}{\eps_+}+r\eps_+ ),
		    \label{eq:Condt}
		\end{align}
		after using the fact that $\frac{r^2}{\eps_+}+r\eps_+\ge 2 r\sqrt{r}$. In summary, what we have shown is that if $t= \dist(x, \T\M_k) > C (\frac{r^2}{\veps_+} + r \veps_+) $, then there can not be a constrained path between $x$ and $y$.

		To finalize the proof, we must now find a condition on $\dist(x, \M_k)$ that implies \eqref{eq:Condt}. For this purpose, we use notation analogous to the one in \eqref{eq:PTML} and compute
		\begin{align*}
		\dist(x,\M_k)-\dist(x,\T\M_k)&\le |x-\mathcal{P}^{-1}_{\T\M_k}(z)|- |x-z| \\
		&\le|\mathcal{P}^{-1}_{\T\M_k}(z)-z|\le C|z-O_x|^2\le Cr^2,
		\end{align*}
		where the last inequality is because $|z-O_x|\le |x-O_x|\le  C\dist(x,\M_k)\le Cr$. Given that $r^2 \ll r\eps_+$, we conclude that
		\begin{align*}
		    \dist(x,\M_k) > C(\frac{r^2}{\eps_+}+r\eps_+ )
		\end{align*}
		implies \eqref{eq:Condt}, completing in this way the proof.

	\end{proof}
}

	\blue

	\nc
	
	With Theorem \ref{Proposition: path algorithm, inner fully connected} and Lemmas \ref{lemma: path algorithm, different manifolds |x-y|<epsilon_+} and \ref{Lemma: path algorithm, different manifolds |x-y|>epsilon_-} in hand we can now prove the main result in this section.

	\begin{proof}[Proof Theorem \ref{theorem: path algorithm}]

		\begin{enumerate}
			\item For arbitrary $\veps_- \leq (1/4)\veps_+$ we can use Lemma \ref{lemma: path algorithm, different manifolds |x-y|<epsilon_+} and a standard concentration bound to see that with probability no less than  
			\[ 1- C_{k,l} n\exp(-C_{k,l} n   \min\{ r^{m_k - m_{kl}}, r^{m_l - m_{kl}}, \veps_+^{m_k}, \veps_+^{m_l}   \}  , \]
			$N_{kl}$, the number of connections between points in $X \cap \M_k$ and $X \cap \M_l$, satisfies
			\[ N_{kl} \leq C_{k,l} n^2\max \{ r^{m_k- m_{kl}} \veps_+^{m_l},   r^{m_l- m_{kl}} \veps_+^{m_k}   \}.   \]
			
			\item When we have the lower bound $\veps_-\geq c \veps_+$ we can proceed as above but now using Lemma \ref{Lemma: path algorithm, different manifolds |x-y|>epsilon_-}.
		\end{enumerate}

	\end{proof}

	\begin{remark}
	\blue 

	   	According to Theorem \ref{theorem: path algorithm}, to satisfy the sparse outer connectivity when $\veps_-= c \veps_+$ one needs 
			\begin{align}\label{eq:mixed dimension outer sparse connectivity condition}
			    \frac{n^2\max \{ (\frac{r^2}{\veps_+} +r\eps_+  )^{m_k - m_{kl}} \veps_+^{m_l},   (\frac{r^2}{\veps_+} +r\eps_+  )^{m_k - m_{kl}} \veps_+^{m_k}   \} }{n^2\eps_+^{m+2}}\to 0.
			\end{align}
			In the worst case for \eqref{eq:mixed dimension outer sparse connectivity condition}, one requires 
			\begin{align}
			\label{eq:Condrveps}
			    r\ll \min\{\eps_+^{m+1-\min_l m_l},\eps_+^{\frac{m+3-\min_l m_l}{2}}\},
			\end{align}
			which is a quite restrictive condition when there is a large discrepancy between the dimensions of the manifolds, as one would require a very small value of $r$ and in turn a very large number of data points for condition \eqref{eq:Condrveps} to be satisfied while simultaneously satisfying the inner connectivity condition. The above estimate, however, is quite pessimistic, and in particular assumes that all manifolds intersect with each other. In general, one can replace $\min_{l} m_l$ with the minimum dimension of manifolds that actually intersect the manifolds with larger dimension. If the gap between $m$ and this restricted minimum is not too large, from moderate number of samples we would expect the spectral clustering algorithm with path constraints to be able to separate the data coming from manifolds with larger dimension from the rest of the data set. At that stage, one can consider a new iteration of the algorithm, this time with a data set with fewer points and with a smaller largest dimension. The exploration of this iterative pruning strategy is beyond the scope of this work. 
			

	    \nc 
	\end{remark}

	\subsection{A local PCA approach to MMC}\label{Section: Plane Algorithm with epsilon_+,epsilon_- graph setting}

	An alternative spectral approach to the multimanifold clustering problem that is popular in the literature (e.g. see \citet{arias2017spectral}) is based on building weights $\omega_{ij}$ that depend on the level of alignment of local tangent planes around nearby data points. To be precise, as in the path-construction from section \ref{sec:GraphConstruct} we only consider giving an edge to a pair of data points $x_i, x_j$ if $\veps_- < |x_i- x_j|< \veps_+$. If this condition is satisfied, we then set $\omega_{x_ix_j}= 1$ provided that the angle between $\hat{T}_{x_i}$ (a local tangent plane around $x_i$) and $\hat{T}_{x_j}$ is smaller than a certain threshold, and otherwise we set $\omega_{x_ix_j}=0$. These local ``tangent" planes can be constructed from the observed data using local PCA. Namely, the idea is to run PCA with the data set $X \cap B(x_i, r)$ for some small enough $r$ in order to obtain a collection of principal directions which are then used as generators for the plane $\hat{T}(x_i)$; see \citet{arias2017spectral} for more details.

	Using the estimates from \citet{arias2017spectral} (and some additional computations) it is possible to show that the local PCA graph construction satisfies the sparse outer connectivity condition (with very high probability), provided that the parameter $r$ is tuned appropriately. However, from a theoretical perspective, one should not expect that the full inner connectivity holds with high probability. This is an observation already made in \citet{arias2017spectral} (although not with the exact same words). Indeed, let $x_i$ be a point in the manifold $\M_l$ that is close to the intersection of $\M_l $ and $\M_k$ (closer than $r$). For points $x_j\in \M_l$ within distance $\veps_+$ from $x_i$ and away enough from the intersection $\M_k\cap \M_l$, we expect their PCA-based tangent planes to resemble those of the actual manifold $\M_l$ at those same points (if $r$ has been chosen so that there is consistency in the approximation of tangent planes). However, $x_i$'s empirical tangent plane will be influenced by the presence of the points in $\M_k$ that belong to the ball $B(x_i, r)$ and thus one should not expect this plane to be aligned with the planes of all the other points $x_j\in \M_l$ lying nearby. In contrast, the full inner connectivity for the path-based graph construction from section \ref{sec:GraphConstruct} just depends on the points on each single manifold: having additional points can only help with the full inner connectivity (more points means more possible paths) but never tamper with it. 
	
	One of the implications of the above discussion is that the local PCA approach to MMC may in principle produce more clusters than desirable, and for example groups of points that lie close to the intersection of two manifolds may form their own clusters; see the discussion in section 3 in \citet{arias2017spectral}. It is thus possible that some of the manifolds get split into different components and in particular one may not be able to recover the multi-manifold structure underlying the data without information on the actual location of the manifolds' intersections.

	\section{Numerical experiments}\label{section: discussion}

	The purpose of this section is twofold. On the one hand, we want to explore the limitations and difficulties that may arise when using the MMC methodologies based on spectral clustering \blue with path-based graphs that we have introduced in section \ref{sec:GraphConstruct}. \nc  On the other hand, we want to provide further insights into the theoretical results that we have presented throughout the paper. We present a series of numerical experiments aimed at achieving these two goals. \new{In addition, at the end of this section we compare the performance of spectral clustering using path-based graphs with other spectral-based algorithms by testing them on synthetic and real data sets.}

	\blue 

In our experiments, we consider our graph construction directly as presented in section \ref{sec:GraphConstruct}, or in its \textit{nearest neighbor} version, where we change all parameters that have a lengthscale interpretation with parameters that specify the number of neighbors to a point. Precisely, instead of fixing the two length scales $\veps_+, \veps_-$, we can alternatively fix two natural numbers $k_+ >k_-$ and substitute the conditions $|y_1 - y_2| > \veps_+$ and $|y_1 - y_2| < \veps_-$ in Algorithm \ref{algorithm: Path Algorithm} with the conditions ``neither $y_1$ is one of the $k_+$-nearest neighbors of $y_2$, nor $y_2$ is one of the $k_+$-nearest neighbors of $y_1$" and ``$y_1$ is one of the $k_-$ nearest neighbors of $y_2$ or viceversa", respectively. Likewise, the lengthscale $r$ is substituted with a parameter $\kappa\in \N$, and the second condition in \eqref{def:angle constraint} is changed to ``$x_{i_j}$ is one of the $\kappa$ nearest neighbors of $x_{i_{j+1}}$ or viceversa". Unless otherwise noted, whenever we use the nearest neighbor version of our algorithm we will select $k_-=2/3 k_+$ and tune $k_+$ in order to minimize the misclustering rate of the output clusters. In the toy examples where we use the $(\veps_+, \veps_-)$ version of our algorithm, we tune $\veps_+$ and $\veps_-$ so that $v_m \veps_+^m  =: k_+ \in \N$ and $v_m \veps_-^m  =: k_- \in \N$, where $v_m$ is the volume of the unit ball in $\R^m$. The other parameters in the algorithm, $\alpha$ and $\kappa$ (or $r$), are tuned to minimize the misclustering rate. \footnote{The implementation of our algorithm can be found in \href{https://github.com/chl781/manifold-clustering}{github.com/chl781/manifold-clustering}}

Before we proceed with our experiments, we discuss the theoretical computational complexity of building angle-constrained path proximity graphs in their nearest neighbor version, where we can more directly quantify the contributions of the different steps in the construction.  
Let $|V|$ be the number of neighbors around a point in the base proximity graph, and let $|E|$ be the number of edges among these neighbors. The computational complexity of Algorithm \ref{algorithm: Path Algorithm} is ${\mathcal{O} ((|V|+|E|)\log |V|)}$, which is essentially the same as the complexity of Dijkstra's algorithm using the Fibonacci heap \citet{fredman1987fibonacci}. Therefore, the total computational complexity of constructing $k_+,k_-$-graph with angle constraints is $\mathcal{O}\left(n (k_+ - k_-)\kappa k_+\log(\kappa k_+) +n^2\log k_+\right)$, where $\mathcal{O}(n^2\log k_+)$ is the computational cost of constructing the $k_+,k_-$-nearest neighbor base graph. By using the adapted graph construction in section \ref{sec:other path algorithms}, it is possible to speed up the construction to $\widetilde{\mathcal{O}}\left( n\kappa k_+ + n^2\log k_+  \right)$. If the parameters are chosen as suggested in Remark \ref{rem:rate} for the setting of manifolds with the same dimension $m$, i.e. we use $v_m \veps_+^m  = k_+ $, $v_m \veps_-^m  = k_-$, and $v_m r^m = \kappa$, then the computational complexity of the adapted method and Algorithm \ref{algorithm: Path Algorithm} are, in terms of the total number of data points, $\widetilde{\mathcal{O}}\left( n^{\frac{6m-3}{3m-1}}+n^2\right)$ and $\widetilde{\mathcal{O}} \left(n^{\frac{8m-4}{3m-1}}\right)$, respectively. In contrast, the computational complexity of building a vanilla $k$-nearest neighbor graph is $\mathcal{O}(n^2 \log k)$ by using a priority queue structure. Therefore, by using the adapted structure for the angle-constrained path construction, we can build graphs at the same computational complexity as the one for vanilla $k$-nearest neighbor graphs. 
It is worth highlighting that once the similarity matrix has been constructed, the computational complexity for the eigendecomposition needed to run spectral clustering will typically depend on the level of sparsity of the input weight matrix $\omega$. In this regard, it is important to notice that the angle-constrained graph will always be sparser than its base proximity graph.  Finally, we remark that our algorithm may not be as efficient in practice as the previous theoretical analysis would suggest because, in general, we need to use denser graphs than if no curvature constraints were imposed. In this regard, the use of landmark points in the algorithm by \citet{BABAEIAN2015118} is an alternative to speed up the computation, although at the expense of weaker theoretical consistency guarantees.


	\nc


	
	
	\nc
	
	\subsection{Bottlenecks and multiple manifolds}\label{sec:bottlenecks and multiple manifolds}

	Our theoretical results imply that the spectra of suitable graphs resemble the spectrum of a tensorized Laplacian on the union of smooth manifolds underlying the data set. In particular, when using spectral clustering on finite data sets with a multi-manifold structure, it is possible to obtain a partition of the data into multiple smooth manifolds and/or into regions that are separated by thin bottlenecks. In this section we explore numerically the ``confounding" role that bottlenecks may play in MMC.  
	
	First, let us consider the bottle and plane example illustrated in Figures \ref{fig: bottleplane middle 2}, \ref{fig: bottleplane middle 3}, and \ref{fig: bottleplane middle rotated}. There, data set $X$ is sampled uniformly from the set $\M=\M_1 \cup \M_2$, where $\M_1$ is a plane and $\M_2$ is a 2-dimensional dumbbell with a bottleneck at its center. A graph $(X, \omega)$ has been constructed as in section \ref{sec:GraphConstruct} for appropriate values of $\veps_-, \veps_+, r, \alpha$. Intuitively, this graph should help identify the two manifolds given that they meet perpendicularly (i.e., $\beta$ is zero in \eqref{eqn:AngleConstraint}). On the other hand, the same graph captures the internal geometry of $\M_2$ and thus should also detect the bottleneck in $\M_2$. Figure \ref{fig: bottleplane middle 2} shows the sign of the first non-trivial eigenvector of the graph Laplacian, which, as we can observe from the picture, is able to detect the bottleneck. Figure  \ref{fig: bottleplane middle 3} shows the sign of the second non-trivial eigenvector (orthogonal to the first non-trivial eigenvector). In our experiments, our graph Laplacian's first two non-zero eigenvalues are close to zero, and their relative difference is quite small compared to the relative difference between the second and third non-zero eigenvalues. The partition illustrated in Figure \ref{fig: bottleplane middle 3} is not directly interpretable. However, when considering a suitable linear combination of the first and second non-trivial eigenvectors, we recover the partition illustrated in Figure \ref{fig: bottleplane middle rotated} which correctly separates the two manifolds. This linear combination is obtained by minimizing the \textit{Ratio cut} functional (see \citet{DBLP:journals/corr/abs-0711-0189} for a definition) among all the partitions induced by norm one linear combination of the first two non-trivial eigenvectors. In this case, it is a simple one-dimensional search. 
	
	This example illustrates that bottlenecks are indeed confounders for MMC when using spectral methods. Still, even in the presence of competitor bottlenecks, we see that the graph Laplacian's spectrum possesses the information needed to recover the desired partition of the data, and the combination of spectral clustering with Ratio cut minimization is shown to help in the detection of the desired partition. Warm start initialization for balanced cut minimization using spectral clustering has been considered in the literature before (e.g. \citet{bresson2012,bluv13,bresson2013adaptive,Thomas1,bresson2012multi}).

	\begin{figure}[htbp]
		\par\medskip
		\centering
		\begin{minipage}[t]{0.3\textwidth}
			\centering
			\includegraphics[width=3.5cm]{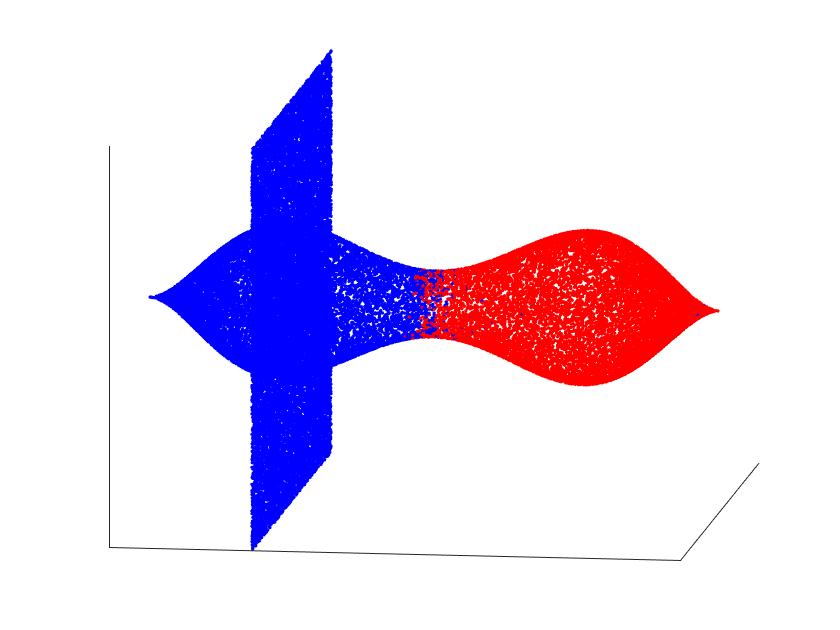}
			\caption{Ratio cut: 0.129}
			\label{fig: bottleplane middle 2}
		\end{minipage}
		\begin{minipage}[t]{0.3\textwidth}
			\centering
			\includegraphics[width=3.5cm]{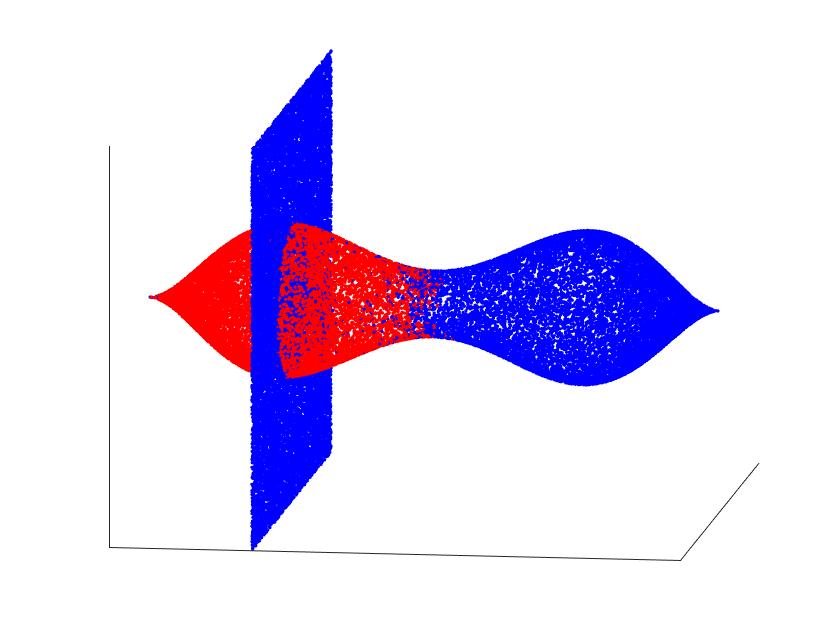}
			\caption{Ratio cut: 0.240}
			\label{fig: bottleplane middle 3}
		\end{minipage}
		\begin{minipage}[t]{0.3\textwidth}
			\centering
			\includegraphics[width=3.5cm]{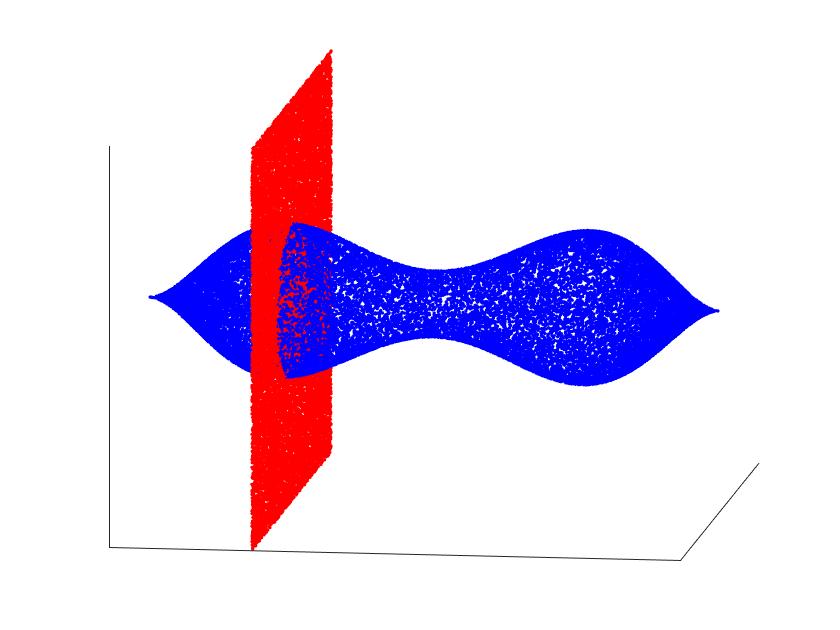}
			\caption{Ratio cut: 0.083}
			\label{fig: bottleplane middle rotated}
		\end{minipage}
	\end{figure}


	Another example where multiple manifolds and bottlenecks are present is the one illustrated in Figures \ref{fig: dollar sign fail} and \ref{fig: dollar sign success} which we will refer to as the \textit{dollar sign} example. We again build the graph Laplacian as in section \ref{sec:GraphConstruct}. Figure \ref{fig: dollar sign fail} shows the sign of the first non-trivial eigenvector, which, as we can observe, can detect the ``bottleneck" at the center of the dollar sign shape. Figure \ref{fig: dollar sign success} shows the sign of the second non-trivial eigenvector. Notice that the multi-manifold structure is identified correctly using this eigenvector. In this example, the partition induced by the second non-trivial eigenvector is a minimizer of the Ratio cut functional among partitions induced by linear combinations of the first two non-trivial eigenvectors. For comparison, in Figures \ref{fig: epsilon 2} and \ref{fig: epsilon 3} we illustrate the partitions induced by the first and second non-trivial eigenvectors of a graph Laplacian on a standard $\veps$-graph with no path constraints. We can see that with that graph construction we can not retrieve the desired multimanifold structure. 
	
	\begin{figure}[htbp]
		\centering
		\begin{minipage}[t]{0.2\textwidth}
			\centering
			\includegraphics[width=4cm]{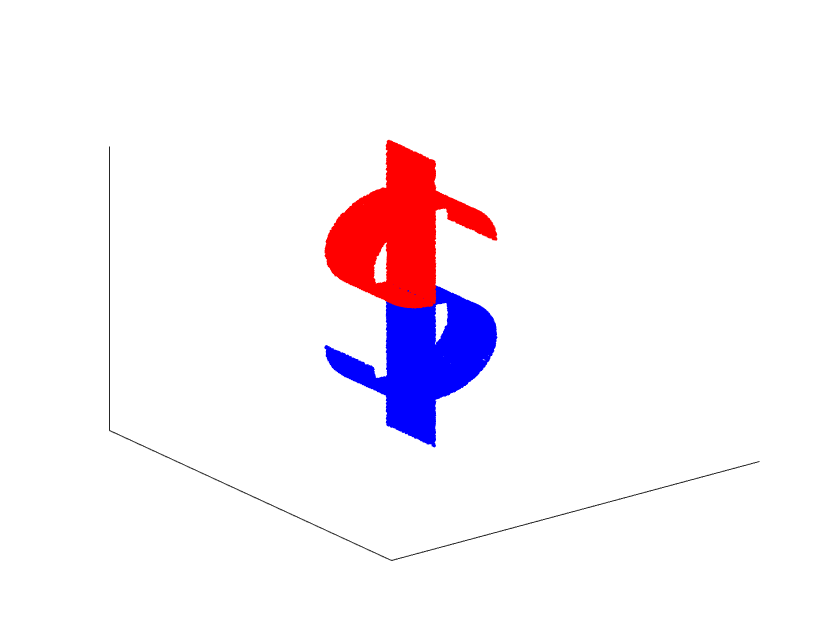}
			\caption{}
			\label{fig: dollar sign fail}
		\end{minipage}
		\begin{minipage}[t]{0.2\textwidth}
			\centering
			\includegraphics[width=4cm]{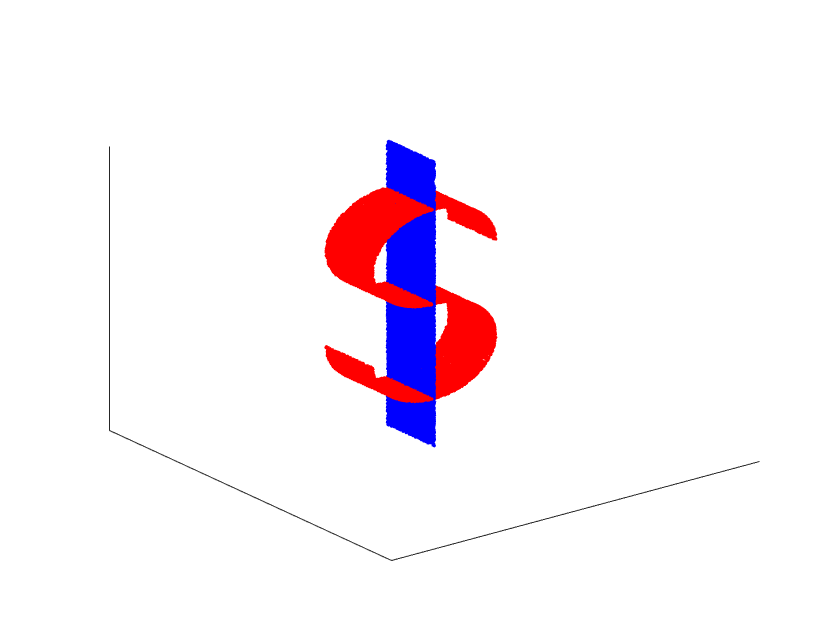}
			\caption{}
			\label{fig: dollar sign success}
		\end{minipage}
		\begin{minipage}[t]{0.2\textwidth}
			\centering
			\includegraphics[width=4cm]{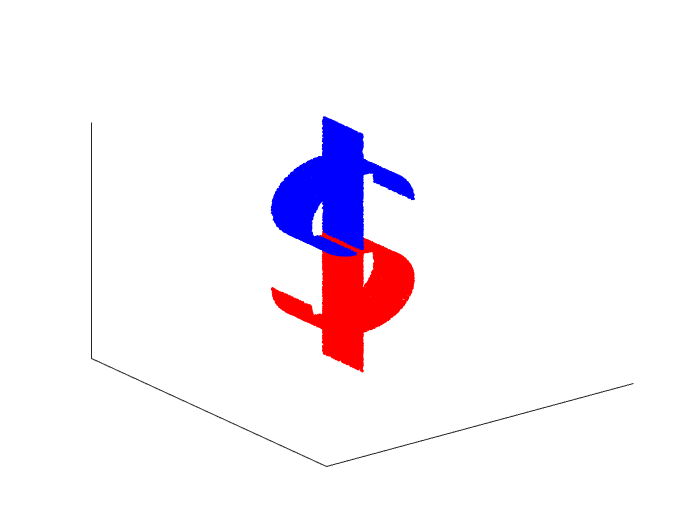}
			\caption{}
			\label{fig: epsilon 2}
		\end{minipage}
		\begin{minipage}[t]{0.2\textwidth}
			\centering
			\includegraphics[width=4cm]{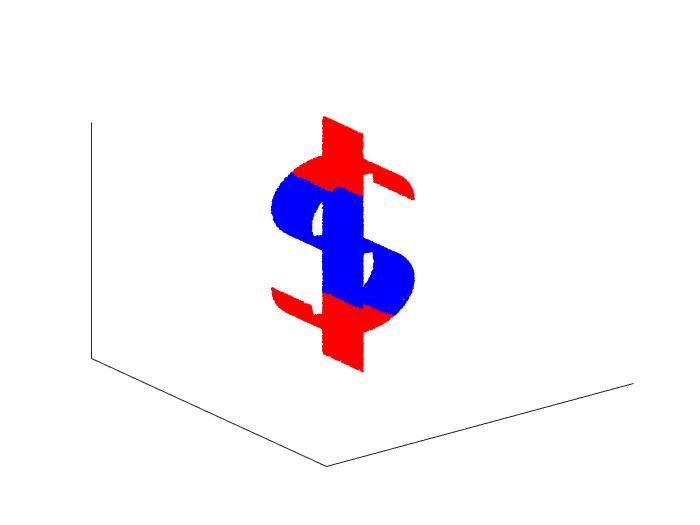}
			\caption{}
			\label{fig: epsilon 3}
		\end{minipage}
		\caption*{In Figures \ref{fig: dollar sign fail} and \ref{fig: dollar sign success} 2nd and 3th eigenvectors from annular graph with angle constraints. In Figures \ref{fig: epsilon 2} and \ref{fig: epsilon 3} 2nd and 3th eigenvectors from standard $\veps$-proximity graph.}
	\end{figure}

	\begin{remark}
		When using ratio cut for MMC the important energy to consider is the cut functional/total variation functional:
		\[ TV_n(u) \sim \frac{1}{n^2\veps^{m+1}_+} \sum_{i,j} \omega_{ij}|u(x_i) - u(x_j)|. \]
		Notice that the correct scaling factor $\frac{1}{n^2\veps_+^{m+1}}$ for $TV_n$ is different from the one for the graph Dirichlet energy which scales like $\frac{1}{n^2\veps_+^{m+2}}$ (see \eqref{eqn:GraphDirichlet} in the Appendix). In general, this discrepancy in scaling factors explains the superior performance of Ratio cut over spectral clustering on MMC problems. To see this, we return to the discussion in Remark \ref{rem:alphaSmall}, and notice that in order to create a very cheap balanced cut that captures the multi-manifold structure underlying the data it is sufficient to choose the angle $\alpha$ to be small enough without making it arbitrarily small.

		While in general ratio cut minimization is expected to be superior to spectral clustering from a theoretical perspective, it is at the algorithmic level that ratio cut is less appealing. In general, using spectral clustering as a warm start for ratio cut minimization is a reasonable strategy to consider as we have illustrated in the dumbbell-and-plane and dollar sign examples.

	\end{remark}

	\subsection{Comparison of different proximity graphs}

	
	We consider the setting illustrated in Figures \ref{fig: uneven knn} and \ref{fig: uneven epsilon} where we have generated 6500 points on the horizontal line and 2000 points on the vertical line, i.e., an uneven setting. We can see that when we add angle constraints to a $k_+, k_-$-NN base graph, we do not recover the two lines as we do when we use the $\veps_+, \veps_-$-graph setting. The outcomes illustrated here are markedly different from the ones in the even case where the NN approach provides more stable results, \blue as is the case with vanilla spectral clustering using a standard $k$-NN graph\nc. The reason is that around the intersection of the two lines, a $k$-NN neighborhood of a point in the vertical line mostly picks points in the horizontal line (since there is a higher density of points there), which results in very few connections with other points on the vertical line.

	In general, we expect the NN setting to struggle in settings where the densities of points are different around the intersections of the manifolds, as illustrated by the ``inclusion problem" in the setting from Figures \ref{fig: cone and plane knn} and \ref{fig: cone and plane epsilon}. In Figure \ref{fig: cone and plane knn} we illustrate the clusters output by spectral clustering in the path constrained NN setting. We see that the part of the plane contained inside the cone has been merged with the cone despite the fact that a strong angle constraint was used. The points in the inner part of the plane around the boundary have many more nearest neighbors in the cone than in the external portion of the plane, thus effectively discarding connections that would otherwise keep the plane better connected (as it is the case with the $\veps_+, \veps_-$ construction as shown in Figure \ref{fig: cone and plane epsilon}).

	We remark that the effect of density in clustering can be reduced by considering suitable normalized versions of proximity graphs as in \citet{coifman2006diffusion}, where in particular one can take the random walk Laplacian associated to a new set of weights $\omega^\alpha_{ij}$ of the form:
	\[ {\omega}^\alpha_{ij} = \frac{\omega_{ij}}{d_i^\alpha d_j^\alpha}, \]
	for some $\alpha\in (0,1]$, where $d_i$ and $d_j$ are the degrees of $x_i$ and $x_j$ relative to the original weight matrix $\omega$. It is possible to show that with the choice $\alpha =1$ one can effectively remove the effect of density in clustering. We notice that in the multi-manifold clustering setting, when manifolds have different dimensions, the role of density is more severe than when manifolds have the same dimension. This is because we are assuming that the number of points in each manifold is roughly the same, and so, densities on smaller dimensional objects tend to be considerably larger than densities on larger dimensional objects. The use of appropriate normalized Laplacians may thus help considerably with multi-manifold clustering problems.

	\begin{figure}[htbp]
		\centering
		\begin{minipage}[t]{0.4\textwidth}
			\centering
			\includegraphics[width=4cm]{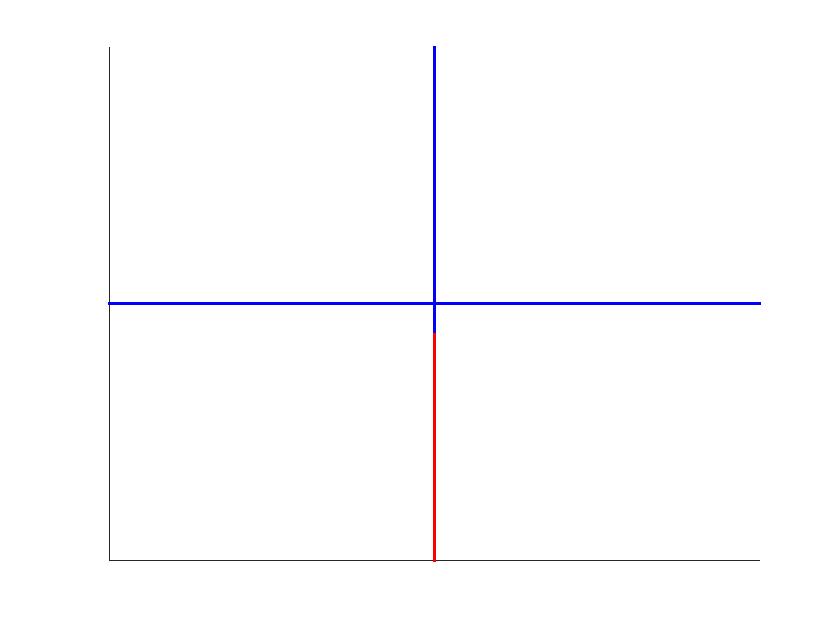}
			\caption{$k_+,k_-$ NN with angle constraint}
			\label{fig: uneven knn}
		\end{minipage}
		\begin{minipage}[t]{0.4\textwidth}
			\centering
			\includegraphics[width=4cm]{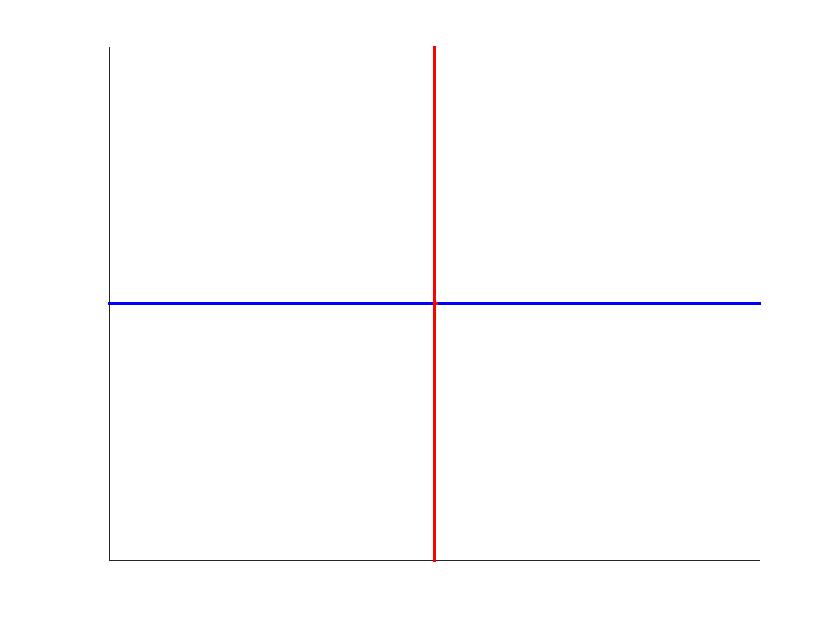}
			\caption{$\veps_+,\veps_-$ graph with angle constraint}
			\label{fig: uneven epsilon}
		\end{minipage}
	\end{figure}
	
	%
	%
	
	\begin{figure}[htbp]
		\centering
		\begin{minipage}[t]{0.4\textwidth}
			\centering
			\includegraphics[width=4cm]{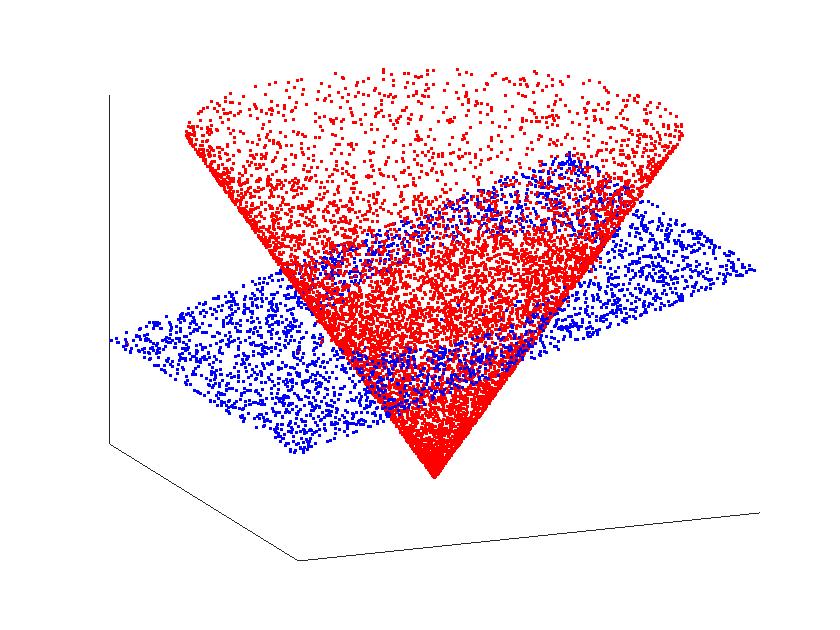}
			\caption{$k_+,k_-$ NN setting with angle constraint}
			\label{fig: cone and plane knn}
		\end{minipage}
		\begin{minipage}[t]{0.4\textwidth}
			\centering
			\includegraphics[width=4cm]{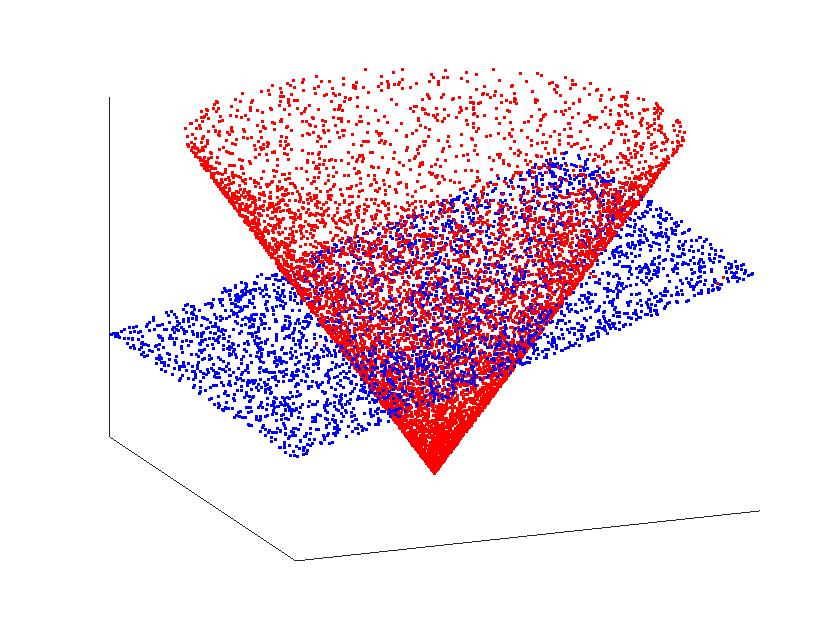}
			\caption{$\veps_+,\veps_-$ graph with angle constraint}
			\label{fig: cone and plane epsilon}
		\end{minipage}
	\end{figure}

	%
	%

	\subsubsection{Role of $\veps_-$ in annular graphs}
	\label{sec:veps-}
	
	Here we illustrate the effect of $\veps_-$ in the performance of spectral clustering. We focus on two possible choices: $\veps_-=0$ Vs $\veps_- \sim \veps_+$.  We consider points uniformly sampled from two intersecting 2-dimensional spheres with radius $1$ and distance between their centers equal to $0.6$ as illustrated in Figures \ref{fig: Path Algorithm with epsilon_+,epsilon_- -graph} and \ref{fig: Path Algorithm with epsilon-graph}. We can observe the discrepancy between the clusters obtained in both settings and how when we set $\veps_- \sim \veps_+$, the two manifolds are correctly identified. $\veps_+$ is the same in both cases.

	Notice that our theory shows that, in principle, any choice of $\veps_-$ (not too close to $\veps_+$) can provide correct identification of the manifolds as long as the number of samples is large enough. On the other hand, our theory also suggests that a non-zero $\veps_-$ can reduce the error of approximation (see Remark \ref{rem:Eps-equalzero}) as more faulty connections can be removed between points that are too close to the intersection. Our numerical experiments complement our theoretical findings.
	
	
	
	\begin{figure}[htbp]
		\centering
		\begin{minipage}[t]{0.3\textwidth}
			\centering
			\includegraphics[width=5cm]{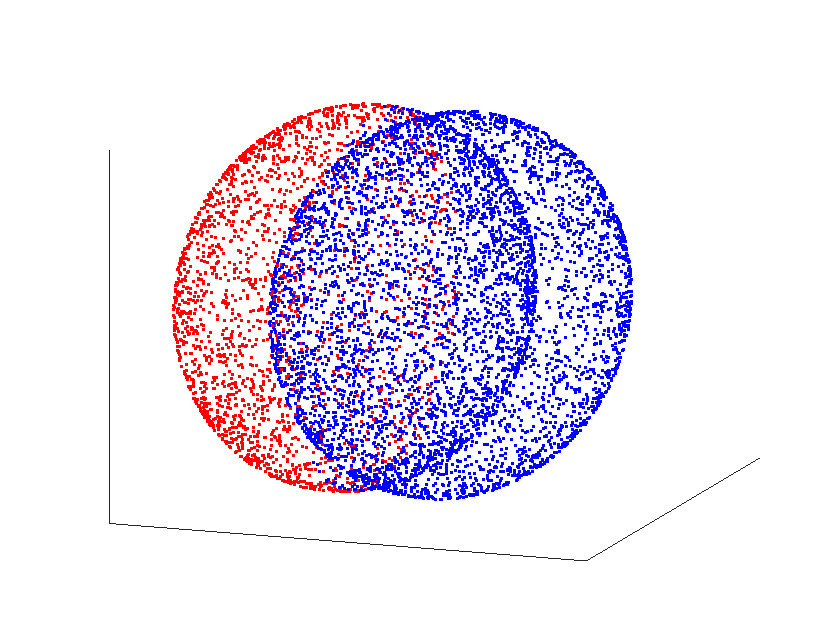}
			\caption{$(k_+,0)$-graph with angle constraint}
			\label{fig: Path Algorithm with epsilon-graph}
		\end{minipage}
		\begin{minipage}[t]{0.3\textwidth}
			\centering
			\includegraphics[width=5cm]{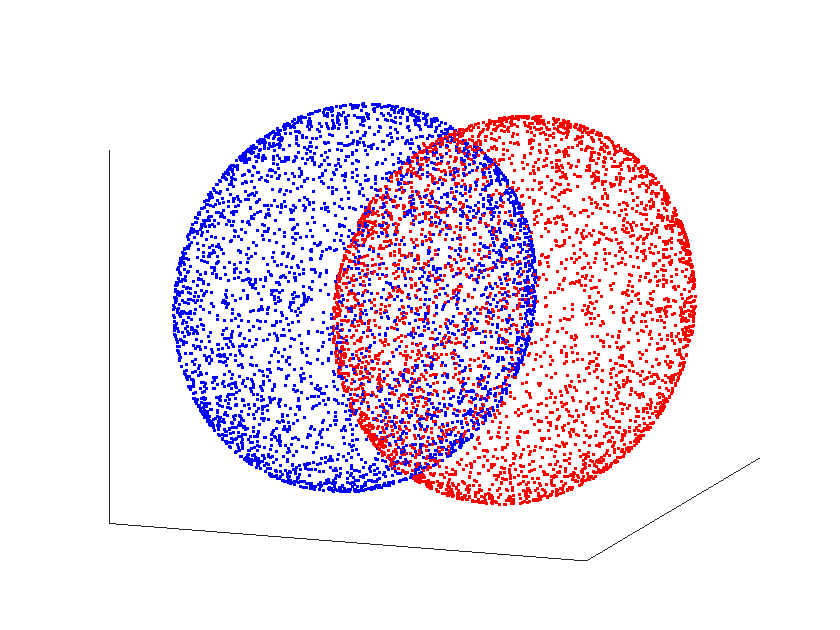}
			\caption{$(k_+,\frac{2k_+}{3})$-graph with angle constraint}
			\label{fig: Path Algorithm with epsilon_+,epsilon_- -graph}
		\end{minipage}
	\end{figure}

	\subsubsection{Path constraint graphs Vs other proximity graphs: self-intersections}

	In general, the use of fully inner connected and sparsely outer connected graphs on data sets imposes a specific geometric structure on the set $\M$ that is not necessarily inherited from the ambient space $\R^d$. This is true in the multi-manifold setting or even in the case of a single self-intersecting manifold (a setting not considered in our theoretical results). Take for example the self-intersecting manifold illustrated in Figures \ref{fig: selfintersect epsilon+,epsilon_-}-\ref{fig: selfintersect simple epsilon-neighbor}. When running spectral clustering with the annular graph with angle constraints, we get a partition of the data corresponding to the one we would have obtained when clustering a one-dimensional curve with no self-intersections. This is illustrated in Figure \ref{fig: selfintersect epsilon+,epsilon_-}. Figures \ref{fig: selfintersect simple KNN} and \ref{fig: selfintersect simple epsilon-neighbor}, on the other hand, show the clusters obtained when running spectral clustering based on a standard $k$-NN graph and a standard $\veps$-proximity graph respectively. As can be observed, these partitions are markedly different from the one in Figure \ref{fig: selfintersect epsilon+,epsilon_-}. Notice that the MMC method can detect the self-intersection point in the manifold from this example.

	Figures \ref{fig: selfintersect epsilon+,epsilon_-}-\ref{fig: selfintersect simple epsilon-neighbor} illustrate the effect of different graphs on the output clusters. Likewise, different graphs capture the underlying manifold differently when using higher eigenmodes to summarize additional geometric content ({a} specific geometric content) of the self-intersecting manifold.

	%
	%
	
	\begin{figure}[htbp]
		\par\medskip
		\centering
		\begin{minipage}[t]{0.3\textwidth}
			\centering
			\includegraphics[width=3.5cm]{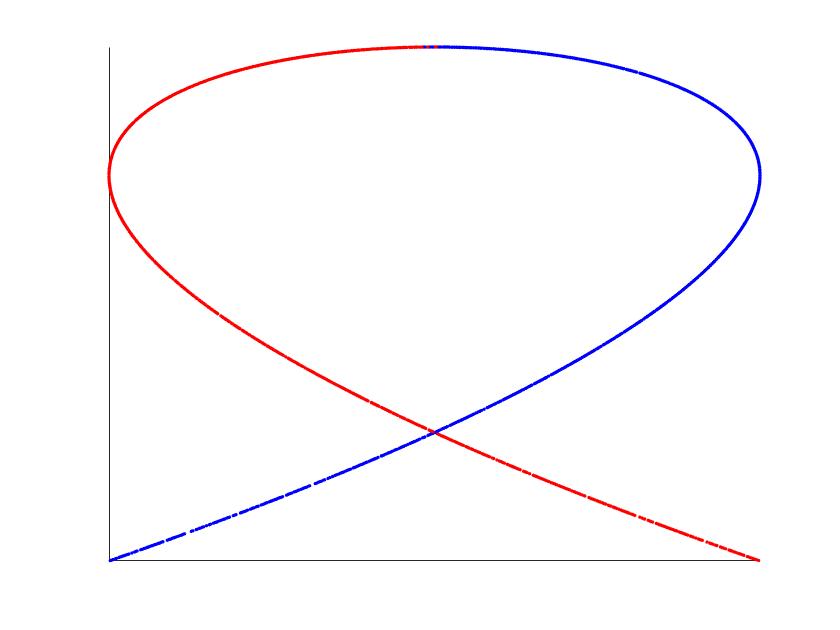}
			\caption{$\eps_+,\eps_-$-graph with angle constraint}
			\label{fig: selfintersect epsilon+,epsilon_-}
		\end{minipage}
		\begin{minipage}[t]{0.3\textwidth}
			\centering
			\includegraphics[width=3.5cm]{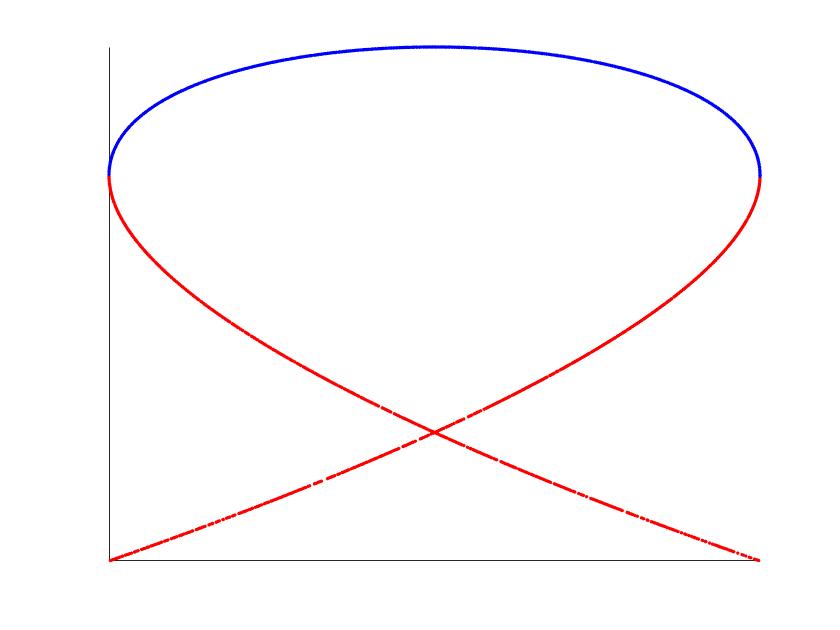}
			\caption{standard $k$-NN}
			\label{fig: selfintersect simple KNN}
		\end{minipage}
		\begin{minipage}[t]{0.3\textwidth}
			\centering
			\includegraphics[width=3.5cm]{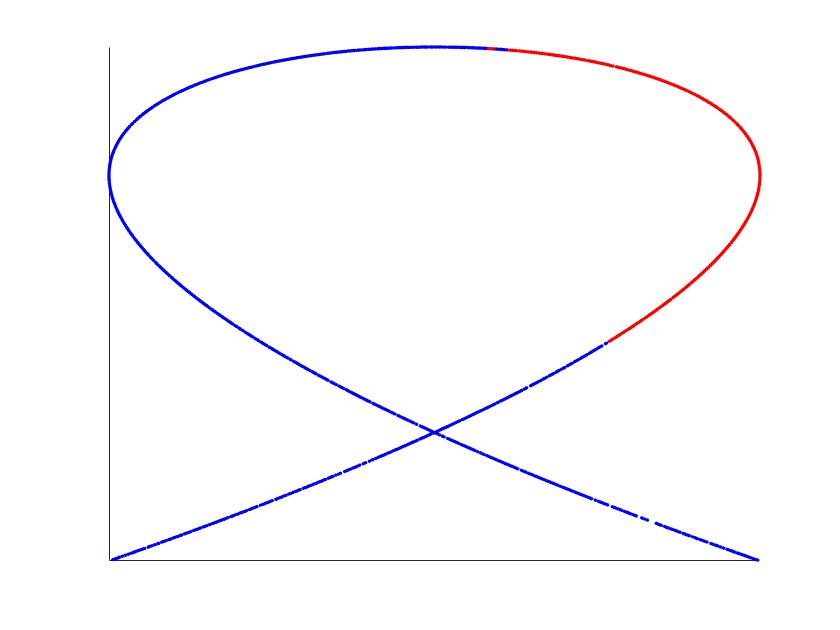}
			\caption{standard $\veps$-neighbor}
			\label{fig: selfintersect simple epsilon-neighbor}
		\end{minipage}
	\end{figure}

	\subsubsection{Other path algorithms}\label{sec:other path algorithms}

    The path-based similarity weights we study in this paper are inspired by an algorithm proposed in \citet{BABAEIAN2015118}. There, a less stringent notion of a ``smooth" discrete path is used to construct a proximity graph on the data set. In our construction, we force discrete paths to satisfy that every line segment in the path must be aligned with the segment connecting the first and last point in the path (i.e., essentially requiring a straight path). In contrast, in \citet{BABAEIAN2015118} the constraint is that any two consecutive segments in the path must be aligned (i.e., a path that does not turn too quickly). It is straightforward to see that when two manifolds with a dimension larger than two intersect, it is possible to construct paths connecting points in the two manifolds that meet the criterion in \citet{BABAEIAN2015118} but not our criterion. In summary, the more stringent constraint we impose helps remove more connections (faulty and correct). The removal of faulty connections seems more significant, and overall, our path algorithm outperforms the one in \citet{BABAEIAN2015118}.

	Another sensible path algorithm to build graphs for MMC is to directly find geodesic paths along the graph using Dijkstra's algorithm and then check whether they satisfy the angle constraints or not. The theoretical analysis for this approach is more involved since one needs to check that geodesics do satisfy the angle constraint in the cases where one expects them to (i.e., when connecting two points on the same manifold). Still, in practice this graph construction behaves comparably to the path algorithm we analyze mathematically. Our path construction and the geodesic based one are two examples of a more general procedure where one seeks a path that connects a pair of points satisfying the angle constraints \textit{and} whose length is no larger than a constant parameter times the geodesic distance along the path between the two points. This construction can be analyzed by combining ideas similar to the ones we have presented in section \ref{sec:GraphConstruct} with some analysis of the geodesic distance in a proximity graph.

	
	\nc

	\subsection{Sensitivity of theoretical assumptions}
	\subsubsection{Angles}
	We test the performance of spectral clustering on an annular graph with angle constraints when trying to separate manifolds as their angles of intersection decrease (i.e. $\beta$ in \eqref{eqn:AngleConstraint} grows). In our experiment, we consider the simple setting of two intersecting planes. 
	\begin{figure}[htbp]
		\centering
		\begin{minipage}[t]{0.22\textwidth}
			\centering
			\includegraphics[width=4cm]{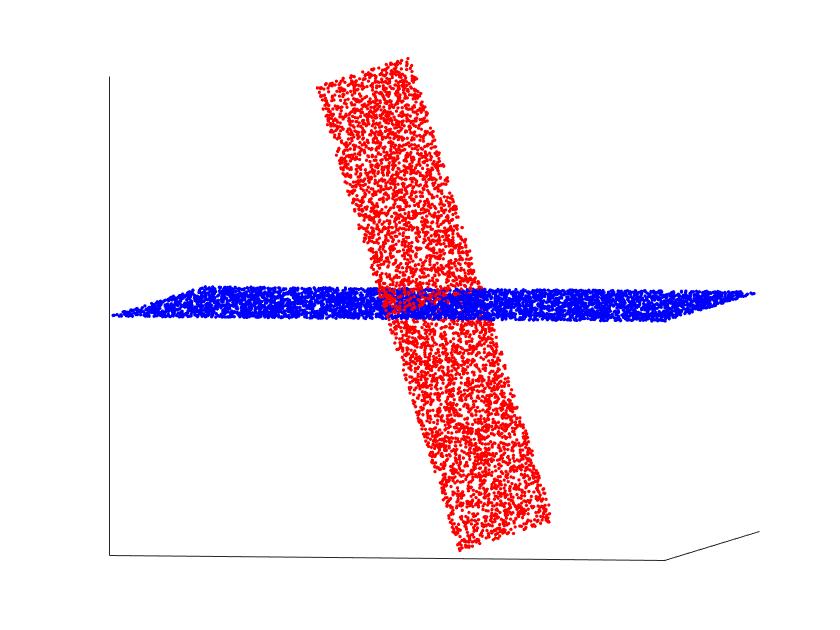}
			\caption{$75\degree$}
			\label{fig: 75 degree}
		\end{minipage}
		\begin{minipage}[t]{0.22\textwidth}
			\centering
			\includegraphics[width=4cm]{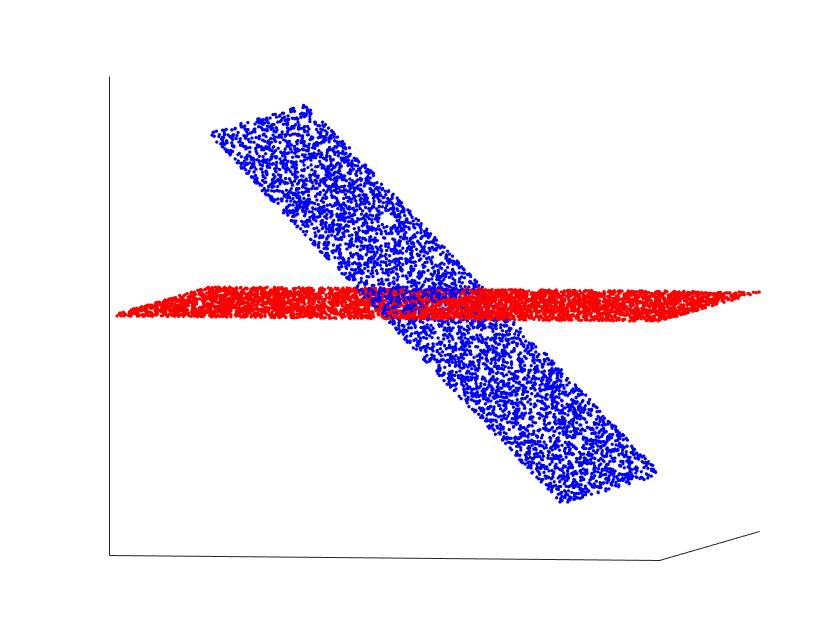}
			\caption{$50\degree$}
			\label{fig: 50 degree}
		\end{minipage}
		\centering
		\begin{minipage}[t]{0.22\textwidth}
			\centering
			\includegraphics[width=4cm]{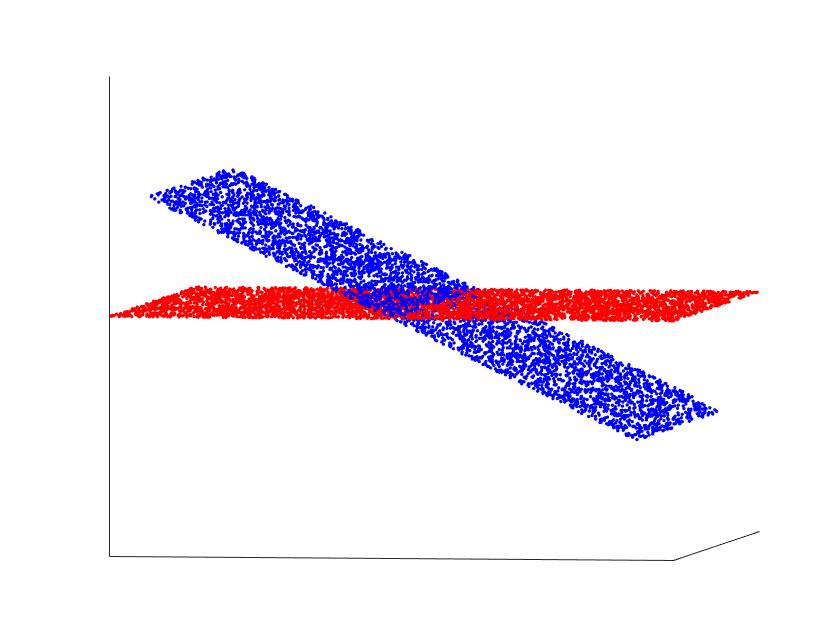}
			\caption{$30\degree$}
			\label{fig: 30 degree}
		\end{minipage}
		\begin{minipage}[t]{0.22\textwidth}
			\centering
			\includegraphics[width=4cm]{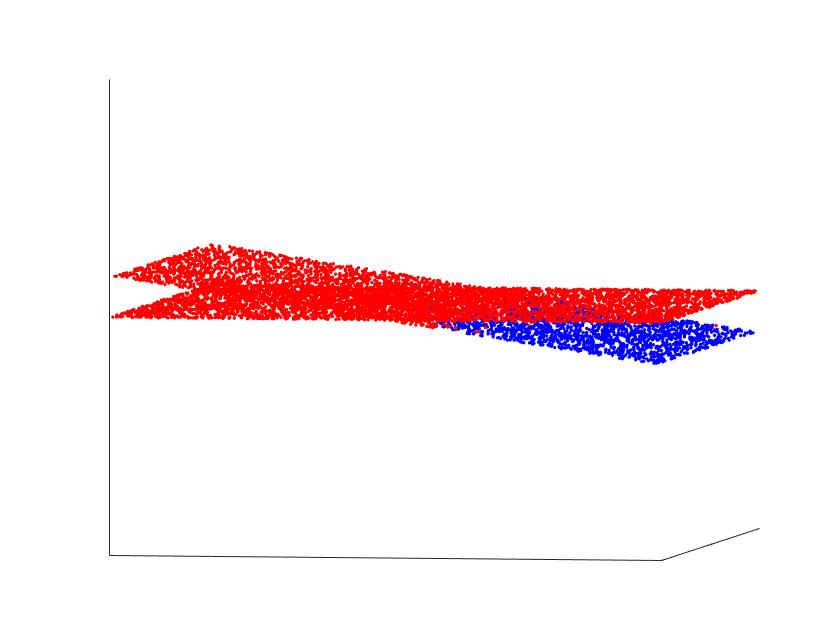}
			\caption{$10\degree$}
			\label{fig: 10 degree}
		\end{minipage}
	\end{figure}
	We see that in Figures \ref{fig: 75 degree}-\ref{fig: 30 degree} we recover the two planes, while in Figure \ref{fig: 10 degree} we do not. The results here are reasonable because when the angle of intersection is too small, a much smaller threshold value for the angle constraint is needed to discriminate different manifolds at the expense of removing connections between points that should have been connected otherwise. \blue For these experiments we have used the NN version of our algorithm.\nc 
	

	\subsubsection{Orthogonal noise}
	
	In our theoretical results, we assumed data points to lie exactly on top of a set of the form $\M= \M_1 \cup \dots \cup \M_K$. However, a natural question is whether spectral clustering with the similarity graph constructed with the path algorithm continues to perform well when orthogonal noise is added to the data. Figures \ref{fig: 3 lines with perturbation} and \ref{fig: 3 lines with large perturbation} show two examples of data sets contaminated by orthogonal noise. In both cases, the multi-manifold structure is readily apparent: three intersecting lines at a single point. However, in the setting depicted in Figure \ref{fig: 3 lines with large perturbation}, where the noise level is large, we see that the path algorithm does not recover the multi-manifold structure correctly. This suggests that the path algorithm is quite sensitive to noise. \blue For these experiments we have used the NN version of our algorithm.\nc
	
	We can use the number of connections to see how much the noise affects the algorithm. For example, in Figure \ref{fig: 3 lines with large perturbation}, where we exhibit the clean data, the total number of connections between data points is 579208, while the number of faulty connections is 1126. When noise is added, the number of total connections is 193426, while the number of faulty connections is 5414. That is, in general, we expect noise to worsen both inner and outer connectivities. 
	
	A potential remedy is to pre-process the data set by running a denoiser. However, some naive denoising methods, including the centering method or projected PCA, do not improve the performance. In Figure \ref{fig: 3 lines after naive denoising} we illustrate the outcome of spectral clustering on an annular proximity graph with angle constraints on the denoised data set. Specifically, using the centering method, the total number of connections was 215260, and the number of faulty connections was 5418. For the projected PCA method, the total number of connections is 271498, and the number of faulty connections is 6956. Roughly speaking, these methods can improve the inner connectivity while worsening the outer connectivity. How to implement a good denoising strategy in the MMC setting is an interesting direction to explore.
	
	

	

	\begin{figure}[htbp]
		\centering
		\begin{minipage}[t]{0.3\textwidth}
			\centering
			\includegraphics[width=4cm]{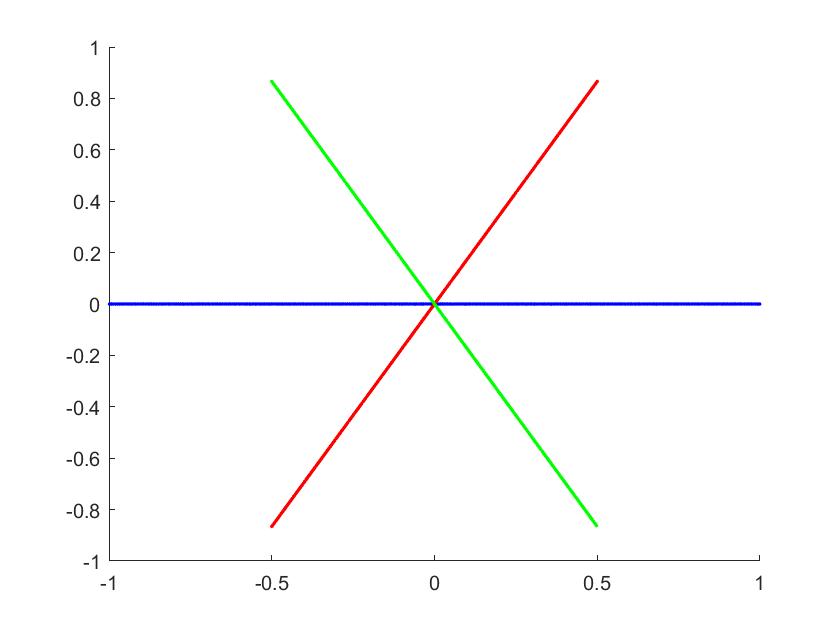}
			\caption{Small Perturbation}
			\label{fig: 3 lines with perturbation}
		\end{minipage}
		\begin{minipage}[t]{0.3\textwidth}
			\centering
			\includegraphics[width=4cm]{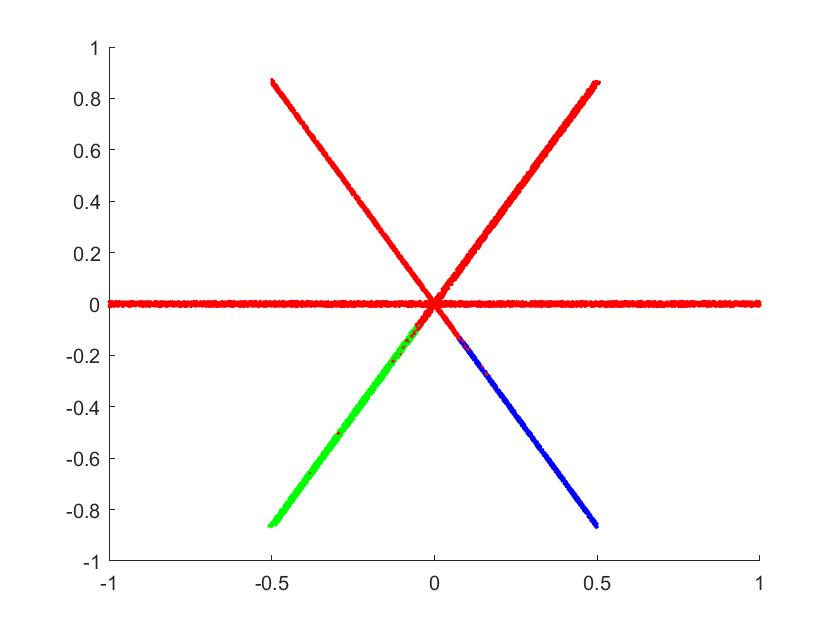}
			\caption{Large Perturbation}
			\label{fig: 3 lines with large perturbation}
		\end{minipage}
		\begin{minipage}[t]{0.3\textwidth}
			\centering
			\includegraphics[width=4cm]{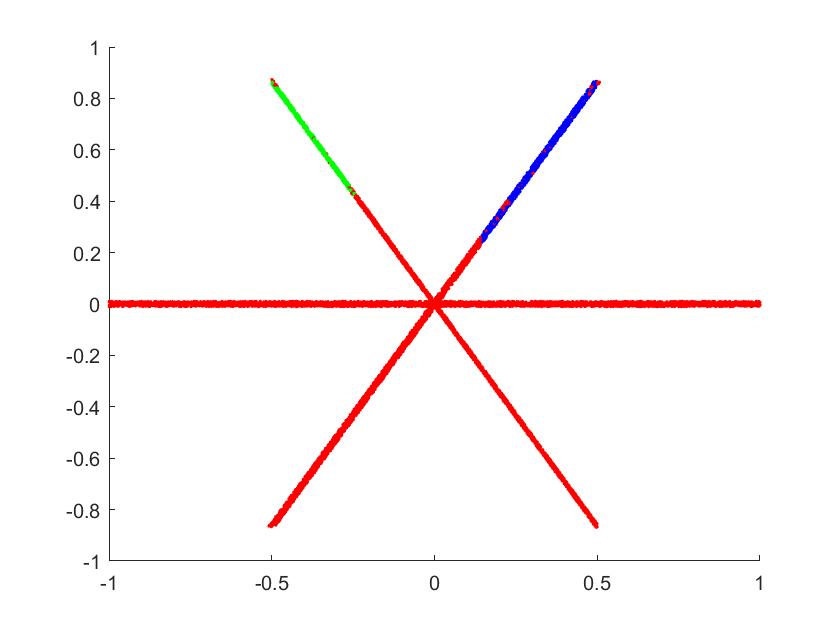}
			\caption{After denoising}
			\label{fig: 3 lines after naive denoising}
		\end{minipage}
	\end{figure}


	
	

	\subsubsection{Small Vs large number of data points}
 \label{sec:3planes}
	In this section, we consider data sets supported on the union of three intersecting planes as illustrated in Figures \ref{fig: Smaller sample size} and \ref{fig: Larger sample size}. In both figures, the underlying planes are the same, and the only thing that changes from one figure to the other is the sample size. 
	
	\begin{figure}[htbp]
		\centering
		\begin{minipage}[t]{0.3\textwidth}
			\centering
			\includegraphics[width=4cm]{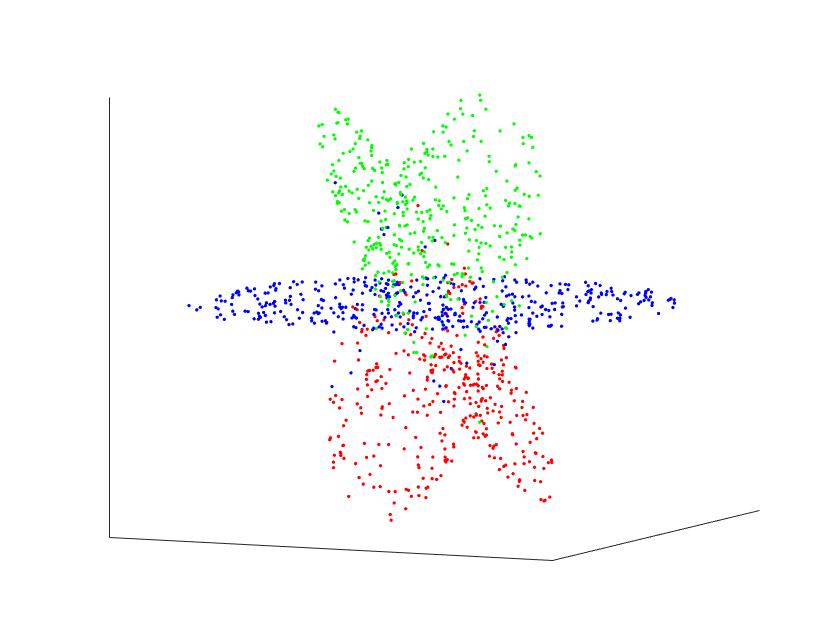}
			\caption{$n$ data points}
			\label{fig: Smaller sample size}
		\end{minipage}
		\qquad \qquad
		\begin{minipage}[t]{0.3\textwidth}
			\centering
			\includegraphics[width=4cm]{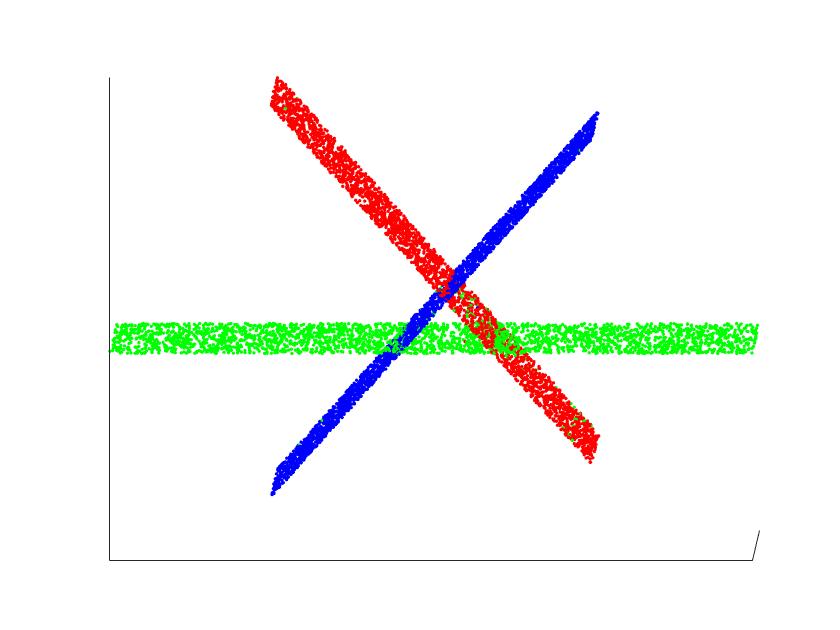}
			\caption{$2n$ data points}
			\label{fig: Larger sample size}
		\end{minipage}
	\end{figure}
	As can be observed, the three planes are not appropriately identified in the small sample size regime from Figure \ref{fig: Smaller sample size}. In contrast, when we duplicate the amount of data as in Figure \ref{fig: Larger sample size} the three planes are identified correctly.

	This simple example illustrates some crucial drawbacks of the MMC methods based on spectral clustering discussed throughout the paper. In order to correctly construct local paths (or local tangent planes) to, in turn, detect the underlying manifolds, one needs to consider a large enough neighborhood around every point containing enough samples for the variance of the estimation to be small at the expense of increasing the bias considerably. Building MMC methods that can operate at smaller sample sizes is an interesting direction to explore in future research. \blue For example, one could attempt to design a hybrid method that uses both path-based and local tangent plane information to make the method more robust to lower sample size; this is motivated by the fact that local PCA approaches can more accurately operate at smaller sample sizes when considering points that are far away from the intersection of manifolds. 
	
	\nc
	
	
	%
	\subsection{Different dimensions}
	
    In section \ref{subsection: Different dimensions} we presented a series of theoretical results for multi-manifold clustering when $\M$ is the union of smooth manifolds with different dimensions.  We now illustrate these results with a few simple numerical examples.
	
	\subsubsection{Planes and lines}
	We consider a data set uniformly sampled from the union of two planes and two lines that meet orthogonally, as illustrated in Figures \ref{fig: 2 Clusters}-\ref{fig: 5 Clusters}. 
	We run spectral clustering with $K=2,3,4,5$ to understand how the geometries of the manifolds get captured.

	\begin{figure}[htbp]
		\centering
		\begin{minipage}[t]{0.24\textwidth}
			\centering
			\includegraphics[width=4cm]{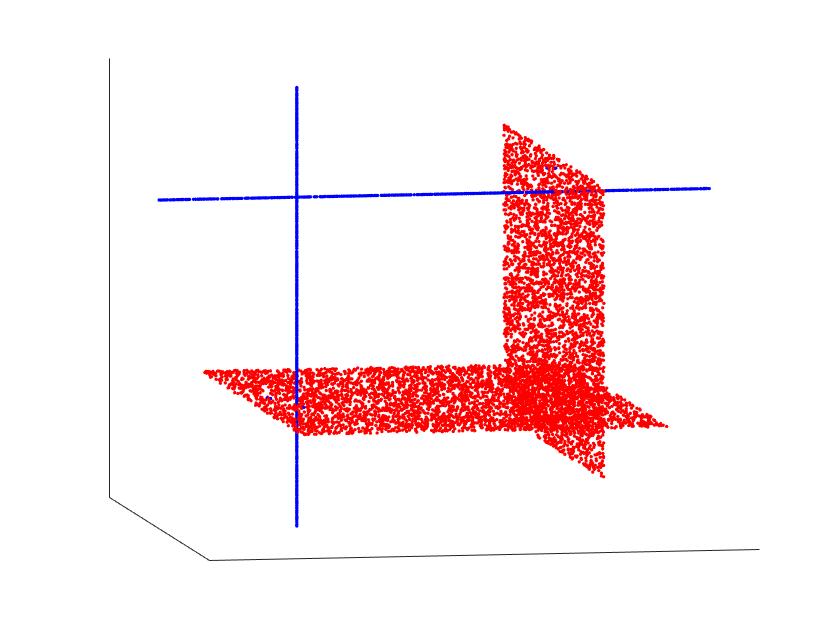}
			\caption{2 Clusters}
			\label{fig: 2 Clusters}
		\end{minipage}
		\begin{minipage}[t]{0.24\textwidth}
			\centering
			\includegraphics[width=4cm]{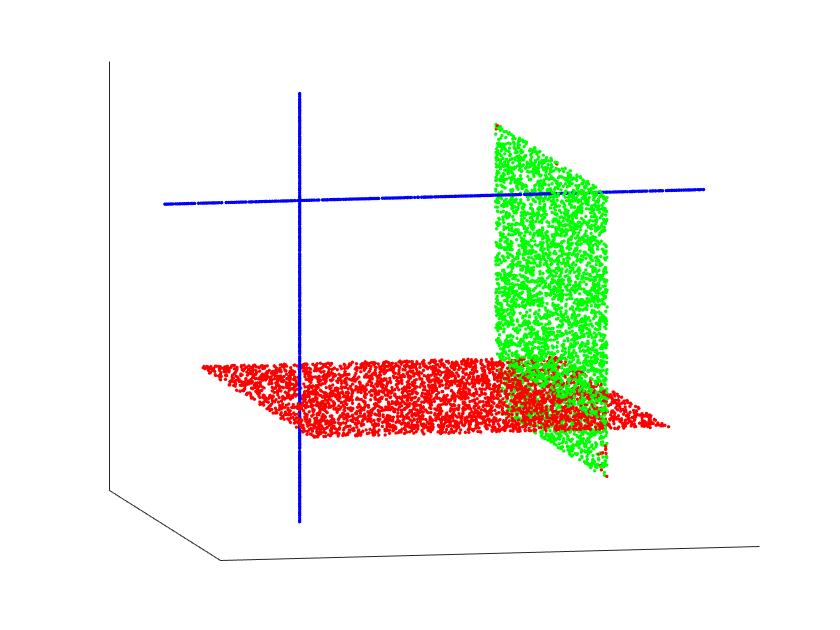}
			\caption{3 Clusters}
			\label{fig: 3 Clusters}
		\end{minipage}
		\begin{minipage}[t]{0.24\textwidth}
			\centering
			\includegraphics[width=4cm]{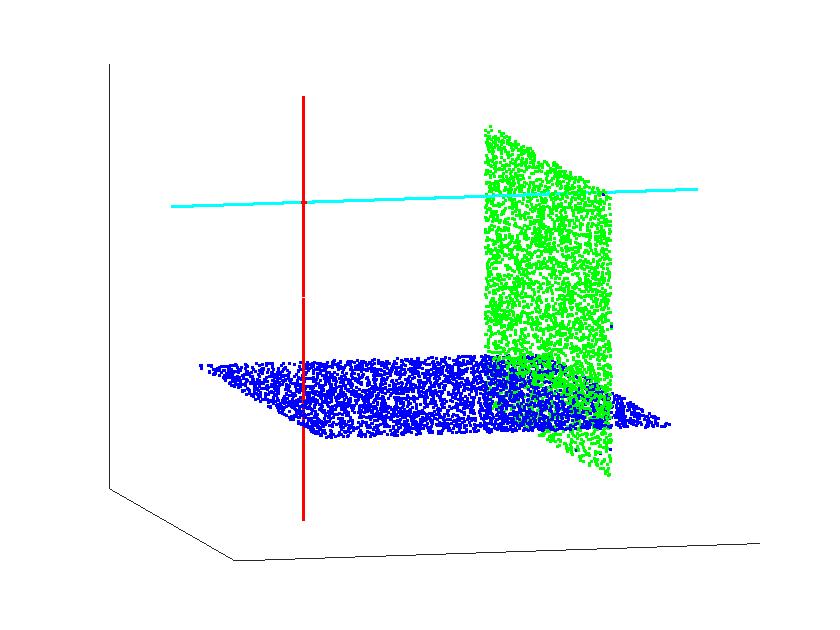}
			\caption{4 Clusters}
			\label{fig: 4 Clusters}
		\end{minipage}
		\begin{minipage}[t]{0.24\textwidth}
			\centering
			\includegraphics[width=4cm]{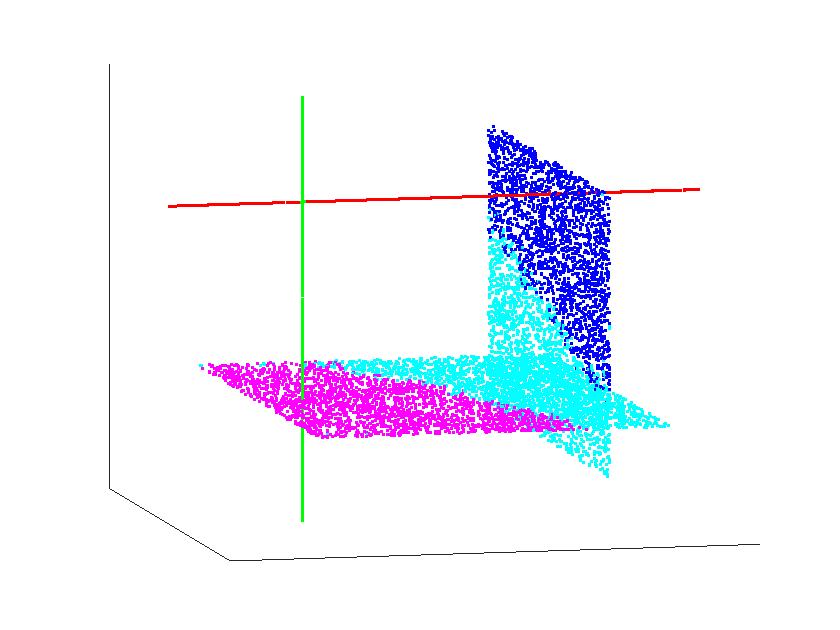}
			\caption{5 Clusters}
			\label{fig: 5 Clusters}
		\end{minipage}
	\end{figure}
	
	In Figure \ref{fig: 2 Clusters}, when we consider $K=2$, the whole data set splits into two parts: lines and planes, indicating that manifolds with different dimensions are separated, and manifolds with the same dimension are put into the same cluster. In Figure \ref{fig: 3 Clusters}, when we try $K=3$ clusters, the two planes get separated perfectly while the lines are clustered as one; this is supported by our theory which indeed suggests that the geometry of the higher dimensional objects is detected first. When $K=4$, lines get separated as shown in Figure \ref{fig: 4 Clusters}. The case $K=5$ illustrates the theory developed in this paper quite well. It shows how the internal geometry of the higher dimensional manifolds (in this case, the planes) is detected because the internal geometry of lines is more expensive than planes. 
	
	
	%

	Another illustration of the behavior of spectral clustering with constrained annular proximity graphs is presented in Figures \ref{fig: Ball And 3 Lines clusters 2}- \ref{fig: Ball And 3 Lines clusters 4}. Here the data set is supported in the union of a $2$-dimensional sphere and three lines that connect at one point. The same observations we made in the planes and lines example also apply to this setting.
	
	\begin{figure}[htbp]
		\centering
		\begin{minipage}[t]{0.3\textwidth}
			\centering
			\includegraphics[width=4cm]{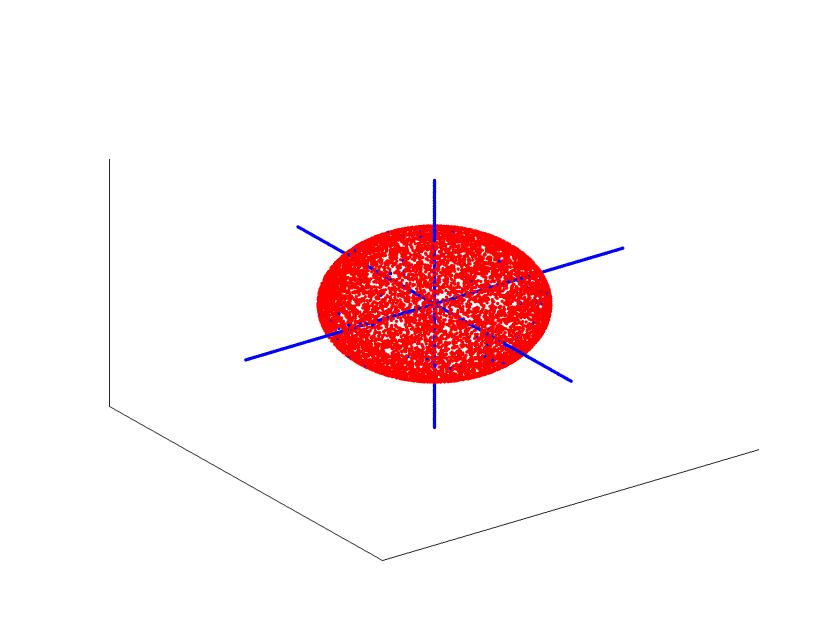}
			\caption{2 clusters}
			\label{fig: Ball And 3 Lines clusters 2}
		\end{minipage}
		\begin{minipage}[t]{0.3\textwidth}
			\centering
			\includegraphics[width=4cm]{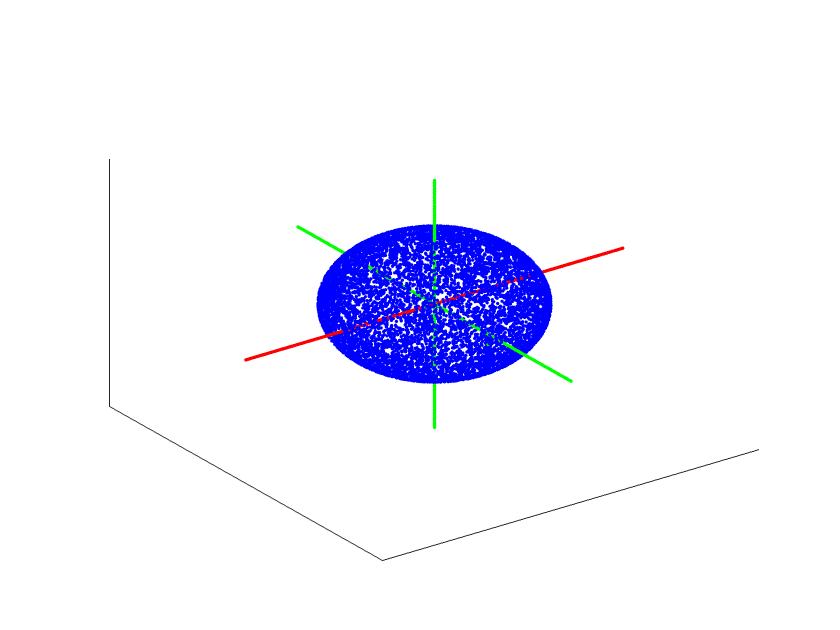}
			\caption{3 clusters}
			\label{fig: Ball And 3 Lines clusters 3}
		\end{minipage}
		\begin{minipage}[t]{0.3\textwidth}
			\centering
			\includegraphics[width=4cm]{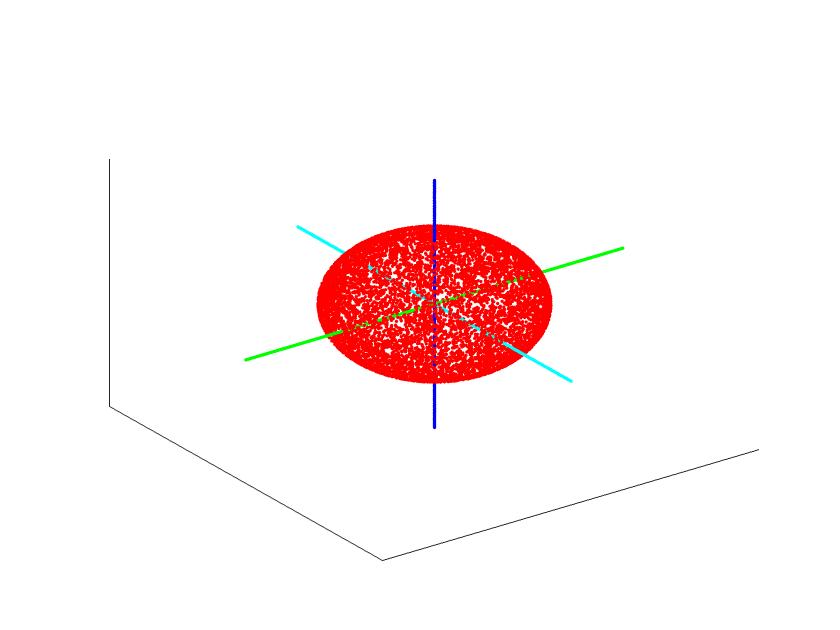}
			\caption{4 clusters}
			\label{fig: Ball And 3 Lines clusters 4}
		\end{minipage}
	\end{figure}
	
	\new{
	\subsection{Comparison with other MMC approaches}
	
	In this section, we compare the performances of spectral clustering using annular proximity graphs with angle constraints, SMCE (Sparse Manifold Clustering and Embedding) \citet{elhamifar2011sparse}, and SC with local PCA \citet{arias2017spectral}; both SMCE and SC with local PCA have been designed for MMC tasks. For SMCE, we follow the parameter choices in \citet{elhamifar2011sparse} and grid search for optimal parameters. For local PCA, we follow the parameter setting as in \citet{arias2017spectral} and tune the radius parameter and dimension by running a grid search to get the lowest misclustering rate. Since the algorithm in \citet{arias2017spectral} is a randomized algorithm, we run the algorithm 100 times and report the average misclustering rate. We use both $k$-NN and $\eps$ graphs in the setting of \citet{arias2017spectral} and report the results of the best performing settings. For SC with the constrained path algorithm, we use its nearest neighbor version as discussed at the beginning of section \ref{section: discussion}. We first compare all algorithms when we run them on synthetic data sets and then conduct a comparison when running them on the MNIST data set.
	
	\subsubsection{Synthetic data sets}
 \label{sec:Synthetic}
	We generate data points from five different settings of intersecting manifolds; see Table \ref{table:2}.  We see that SC with the angle-constrained path algorithm achieves, overall, the lowest misclustering rate, outperforming the competing algorithms. We can see that SC with local PCA can work well for the settings of 2 spheres and 1 sphere with 1 plane; in those settings, most misclustered points are points close to the intersections of manifolds. For the 3 planes example from section \ref{sec:3planes}, i.e. three planes intersecting at the same line, we sample 3000 points. Local PCA has particular difficulty distinguishing points close to the intersection, and only until the sample size has been increased considerably we recover the correct clustering with that algorithm. The 2 planes with 1 line and 1 sphere with 1 line examples are used to evaluate the performance of the algorithms when manifolds have different dimensions, a setting that is not the original target of \citet{arias2017spectral}, where a dimension parameter must be chosen. On the other hand, SMCE is not particularly designed to handle intersecting manifolds, and we see the overall low performance in most of the experiments run. 
	
	\begin{table}[h!]
	\begin{center}
    \begin{tabular}{| c| c| c| c| c| c|}
    \hline
     Algorithm & 3planes
 & 2spheres
 & 2planes 1line & 1sphere 1line & 1sphere 1plane \\ \hline
     path &$14.72\%$ & $0.49\%$  & $0.38\%$ & $0.22\%$ & $0.78\%$ \\ \hline
    SMCE  & $48.9\%$ & $16.2\%$ &$31.1\%$ & $26.8\%$ & $42.7\%$\\ 
    \hline
    local PCA & $ 43.66\%$ & $3.1\%$ & $32.17\%$ & $16.43\%$ & $1.19\%$   \\
    \hline
    \end{tabular}
    \caption{Misclustering rate}
    \label{table:2}
    \end{center}
\end{table}

	\subsubsection{MNIST}

	In this section we compare misclustering rates when we test algorithms on subsets of the MNIST data set consisting of different pairs of digits. Following the same preprocessing step as in \citet{babaeian2018multiple}, we first utilize the SURF feature of the Bag of words model to represent the features of each image. The original feature vectors have a size of 500. Then, for some pairs of digits, we use PCA to reduce the dimension of the image vector to $10$. In the final step we apply the unsupervised algorithms to the data sets. We only present the results for some examples of pairs of digits for brevity (see Table \ref{table:3}), but similar observations to the ones that we discuss below can be drawn from other choices of digits. Like in the synthetic data experiments, we grid search the optimal parameters for every algorithm and for every task. In all the tasks considered in this section we also run vanilla SC algorithm with a standard $k$NN graph, tuning $k$ to achieve the best misclustering rate. In contrast to the experiments in section \ref{sec:Synthetic}, here it is not clear that the considered data sets possess an underlying multi-manifold geometric structure.
	
 SC with the angle-constrained path algorithm can be seen as a generalization to vanilla SC, and we can see that it improves SC significantly in some tasks, such as clustering between digits [0,2], and at least behaves comparably to vanilla SC in other tasks; presumably, the improved performance over vanilla SC is manifested when the data manifolds corresponding to different digits do intersect. Notice that for the [0,7] digits SC with path algorithm and vanilla SC fail, while SMCE does perform very well. 
 Overall, local PCA performs poorly for the tasks discussed here. 

 We want to highlight that the performance of the MMC algorithms that we have compared in these experiments may strongly depend on the manifold assumption (which may not hold on first place), and thus, if one was to stick to the theoretical assumptions discussed in this paper, one would need to guarantee, for example, that the data embedding methods in the preprocessing steps preserve or enhance these assumptions. The experiments that we have considered here are thus not meant to suggest that one algorithm is always better than the others. Instead, we wanted to evaluate the performance of algorithms with theoretical guarantees such as SC using path-based graphs on real data sets to test their capabilities and highlight that other methods used in the literature may underperform in some standard real data tasks.

    \begin{table}[h!]
	\begin{center}
    \begin{tabular}{ c c c c c c c c c c}
    \hline
    Algorithm & [0,1] &[0,2] & [0,3] & [0,4] & [0,5]& [0,6]& [0,7] &[0,8] & [0,9] \\
    \hline 
    path & $14.0\%$ & $5.6\%$  & $1.9\%$ & $1.8\%$ & $2.6\%$ & $7.7\%$ & $46.4\%$ &  $9.7\%$ & $1.9\%$ \\ \hline
    local PCA &  $6.4\%$ & $25.9\%$ & $30.0\%$ & $45.5\%$ & $34.8\%$  & $34.5\%$  & $34.1\%$ &  $26.6\%$ & $25.1\%$ \\ \hline
    SMCE &$20.0\%$&  $25.5\%$ & $6.9\%$ & $9.2\%$ &$24.1\%$ & $12.1\%$ & $2.9\%$& $17.8\%$& $3.8\%$  \\ 
    \hline
    SC &$18.8\%$&  $12.8\%$ & $1.8\%$ & $2.2\%$ &$2.6\%$ & $10.0\%$ & $46.4\%$& $11.8\%$& $2.3\%$ \\
    
    \hline
    \end{tabular}
    \caption{Misclustering rates for some subsets of MNIST}
    \label{table:3}
    \end{center}
    \end{table}

}


	%
	
	%

	\acks{ NGT was supported by NSF-DMS grant 2005797. 
	Support for this research was provided by the Office of the Vice Chancellor for Research and Graduate Education at the University of Wisconsin-Madison with funding from the Wisconsin Alumni Research Foundation. The authors would like to thank the IFDS at UW-Madison and NSF through TRIPODS grant 2023239 for their support.}
	
	\bibliography{references.bib}

	\appendix

	\section{Proofs of main results}\label{Section: main results}

	\subsection{Discrete Dirichlet energies}
	\label{app:1}

	It is well known that an operator like $\gL$ (defined in \eqref{eqn:GraphLaplacian}) is positive semi-definite with respect to $\langle \cdot,\cdot \rangle_{L^2(\mu^n)}$ (e.g. \citet{DBLP:journals/corr/abs-0711-0189}); here and in the remainder we use $\mu^n$ to denote the empirical measure of $X$. Notice that $\gL$'s eigenvalues, labeled in ascending order as
	\[
	0= \lambda_1^{\veps_+,\veps_-} \leq \lambda_2^{\veps_+,\veps_-} \leq \lambda_3^{\veps_+,\veps_-} \leq \dots \leq \lambda_n^{\veps_+,\veps_-},
	\]
	can be characterized variationally according to the Courant-Fisher minmax principle: 
	\begin{equation}\label{minmax principle for graph laplacian}
	\lambda_l^{\veps_+,\veps_-}=\min_{S\in \mathcal{G}_l}\max_{u\in S\backslash \{0\}}\frac{b^{\veps_+,\veps_-}(u)}{\lVert u\lVert^2_{L^2(\mu^n)}},
	\end{equation}
	where $\mathcal{G}_l$ denotes the set of all linear subspaces of $L^2(\mu^n)$ of dimension $l$. Here, $b^{\veps_+,\veps_-}$ is the Dirichlet energy:
	\begin{equation}
	\begin{split}
	b^{\veps_+,\veps_-}(u):&=\frac{1}{n^2(\veps_+^{m+2}-\veps_-^{m+2})}\sum_{x_i,x_j\in\X_n }\omega
	_{x_i x_j}(u(x_i)-u(x_j))^2
	=\frac{1}{2}\langle\mathcal{L}^{\veps_+,\veps_-}u,u\rangle_{L^2(\mu^n)}
	\end{split}
	\label{eqn:GraphDirichlet}
	\end{equation}
	where $u\in L^2(\mu^n)$.
	
	We introduce  \textit{inner} and \textit{outer} weights associated to the $\omega$ defined as  $\omega^I_{x_i x_j}=\omega_{x_i x_j}$ and $\omega^O_{x_i x_j}=0$ when $x_i, x_j$ belong to the same manifold, and $\omega^I_{x_i x_j}=0$ and $\omega^O_{x_i x_j}=\omega_{x_i x_j}$ otherwise. With this notation in place, we can introduce outer and inner Dirichlet energies associated to $\mathcal{L}^{\veps_+,\veps_-}$ according to:
	\begin{equation}
	\begin{split}
	b_O^{\veps_+,\veps_-}(u):&=\frac{1}{n^2(\veps_+^{m+2}-\veps_-^{m+2})}\sum_{x_i,x_j\in\X_n }\omega^O
	_{x_i x_j}(u(x_i)-u(x_j))^2,\\
	b_I^{\veps_+,\veps_-}(u):&=\frac{1}{n^2(\veps_+^{m+2}-\veps_-^{m+2})}\sum_{x_i,x_j\in\X_n }\omega^I
	_{x_i x_j}(u(x_i)-u(x_j))^2.
	\end{split}
	\label{eqn:InnerOuterDirichlet}
	\end{equation}
	Clearly $b^{\veps_+,\veps_-} = b_O^{\veps_+,\veps_-}+ b_I^{\veps_+,\veps_-}$. 
	
	It will be convenient for our analysis to decompose $b_I^{\veps_+,\veps_-}$ further and write it as the sum of Dirichlet energies associated to each of the manifolds $\M_k$. For that purpose we split the data set $X$ into disjoint sets $X= \bigcup_{k=1}^N X_k$, where each of the $X_k$ can be taken to be, without the loss of generality, equal to $X_k = X \cap \M_k$ (this is due to the first condition in Assumption \ref{assump:WellSeparated} which implies that with probability one no $x_i$ belongs to two or more of the $\M_k$). It is worth highlighting that the previous partitioning of the data makes sense mathematically even if it is not meaningful in practice (because we do not know the manifolds $\M_k$). In what follows and whenever needed we list the points in $X_k$ as $ \{x_{1k},x_{2k},\cdots,x_{n_kk}\} $ and use $\mu_{k}^n$ to denote their associated empirical probability measure. The number of data points in  $\M_k$, i.e. $n_k$, is easily seen to satisfy $\mathbb{E}n_l=n w_l$. Moreover, the following concentration estimate holds.
	\begin{proposition}\label{Proposition: Tail Inequality}
With probability no less than $1-2\exp \left(\frac{-2 t^{2}}{n}\right)$, we have 
		\begin{equation}
		nw_i-t<n_i<nw_i+t.
		\end{equation}
	\end{proposition}
	
	The graph Dirichlet energy associated to an individual manifold is defined by
	\begin{equation}
	b^{\veps_+,\veps_-}_l(u_l):=\frac{1}{n_l^2(\veps_+^{m+2}-\veps_-^{m+2})}\sum_{x_i,x_j\in\chi_n^l }\omega_{x_i x_j}(u_l(x_i)-u_l(x_j))^2, \quad u_l \in L^2(\mu_l^n).
	\label{eqn:DirichletIndivManifold}
	\end{equation}
	It follows that
	\[  b^{\veps_+,\veps_-}_I(u) = \sum_{l=1}^N\left(\frac{n_l^2}{n^2}\right) b_{l}^{\veps_+, \veps_-} (u_l), \quad u \in L^2(\mu^n),    \]
	where in the above and in the remainder we identify a function $u: X \rightarrow \R$ with a tuple $(u_1, \dots, u_N)$ where each of the $u_k$ is a function from $X_k$ into $\R$.
	
	\begin{remark}
		The local discrete Dirichlet energies $\gbk$ are similar to discrete Dirichlet energies that have been studied in the literature under the smooth manifold assumption. There is however an important difference. Indeed, although the weight matrix $\omega $ is  assumed to satisfy the full inner connectivity condition, i.e. with high probability the weights $\omega_{x_i x_j}$ can be thought of as those coming from a proximity graph, the type of proximity graph that we consider here is not standard since it is built with a kernel that has annular geometric structure. This type of kernel has not been considered nor analyzed before. As observed intuitively, as well as in our experiments, the idea of removing connections between points that are too close to each other significantly helps in reducing the number of connections between points in different manifolds, a feature that is useful for the multi-manifold clustering problem.

	\end{remark}

	\subsection{Discretization and interpolation maps}\label{set up}
	Our first goal is to find a quantitative relationship between the Dirichlet energies $D$ and $b^{\veps_+,\veps_-}$ via two conveniently chosen maps $P: L^2(\mu) \rightarrow L^2(\mu^n) $ and $\mathcal{I} : L^2(\mu^n) \rightarrow L^2(\mu) $. We look forward to obtaining inequalities of the form:
	\begin{equation}
	\label{eqn:InformalDirichletIneq}
	\sigma_\eta D(Iu)\leq(1+e_1)b^{\veps_+,\veps_-}(u); \quad b^{\veps_+,\veps_-}(Pf)\leq (1+e_2)\sigma_\eta D(f)+e_3
	\end{equation}
	where $e_1, e_2, e_3$ are small error terms depending on the problem's parameters, and $\sigma_\eta$ is the constant in \eqref{eqn:sigmaeta}. 
	
	\medskip

	%

	We start by combining Proposition 2.11  in \citet{calder2019improved} with Proposition \eqref{Proposition: Tail Inequality} to obtain the probabilistic estimates that we use in the remainder to connect graph-based energies with their continuum counterparts.

	\begin{corollary}
		\label{cor:Densities}
		With probability at least  $1-\sum_{l=1}^N (nw_l+t) \exp \left(-\mathrm{C}(nw_l-t) \theta^{2} \widetilde{\delta}^{m_l}\right)-2N\exp \left(\frac{-2 t^{2}}{n}\right)$, there exist:
		\begin{enumerate}
			\item probability density functions $\widetilde{\rho}_{l}^n: \M_l\rightarrow \R$ satisfying: 
			\[ \left\|\rho_l-\widetilde{\rho}_{l}^n\right\|_{L^{\infty}(\M_l)} \leq C(\theta+\widetilde{\delta})\]
			for each $l=1, \dots, N$, and also
			\item maps $\widetilde{T}_1, \dots, \widetilde{T}_N$ such that for each $l$, $\widetilde{T}_l: \M_l \rightarrow X_l $ is the $\infty$-OT map between $\widetilde \rho_{l}^n d \vol_{\M_l}$ and $\mu_l^n$, and
			\[   \sup _{x \in \mathcal{M}_l} d_{\mathcal{M}_l}(x, \tT_l(x)) \leq \widetilde{\delta} . \]
		\end{enumerate}
	\end{corollary}

	Each of the maps $\widetilde{T}_l$ in the above corollary induces a partition $\widetilde{U}_{1l}, \ldots, \widetilde{U}_{n_ll }$ of $\M_l$, where:
	\[
	\widetilde{U}_{il}:=\widetilde{T}_{l}^{-1}\left(\left\{x_{il}\right\}\right).
	\]
	For each $l=1, \dots, N,$ a (local) \textit{discretization} map $\widetilde{P}_l: L^{2}(\mu_l) \rightarrow L^{2}\left(\mu_{l}^n\right)$ is defined as
	\begin{equation}
	(\widetilde{P}_l f_l)\left(x_{il}\right):=n_l \cdot \int_{\widetilde{U}_{il}} f(x) \widetilde{\rho}_{l}^n(x) d\vol_{\M_l} (x), \quad f_l \in L^{2}(\mu_l),
	\end{equation}
	and an associated (local) \textit{extension} map $\widetilde{P}_l^{*}: L^{2}\left(\mu_{l}^n\right) \rightarrow L^{2}\left(\widetilde{\mu}_{l}^n\right)$ defined as
	\[\widetilde{P}_l^{*} u=u \circ \widetilde{T}_l .\]
	The (global) discretization map $P: L^2(\mu) \rightarrow L^2(\mu^n)$ can now be defined according to
	\[  P f : = (  P_1f_1, \dots, P_N f_N      ) \]
	where $f=(f_1, \dots, f_N) \in L^2(\mu)$. In other words, $P$ acts on $f$ according to the coordinatewise action of the $P_l$ on the $f_l$. Likewise, we may define $\widetilde{P}^*: L^2(\mu^n) \rightarrow L^2(\mu) $ according to:
	\[ P^* u = ( P^*_1 u_1 , \dots, P^*_N u_N).  \]
	
	\medskip
	
	We now introduce the interpolation map $\mathcal{I}: L^2(\mu^n) \rightarrow L^2(\mu)$. This map takes the form $\mathcal{I} = \Lambda \widetilde{P}^* $, i.e. it is the composition of the extension map $\widetilde{P}^*$ and a smoothening operator that acts coordinatewise. The smoothening operator is chosen conveniently so as to make the error in the first inequality in \eqref{eqn:InformalDirichletIneq} as small as possible; the first work to our knowledge that attempted to do something similar when analyzing graph Laplacians is \citet{BIK}. To conduct the analysis in our setting we must introduce new constructions and prove new results given the annular geometry of the kernel used to build the data graph.


	Let  $\eta:[0,\infty) \rightarrow\R$  and $\psi:[0, \infty) \rightarrow[0, \infty)$ be the functions given by
	\begin{equation}
	\label{eqn:PsiEta}
	\eta(t) :=  \begin{cases}  1 \quad  0 \leq t \leq 1 \\ 0 \quad t >1, \end{cases}    \quad \psi(t):=\frac{1}{\sigma_{\eta}} \int_{t}^{\infty} \eta(s) s d s,
	\end{equation}
	where recall $\sigma_\eta$ was defined in \eqref{eqn:sigmaeta}.
	
	For every $r_1,r_2$ such that $r_1>r_2$ we  define the function:
	\[
	\K^l_{r_1,r_2}(x, y):=\left(\frac{r_1^2}{r_1^{m+2}-r_2^{m+2}}\psi\left(\frac{d_{\M_l}(x,y)}{r_1}\right)-\frac{r_2^2}{r_1^{m+2}-r_2^{m+2}}\psi\left(\frac{d_{\M_l}(x,y)}{r_2}\right)\right), \quad x,y \in \M_l
	\]
	which serves as ``kernel" and induces the convolution operator: 
	\[
	\Lambda_{r_1,r_2}^l f(x):=\frac{1}{\tau_l(x)} \int_{\mathcal{M}_l} \K^l_{r_1,r_2}(x, y) f_l(y) d \vol_{\M_l}(y),
	\]
	which acts on functions $f_l: \M_l \rightarrow\R$. In the above, $\tau_l(x)$ is a normalization factor given by
	\[
	\tau_l(x):=\int_{\mathcal{M}_l} \K^l_{r_1,r_2}(x, y)  d \vol_{\M_l}(y);
	\]
	notice that $\K_{r_1, r_2}^l$ is non-negative.
	\nc
	
	We can put together the action of each convolution operator on each of the manifolds and define: 
	\[ \Lambda_{r_1, r_2} f := (  \Lambda_{r_1, r_2}^1 f_1 , \dots,  \Lambda_{r_1, r_2}^N f_N   ), \quad  f=(f_1, \dots, f_N) \in L^2(\mu). \]
	Our global \textit{interpolation operator}  takes the form:
	\[   \mathcal{I}  u  : =  \Lambda  \widetilde P^*u, \quad u \in L^2(\mu^n);   \]
	In the remainder it will be convenient to write the above in coordinates as:
	\[ \mathcal{I} u = ( \mathcal{I}_1 u_1, \dots, \mathcal{I}_N u_N). \]
	%
	%
	%
	%
	%

	Having defined the maps $P$ and $\I$ we are now ready to state precisely the connection between the Dirichlet energies $D$ and $b$.
	%
	%
	
	\begin{proposition}[Inequality for Dirichlet energies]\label{Proposition: Inequality for Dirichlet energies}
		Let $\varepsilon_+,\varepsilon_-, \widetilde{\delta},$ and $\theta$ be fixed but small enough numbers satisfying Assumptions \ref{assumption: main}. Let $b$ be the Dirichlet energy associated to the weighted graph $(X, \omega)$ defined in \eqref{eqn:GraphDirichlet} and $D$ the Dirichlet energy  defined in \eqref{Equ: dirichlet energy}.
		
		Then, with probability greater than $1-\sum_{l=1}^N (nw_l+t) \exp \left(-\mathrm{C}(nw_l-t) \theta^{2} \widetilde{\delta}^{m}\right)-2N\exp \left(\frac{-2 t^{2}}{n}\right)-C_1(n)$,
		we have:
		
		\begin{enumerate}[(1)]
			\item For any $f \in L^{2}(\mu)$,
			$$
			\sigma_\eta D(\widetilde{\mathcal{I}} u) \leq\left(1+C(\veps_+ +\frac{\widetilde{\delta}}{\veps_+}+\theta+\widetilde{\delta} )\right)b^{\veps_+,\veps_-}(u)
			$$
			
			\item For any $f \in L^{2}\left(\mu^n\right)$,
			$$
			\begin{aligned}
			b_I^{\varepsilon_+,\varepsilon_-}(\widetilde{P} f)& \leq  \left(1+C(\veps_+ +\frac{\widetilde{\delta} }{\veps_+}+\theta+\widetilde{\delta})\right)\sigma_\eta D(f)
			\end{aligned}
			$$
			In addition, if $f$ is in the span of $\Delta_\M$'s eigenfunctions with corresponding eigenvalue less than $\lambda$, then: 
			\[   b_O^{\veps_+, \veps_-}(\widetilde{P} f) \leq \frac{CN N_0}{w_{min}^2n^2(\veps_+^{m+2}-\veps_-^{m+2})}\left(1+\lambda^{m / 2+2}\right)\left\|f\right\|_{L^{2}(\mathcal{M})}^{2}  \]
			
		\end{enumerate}
	\end{proposition}
	We recall that the quantity $N_0$ was introduced in section \ref{Section:Conditions on algorithms for the recovery of tensorized spectrum from data clouds} in Definition \ref{Outer Partially Connected} and it represents the largest number of connections in the graph $(X,\omega)$ between two distinct manifolds.  The following estimates complement Proposition \ref{Proposition: Inequality for Dirichlet energies} and essentially state that the maps $\mathcal{I}$ and $P$ are almost isometries.
	
	
	\begin{proposition}[Discretization and interpolation maps are almost isometries]\label{Proposition: Discretization and interpolation maps are almost isometries}
		Let $\varepsilon_+,\varepsilon_-, \tilde{\delta},$ and $\theta$ be fixed but small enough numbers satisfying Assumptions \ref{assumption: main}. Then, with probability at least $1-\sum_{l=1}^N (nw_l+t) \exp \left(-\mathrm{C}(nw_l-t) w_l^{2} \widetilde{\delta}^{m}\right)-2N\exp \left(\frac{-2 t^{2}}{n}\right)-C_1(n)$, we have:
		\begin{enumerate}[(1)]
			\item For every $f \in L^{2}(\mu)$
			\[
			\left|\|f\|_{L^{2}(\mu)}^{2}-\|{P} f\|_{L^{2}\left(\mu^n\right)}^{2}\right| \leq C \widetilde{\delta}\|f\|_{L^{2}(\mu)}     \sqrt{D(f)}+C(\theta+\widetilde{\delta})\|f\|_{L^{2}(\mu)}^{2}
			\]
			\item For every $u \in L^{2}\left(\mu^n\right)$
			\[
			\left|\|u\|_{L^{2}\left(\mu^n\right)}^{2}-\|{\mathcal{I}} u\|_{L^{2}(\mu)}^{2}\right| \leq C \veps_+\|u\|_{L^{2}\left(\mu^n\right)} \sqrt{b^{\veps_+,\veps_-}(u)}+C(\theta+\widetilde{\delta})\|u\|_{L^{2}\left(\mu^n\right)}^{2}
			\]
		\end{enumerate}
	\end{proposition}

	\subsection{Preliminary local energy estimates}

	In order to prove the above results we first establish a sequence of preliminary estimates on each of the individual manifolds $\M_l$. The results presented in this subsection are independent of the fact that all manifolds forming our model \eqref{eqn:M} have the same dimension or not.

	
	%
	%
	%
	%

	\begin{lemma}\label{lemma: restrict for tau}
		Suppose $0<r_2<\frac{1}{4}r_1$ are small enough, in particular smaller than half the injectivity radius of the manifold $\M_l$. Then, there exists an absolute constant $C>0$ such that
		\[(1+Cm_lK_lr_1^2)^{-1}\leq \tau_l(x)\leq 1+Cm_lK_lr_1^2, \quad \text{ and } \quad  |\nabla\tau_l(x)|\leq \frac{Cm_lK_lr_1}{\sigma_\eta},\]
		for all $x\in \mathcal{M}_l$. Here $K_l$ is a uniform bound on the absolute value of sectional curvatures.
	\end{lemma}
	
	\begin{proof}
		First, notice that: 
		\begin{equation*}
		\begin{aligned}
		\tau_l(x)&=\frac{r_1^2}{r_1^{m_l+2}-r_2^{m_l+2}} \int_{\mathcal{M}_l} \psi\left(\frac{d_{\M_l}(x, y)}{r_1}\right) d\vol_{\M_l}(y)-\frac{r_2^2}{r_1^{m_l+2}-r_2^{m_l+2}} \int_{\mathcal{M}_l} \psi\left(\frac{d_{\M_l}(x, y)}{r_2}\right) d\vol_{\M_l}(y)\\
		&=\frac{r_1^2}{r_1^{m_l+2}-r_2^{m_l+2}}\int_{B_{m_l}(0,r_1)}\psi\left(\frac{|v|}{r_1}\right)J_x(v)dv-\frac{r_2^2}{r_1^{m_l+2}-r_2^{m_l+2}}\int_{B_{m_l}(0,r_2)}\psi\left(\frac{|v|}{r_2}\right)J_x(v)dv,
		\end{aligned}
		\end{equation*}
		where in the above $J_{x}$ denotes the Jacobian of the exponential map $\exp_x: B_{m_l}(0, \iota_l) \rightarrow B_{\M_l}(x,\iota_l) $.  Using a standard estimate for the Jacobian, namely
		\begin{equation}
		(1+ Cm_l K_l |v|^2)^{-1}  \leq   J_x(v) \leq 1+ C m_l K_l |v|^2, \quad \forall  v \in B_{m_l}(0, \iota_l/2) , 
		\label{eqn:JacobianBound}
		\end{equation}
		we can see that $\tau_{l}(x)$ satisfies   $(1+Cm_lK_lr_1^2)^{-1} C_\alpha \leq \tau_l(x)\leq (1+Cm_lK_lr_1^2)C_\alpha $ for some constant $C$, and for 
		\[  C_\alpha:= \frac{r_1^2}{r_1^{m_l+2}-r_2^{m_l+2}}\int_{B_{m_l}(0,r_1)}\psi\left(\frac{|v|}{r_1}\right)dv-\frac{r_2^2}{r_1^{m_l+2}-r_2^{m_l+2}}\int_{B_{m_l}(0,r_2)}\psi\left(\frac{|v|}{r_2}\right)dv. \]
		\nc
		A direct computation using polar coordinates and integration by parts reveals that $C_\alpha$ is actually equal to one. This establishes the first assertion.
		
		To obtain the estimate for the gradient of $\tau_{l}(x)$, we notice that from the the definition of $\psi$ in \eqref{eqn:PsiEta}, the chain rule, and the fact that $\nabla  d_{\M_l}(\cdot, y)(x) =-\frac{1}{d_{\M_l}(x,y)} \exp_{x}^{-1}(y) $, it follows:
		\[
		\begin{aligned}
		|\nabla\tau_l(x)|&=\frac{1}{\sigma_{\eta}(r_1^{m_l+2}-r_2^{m_l+2})}\left|\int_{B_{\M_l}(x,r_1)}\left(\eta\left(\frac{d_{\M_l}(x,y)}{r_1}\right)-\eta \left(\frac{d_{\M_l}(x,y)}{r_2}\right)\right)\exp_x^{-1}(y)d\vol_{\M_l}(y)\right|\\
		&=\frac{1}{\sigma_{\eta}(r_1^{m_l+2}-r_2^{m_l+2})}\left|\int_{B_{m_l}(r_1)}\eta\left(\frac{|v|}{r_1}\right)vJ_x(v)dv-\int_{B_{m_l}(r_1)}\eta\left(\frac{|v|}{r_2}\right)v J_x(v)dv\right|\\
		&\leq\frac{Cm_l K_lr_1^2}{\sigma_{\eta}(r_1^{m_l+2}-r_2^{m_l+2})}\left(\int_{B_{m_l}(r_1)}\eta\left(\frac{|v|}{r_1}\right)|v|dv+\int_{B_{m_l}(r_2)}\eta\left(\frac{|v|}{r_2}\right)|v|dv\right)\\
		&\leq\frac{Cm_l K_lr_1^2(r_1^{m_l+1}+r_2^{m_l+1})}{\sigma_{\eta}(r_1^{m_l+2}-r_2^{m_l+2})}\leq \frac{Cm_l K_lr_1}{\sigma_{\eta}}.
		\end{aligned}
		\]
		Notice that in the first inequality we have used \eqref{eqn:JacobianBound} and the radial symmetry of the integrands (which induces a cancellation). In the last step we used $0<r_2\le \frac{1}{4}r_1$.
	\end{proof}

	The next definitions is used in the subsequent lemmas. For every $l=1, \dots, N$ we define
	\begin{equation}
	\label{eqn:NonLocalContinuumDirich}
	\begin{split}
	\widetilde{D}_{NL,l}^{\veps_+,\veps_-}(f_l):=\frac{1}{\veps_+^{m+2}-\veps_-^{m+2}}\int_{\mathcal{M}_l} \int_{\mathcal{M}_l}  &\left[\eta\left(\frac{d_{\mathcal{M}_l}(x,y)}{\veps_+}\right)-\eta\left(\frac{d_{\mathcal{M}_l}(x,y)}{\veps_-}\right)\right](f_l(x)-f_l(y))^{2} 
	\\&\widetilde{\rho}_{l}^n(x) \widetilde{\rho}_{l}^n(y) d \vol_{\M_l}(x) d\vol_{\M_l} (y),
	\end{split}
	\end{equation}
	where we recall that the densities  $\widetilde{\rho}_{l}^n$ are defined in Corollary \ref{cor:Densities}.  We also consider:
	\begin{equation}
	E^{r}_l(f_l):=\int_{\mathcal{M}_l} \int_{\mathcal{M}_l}  \eta\left(\frac{d_{\mathcal{M}}(x, y)}{r}\right)(f_l(x)-f_l(y))^{2} \widetilde{\rho}_{l}^n(x) \widetilde{\rho}_{l}^n(y) d\vol_{\M_l} (x) d\vol_{\M_l}(y), \quad f_l \in L^{2}(\mu_l)
	\label{eqn:DirichletEnerSpheresAnnular}
	\end{equation}
	%
	Notice that for every $r_1>r_2>0$ we have:
	\begin{equation}\label{equ: D and E}
	(r_1^{m+2}-r_2^{m+2})D_{NL,l}^{r_1,r_2}(f_l)= E_l^{r_1}(f_l)- E_l^{r_2}(f_l), \quad f_l \in L^2(\mu_l).
	\end{equation}
	\begin{lemma}\label{lemma: restrict E by using D}
		Suppose $\veps_+,\veps_-$ satisfy Assumptions \ref{assumption: main}. Then, there exists a universal
		constant $C>0$ such that for every $0<\veps_-<\frac{1}{4}\veps_+$ and every $f_l \in L^{2}(\mu_l)$
		\[
		\frac{1}{\veps_+^{m+2}-\veps_-^{m+2}}E_{l}^{\veps_+}(f_l) \leq C \left(1+c_\rho L_{\rho_l}\right) D_{NL,l}^{\veps_+,\veps_-}(f_l),
		\]   
		where $L_{\rho_l}$ is a constant depending on $\rho_l$.
	\end{lemma}
	\begin{proof}
		
		The proof is very similar to the one in Lemma 4 in \citet{trillos2019error}. In that Lemma the idea is to cover a larger ball with smaller balls and use the triangle inequality. Here the only difference is that we want to cover a larger ball with a collection of annuli.

		\nc
	\end{proof}
	\begin{lemma}\label{lemma: D(lambda) and D }
		Suppose that $\veps_+,\veps_-$ satisfy Assumptions \ref{assumption: main}. Then, there exists a universal constant $C>0$ such that
		\[
		\begin{aligned}
		D_l(\Lambda^l_{\veps_+,\veps_-}f_l)\leq (1+c_\rho L_{\rho_l}\veps_+)\left[1+Cm_lK_l\veps_+^2\left(1+\frac{\sqrt{1+c_\rho L_p}}{\sigma_{\eta}}\right)\right]\frac{1}{\sigma_{\eta}}D_{NL,l}^{\veps_+,\veps_-}(f_l), \quad \forall f_l \in L^2(\M_l , \rho_l).
		\end{aligned}
		\]
		We recall that $D_l$ was defined in \eqref{Equ: dirichlet energyLocal} and $D_{NL,l}^{\veps_+,\veps_-}$ in  \eqref{eqn:NonLocalContinuumDirich}. 
	\end{lemma}
	
	\begin{proof}
		We can write $\nabla(\Lambda^l_{\veps_+,\veps_-}f_l)$ as
		$$
		\nabla(\Lambda^l_{\veps_+,\veps_-}f_l)=\frac{1}{\tau_l(x)}A_1^l(x)+A_2^l(x),
		$$
		where 
		$$
		A_1^l(x):=\int_{R_{\mathcal{M}_l}(x,\veps_+,\veps_-)}\nabla \K^l_{\veps_+,\veps_-}(\cdot,y)(x)\left(f_l(y)-f_l(x)\right) d\vol_{\M_l}(y)
		$$
		and
		$$
		A_2^l(x)=\nabla(\frac{1}{\tau_l(x)})\int_{R_{\mathcal{M}_l}(x,\veps_+,\veps_-)}\K^l_{\veps_+,\veps_-}(x,y)\left(f_l(y)-f_l(x)\right) d\vol_{\M_l}(y);
		$$
		here $R_{\M_l}(x,\veps_+, \veps_-):=\{  \tilde x \in \M_l \: : \: \veps_{-} < d_{\M_l}(x, \tilde x) < \veps_+  \}$.

		We find a bound for $|A_1^l(x)|^2$; notice that $\frac{1}{\tau_l(x)} \leq 1+Cm_lK_l\veps_+^2 $ by \eqref{lemma: restrict for tau}. First, notice that:
		$$
		\begin{aligned}
		\nabla \K^l_{\veps_+,\veps_-}(\cdot,y)(x)&=\nabla \left[\frac{\veps_+^2}{\veps_+^{m+2}-\veps_-^{m+2}}\psi\left(\frac{d_{\M_l}(x,y)}{\veps_+}\right)-\frac{\veps_-^2}{\veps_+^{m+2}-\veps_-^{m+2}}\psi\left(\frac{d_{\M_l}(x,y)}{\veps_-}\right)\right]\\
		&=-\frac{1}{\veps_+^{m+2}-\veps_-^{m+2}}\left(\veps_+\psi'\left(\frac{d_{\M_l}(x,y)}{\veps_+}\right)-\veps_-\psi'\left(\frac{d_{\M_l}(x,y)}{\veps_-}\right)\right)\frac{\exp_x^{-1}(y)}{d_{\M_l}(x,y)}\\
		&=\frac{\exp_x^{-1}(y)}{\sigma_{\eta}(\veps_+^{m+2}-\veps_-^{m+2})}\left[\eta\left(\frac{d_{\M_l}(x,y)}{\veps_+}\right)-\eta\left(\frac{d_{\M_l}(x,y)}{\veps_-}\right)\right].
		\end{aligned}
		$$
		Since for $A_1^l(x)$ we have $|A_1^l(x)|=\langle A_1^l(x),w\rangle$ for some unit vector $w\in T_x\mathcal{M}_l$, we can combine with the inequality above to obtain:
		$$
		\begin{aligned}
		|A_1^l(x)|&=\langle A_1^l(x),w\rangle\\
		&=\frac{1}{\sigma_{\eta}(\veps_+^{m+2}-\veps_-^{m+2})}\int_{R_{\mathcal{M}_l}(x,\veps_+,\veps_-))}\left[\eta\left(\frac{d(x,y)}{\veps_+}\right)-\eta\left(\frac{d(x,y)}{\veps_-}\right)\right]
		\\& \ \ \ \ \ \cdot \left(f_l(y)-f_l(x)\right) \langle \exp_x^{-1}(y),w\rangle d\vol_{\M_l}(y)\\
		&=\frac{1}{\sigma_{\eta}(\veps_+^{m+2}-\veps_-^{m+2})}\int_{R(\veps_+,\veps_-)}\left[\eta\left(\frac{|u|}{\veps_+}\right)-\eta\left(\frac{|u|}{\veps_-}\right)\right]\phi(u)\langle u,w\rangle J_x(u)du,
		\end{aligned}
		$$
		where $\phi(u)=f_l\left(\exp_x(u)\right)-f_l(x)$ and $R(\veps_+ , \veps_-) := \{ u \in \R^{m_l} \: : \: \veps_- < |u| < \veps_+ \} $. By the Cauchy-Schwartz inequality,
		
		\begin{align}
		\label{eqn:AuxKernel}
		\begin{split}
		|A_1^l(x)|^2 & \leq \frac{1}{\sigma_{\eta}^2(\veps_+^{m+2}-\veps_-^{m+2})^2}\int_{R(\veps_+,\veps_-)}|\phi(u)|^2J_x(u)^2\left[\eta\left(\frac{|u|}{\veps_+}\right)-\eta\left(\frac{|u|}{\veps_-}\right)\right]du \\
		& \ \ \ \ \ \cdot \int_{R(\veps_+,\veps_-)}\langle u,w \rangle ^2\left[\eta\left(\frac{|u|}{\veps_+}\right)-\eta\left(\frac{|u|}{\veps_-}\right)\right]du\\
		&  =\frac{1}{\sigma_{\eta}(\veps_+^{m+2}-\veps_-^{m+2})}\int_{R(\veps_+,\veps_-)}|\phi(u)|^2J_x(u)^2\left[\eta\left(\frac{|u|}{\veps_+}\right)-\eta\left(\frac{|u|}{\veps_-}\right)\right]du
		\\& \leq\frac{1+Cm_l K_l\veps_+^2}{\sigma_{\eta}(\veps_+^{m+2}-\veps_-^{m+2})}\int_{\mathcal{M}_l}\left[\eta\left(\frac{d_{\M_l}(x,y)}{\veps_+}\right)-\eta\left(\frac{d_{\M_l}(x,y)}{\veps_-}\right)\right]\left(f_l(y)-f_l(x)\right)^2 d\vol_{\M_l}(y),
		\end{split}
		\end{align}
		where the equality comes from
		$$
		\int_{R(\veps_+,\veps_-)}\langle u,w \rangle ^2\left[\eta\left(\frac{|u|}{\veps_+}\right)-\eta\left(\frac{|u|}{\veps_-}\right)\right]du=\int_{B_m(\veps_+)}\langle u,w\rangle ^2du-\int_{B_m(\veps_-)}\langle u,w\rangle ^2du=\sigma_{\eta}(\veps_+^{m+2}-\veps_-^{m+2}),
		$$
		and the last inequality from the bound \eqref{eqn:JacobianBound}.
		Integrating \eqref{eqn:AuxKernel} against $\rho_l^2d \vol_{\M_l}$ and using the Lipschitz continuity of $\rho_l$ we deduce
		$$
		\begin{aligned}
		\left \lVert \frac{A_1^l}{\tau_l} \right \rVert^2_{L^2(\mathcal{M}_l,\rho_l^2\vol_{\M_l})}  & \leq  \frac{(1+Cm_l K_l\veps_+^2)(1+c_\rho L_{\rho_l}\veps_+)}{\sigma_{\eta}(\veps_+^{m+2}-\veps_-^{m+2})}
		\\& \cdot  \int_{\mathcal{M}_l}\int_{B_{\mathcal{M}_l}(x,\veps_+,\veps_-)}\left[\eta\left(\frac{d_{\M_l}(x,y)}{\veps_+}\right)-\eta\left(\frac{d_{\M_l}(x,y)}{\veps_-}\right)\right]|f_l(y)-f_l(x)|^2d\mu_l(x)d\mu_l(y)\\
		&\leq \frac{(1+Cm_l K_l\veps_+^2)(1+c_\rho L_{\rho_l}\veps_+)}{\sigma_{\eta}} D_{NL,l}^{\veps_+,\veps_-}(f_l).
		\end{aligned}
		$$

		Now we analyze the term $A_{2}^l(x)$. First, recall that $|\nabla(\tau_l^{-1})|\leq \frac{Cm_l K_l\veps_+}{\sigma_{\eta}}$ and  $\tau_l \leq 1+Cm_l K_l\veps_+^2$ by Lemma \ref{lemma: restrict for tau}.
		%
		%
		Using the mean value theorem, it is straightforward to show that
		\begin{equation}\label{equ: psi to eta}
		\begin{split}
		\K_{\veps_+,\veps_-}^l(x,y)
		\le \frac{\veps_+^2}{\veps_+^{m+2}-\veps_-^{m+2}}\psi(\frac{d_{\M_l}(x,y)}{\veps_+})\le \frac{\veps_+^2 \eta(\frac{d_{\M_l}(x,y)}{\veps_+})}{\sigma_\eta(\veps_+^{m+2}-\veps_-^{m+2})}.
		\end{split}
		\end{equation}
		
		Thus, by the Cauchy-Schwartz inequality and \eqref{equ: psi to eta}, we have
		$$
		\begin{aligned}
		|A_2^l(x)|^2&\leq |\nabla(\frac{1}{\tau_l(x)})|^2 \int_{\mathcal{M}_l}\K^l_{\veps_+,\veps_-}(x,y)d\mu(y)\int_{\mathcal{M}_l}|f_l(x)-f_l(y)|^2\K^l_{\veps_+,\veps_-}(x,y)d\mu_l(y)\\
		&\leq \frac{C^2m^2K_l^2\veps_+^2}{\sigma_{\eta}^3(\veps_+^{m+2}-\veps_-^{m+2})}\int_{\mathcal{M}_l}\veps_+^2\eta\left(\frac{d_{\M_l}(x,y)}{\veps_+}\right)|f_l(x)-f_l(y)|^2d\mu_l(y).\\
		\end{aligned}
		$$
		Integrating both sides of the above inequality with respect to $\rho_l^2 d\vol_{\M_l}$, using the Lipschitz continuity of $\rho_l$, and using Lemma \ref{lemma: restrict E by using D}, we conclude that
		$$
		\lVert A_2^l \rVert_{L^2(\mathcal{M}_l,\rho_l^2\vol_{\M_l})}\leq \frac{Cm_l K_l\veps_+^2(1+c_\rho L_{\rho_l} \veps_+)\sqrt{1+c_\rho L_{\rho_l}}}{\sigma_{\eta}}\sqrt{\frac{1}{\sigma_{\eta}}D_{NL,l}^{\veps_+,\veps_-}(f_l)},
		$$
		for some universal constant $C$. Combining the estimates for $\left \lVert \frac{A_1^l}{\tau_l} \right \rVert^2_{L^2(\mathcal{M}_l,\rho_l^2\vol_{\M_l})}$ and $\lVert A_2^l \rVert_{L^2(\mathcal{M}_l,\rho_l^2\vol_{\M_l})} $ we finally obtain:
		$$
		\begin{aligned}
		\left(D_l(\Lambda_{\veps_+,\veps_-}f_l) \right)^{1/2}\leq (1+c_\rho L_{\rho_l}\veps_+)\left[1+Cm_lK_l\veps_+^2\left(1+\frac{\sqrt{1+c_\rho L_{\rho_l}}}{\sigma_{\eta}}\right)\right]\sqrt{\frac{1}{\sigma_{\eta}}D_{NL,l}^{\veps_+,\veps_-}(f_l)}.
		\end{aligned}
		$$
	\end{proof}

	\begin{lemma}\label{lemma: D^NL(f)<= D(f)}
		Suppose $\veps_+,\veps_-$ satisfy Assumptions \ref{assumption: main}. Then, there exists a universal constant $C>0$ such that 
		\[
		D^{NL,l}_{\veps_+,\veps_-}(f_l)\leq \left(1+c_\rho L_{\rho_l}\veps_+\right)\left(1+Cm_lK_l\veps_+^2\right)\sigma_{\eta}D_l(f_l), \quad  \forall f_l \in L(\mu_l).
		\]
	\end{lemma}
	
	\begin{proof}
		By a density argument we may assume without the loss of generality that $f_l$ is smooth. Now, for every $x \in \M_l$,
		\[
		\int_{R_{\mathcal{M}_l}(x,\veps_+,\veps_-)}\lvert f_l(x)-f_l(y)\lvert ^2d\mu(y)=\int_{R(\veps_+,\veps_-)}|f_l(\exp_x(v))-f_l(x)|^2\rho_l(\exp_x(v))J_x(v)dv,
		\]
		where $R_{\mathcal{M}_l}(x,\veps_+,\veps_-)$ and $R(\veps_+,\veps_-)$ are as defined in the proof of Lemma \ref{lemma: D(lambda) and D }, and $J_x$ is the Jacobian of the exponential map at $x$. From the Fundamental Theorem of Calculus it follows that
		\[
		|f_l(\exp_x(v))-f_l(x)|^2\leq \int_0^1|\frac{d}{dt}f_l\left(\exp_x(tv)\right)|^2dt=\int_0^1|df_l(\Phi_t(x,v)_2)|^2dt,
		\]
		where $\Phi_t$ denotes the time $t$ geodesic flow on $\M_l$'s tangent bundle $\T\M_l$: that is, $\Phi_t(x,v)=(\varphi_{x,v}(t),\varphi_{x,v}'(t)) \in \T\M_l$, where $\varphi_{x,v}(t):=\exp_x(tv)$ and $\varphi_{x,v}'(t)$ is obtained by parallel transporting the vector $v \in \T_{x}\M_l$ along the geodesic connecting $x$ and $\exp_x(tv)$; $\Phi_t(x,v)_2$ denotes the second coordinate of $\Phi_t(x,v)$. We can then obtain:
		\[
		\begin{aligned}
		&\int_{\mathcal{M}_l}\int_{R(\veps_+,\veps_-)}|f_l(\exp_x(v))-f_l(x)|^2\rho_l\left(\exp_x(v)\right)dv\rho_l(x)d\vol_{\M_l}(x)\\
		&\leq \int_0^1\int_{\mathcal{M}}\int_{R(\veps_+,\veps_-)}|df_l\left(\Phi_t(x,v)\right)|^2\rho_l\left(\Phi_1(x,v)_1\right)\rho_l\left(\Phi_0(x,v)_1\right)dvd\vol_{\M_l}(x)dt,
		\end{aligned}
		\]
		where we use $\Phi_0(x,v)_1$ and $\Phi_1(x,v)_1$ to denote the first coordinates of $\Phi_0(x,v)$ and $\Phi_1(x,v)$ respectively.
		From the Lipschitz continuity of $\rho_l$ it follows $\rho_l(x)\leq(1+ c_\rho  L_{\rho_l}\veps_+)\rho(y)$ for all $x,y\in \mathcal{M}_l$ satisfying $ d(x,y)\leq \veps_+$. Combining the fact that $\Phi_t$ preserves the canonical volume form $\vol_{T\mathcal{M}_l}$ on $T\mathcal{M}_l$ and that
		\begin{equation*}
		\begin{split}
		&\mathcal{R}(\veps_+,\veps_-):=\left\{\xi=(x,v)\in T\mathcal{M}:\veps_-\leq |v| \leq \veps_+\right\};
		\quad\mathcal{B}_{r}:=\{\xi=(x, v) \in T \mathcal{M}:|v| \leq r\}
		\end{split}
		\end{equation*}
		are invariant under $\Phi_t$, we obtain
		\[
		\begin{aligned}
		&\int_{\mathcal{M}}\int_{R(\veps_+,\veps_-)}|f_l(\exp_x(v))-f_l(x)|^2\rho_l\left(\exp_x(v)\right)dv\rho_l(x)d\vol_{\M_l}(x)\\
		&\leq (1+c_\rho L_{\rho_l}\veps_+)\int_0^1\int_{\mathcal{R}(\veps_+,\veps_-)}|df_l\left(\Phi_t(\xi)_2\right)|^2\rho_l^2\left(\Phi_t(\xi_1)\right)d\vol_{T\mathcal{M}_l}(\xi)dt\\
		&=(1+c_\rho L_{\rho_l}\veps_+)\left[\int_{\mathcal{B}_{\veps_+}}|df_l(\xi_2)|^2\rho_l^2(\xi_1)d\vol_{T\mathcal{M}_l}(\xi)-\int_{\mathcal{B}_{\veps_-}}|df_l(\xi_2)|^2\rho_l^2(\xi_1)d\vol_{T\mathcal{M}_l}(\xi)\right]\\
		&=(1+c_\rho L_{\rho_l}\veps_+)\sigma_{\eta}(\veps_+^{m+2}-\veps_-^{m+2})\int_{\mathcal{M}}|\nabla f_l|^2\rho_l^2(x)d\vol_{\M_l}(x).
		\end{aligned}
		\]
		Therefore,
		\begin{equation*}
		\begin{aligned}
		D_{NL,l}^{\veps_+,\veps_-}(f_l)&\leq \left(1+Cm_lK_l\veps_+^2\right)\cdot
		\\&\cdot\frac{1}{\veps_+^{m+2}-\veps_-^{m+2}}\int_{\mathcal{M}_l}\int_{R(\veps_+,\veps_-)}|f_l(\exp_x(v))-f_l(x)|^2\rho_l\left(\exp_x(v)\right)dv\rho(x)d\vol_{\M_l}(x)\\
		&\leq \left(1+Cm_lK_l\veps_+^2\right)\cdot(1+c_\rho L_{\rho_l}\veps_+)\sigma_{\eta}D_l(f_l).
		\end{aligned}
		\end{equation*}
	\end{proof}

	The following is an adaptation of Lemma 14 in \citet{trillos2019error} to the kernel with annular geometry that we consider in this paper.
	
	\begin{lemma}\label{lemma: D^NL and b}
		Suppose $\widetilde{\delta},\veps_+,\veps_-$ satisfy Assumptions \ref{assumption: main}. \blue Then, with probability at least $1-\sum_{l=1}^N (nw_l+t) \exp \left(-\mathrm{C}(nw_l-t) \theta^{2} \widetilde{\delta}^{m}\right)-2N\exp \left(\frac{-2 t^{2}}{n}\right) $ and for a universal constant $C>0$, the following statements hold: \nc
		\begin{enumerate}[(1)]
			\item For every $u_l: \mathcal{X}_n^l \rightarrow \R$ we have
			\[
			D_{NL,l}^{\veps_+^\prime,\veps_-^\prime}(P^*_l u_l) \leq \left(1+C(\theta+\widetilde{\delta})\right) \left(1+C(\frac{\widetilde{\delta}}{\veps_+}+\veps_-^2)\right)b^{\veps_+,\veps_-}_l(u_l),
			\]
			where  $\veps_+^\prime:=\veps_+ - 2\widetilde{\delta}$ and $\veps_-^\prime:=\veps_- + 2\widetilde{\delta} + \frac{8\veps_-^{\prime 3}}{R^2}$. 
			\blue Applying Assumption 2, $C\frac{\widetilde{\delta}\veps_+^{m+1}+\veps_-^{m+4}}{\veps_+^{m+2}-\veps_-^{m+2}}$ can be simplified to $C_1\frac{\widetilde{\delta}}{\veps_+}+C_2\veps_-^2$, where $C_1\le \frac{16C}{15},C_2\le \frac{C}{15}$; or just $C(\frac{\widetilde{\delta}}{\veps_+}+\veps_-^2)$. \nc
		\end{enumerate}
		\begin{enumerate}[(2)]
			\item For every $f_l \in H^{1}(\mathcal{M}_l)$
			\begin{equation}
			b_l^{\veps_+,\veps_-}(P_lf_l)\leq\left(1+C(\theta+\widetilde{\delta})\right) \left(1+C(\frac{\widetilde{\delta}}{\veps_+}+\veps_-^2)\right)D^{\veps_+'',\veps_-''}_{NL,l}(f_l), 
			\end{equation}
			where $\veps_+'':=\veps_++\frac{8\veps_+''^3}{R^2}+2\widetilde{\delta}$ and $ \veps_-'':=\veps_- -2\widetilde{\delta}$. 
			We recall that $b^l$ was introduced in \eqref{eqn:InnerOuterDirichlet}.
			
		\end{enumerate}
	\end{lemma}
	\begin{proof}
		
		We first recall a well known relation between the geodesic distance in $\M_l$ and the Euclidean distance in the ambient space $\R^d$. Namely,
		\begin{equation}
		\lvert x-y\lvert \leq d_{\M_l}(x, y) \leq\lvert x-y\lvert+\frac{8}{R_l^{2}}\lvert x-y\lvert^{3}, \quad x, y \in \M_l,
		\label{Proposition: dist relation}
		\end{equation}
		where $R_l$ is the reach of the manifold $\M_l$ (see \citet{FedererReach} for a definition of reach).\nc

		To show $(1)$, notice that if $\lvert x-y\lvert<\veps_-$, then from \eqref{Proposition: dist relation} we get   
		\begin{equation}
		\begin{split}
		&\lvert x-y\lvert\leq d_{\M_l}(x,y)\leq \lvert x-y\lvert+\frac{8\veps_-^3}{R_l^2}.\\
		\end{split}
		\label{eqn:AuxDistance}
		\end{equation}
		
		We now use the map $\tT_l$, the density $\widetilde{\rho}_{l}^n$ from Corollary \ref{cor:Densities}, and the induced partition $\{U_{1l}, \dots, U_{n_l l} \}$ on $\M_l$ of the form $ U_{il}=\tT_l^{-1}(x_{il}) $, where $ x_{il}\in X_l $, to write
		\begin{equation*}
		\begin{split}
		n_l^2 (\veps_+^{ {m+2}}&-\veps_-^{ {m+2}}) b^{\veps_+,\veps_-}_l(u_l)\\
		&=  \sum_{i,j}\int_{U_{il}} \int_{U_{jl}} \left[\eta\left(\frac{\lvert  \tT_l(x)-\tT_l(y)\lvert  }{\veps_+}\right)-\eta\left(\frac{\lvert  \tT_l(x)-\tT_l(y)\lvert  }{\veps_-}\right)\right]
		\\
		&\ \ \ \ \  \cdot \left\lvert  \left(P^{*} u_l\right)(x)-\left(P^{*} u_l\right)(y)\right\lvert  ^{2} \widetilde{\rho}_{l}^n(x) \widetilde{\rho}_{l}^n(y) d \vol_{\M_l}(y) d \vol_{\M_l}(x)  \\
		&\geq \left(1-C(\theta+\widetilde{\delta})\right)\int_{\mathcal{M}_l} \int_{\mathcal{M}_l} \left[\eta\left(\frac{d(\tT_l(x),\tT_l(y))}{\veps_+}\right)-\eta\left(\frac{\left[d(\tT_l(x),\tT_l(y))-\frac{8\veps_-^3}{R^2}\right]_+}{\veps_-}\right)\right]
		\\& \ \ \ \ \  \cdot \left\lvert  \left(P^{*} u_l\right)(x)-\left(P^{*} u_l\right)(y)\right\lvert  ^{2} d \mu_l(y) d \mu_l(x)\\
		&\geq \left(1-C(\theta+\widetilde{\delta})\right)\int_{\mathcal{M}_l} \int_{\mathcal{M}_l} \left[\eta\left(\frac{d(x,y)+2\widetilde{\delta}}{\veps_+}\right)-\eta\left(\frac{\left[d(x,y)-\frac{8\veps_-^{3}}{R^2}-2\widetilde{\delta}\right]_+}{\veps_-}\right)\right]
		\\& \ \ \ \ \  \cdot \left\lvert  \left(P^{*} u_l\right)(x)-\left(P^{*} u_l\right)(y)\right\lvert  ^{2} d \mu_l(y) d \mu_l(x)\\
		& \ge \left(1-C(\theta+\widetilde{\delta})\right)(\veps_+^{\prime m+2}-\veps_-^{\prime m+2}) D^{\veps^\prime_+,\veps^\prime_-}_{NL,l}(P_l^*u_l),
		\end{split}
		\end{equation*}
		where in the first inequality we use i) in Corollary \ref{cor:Densities} and \eqref{eqn:AuxDistance}, and in the second inequality we use ii) in Corollary \eqref{cor:Densities}. Combining the above inequality with Assumptions \ref{assumption: main} we conclude that
		$$
		\left(1+C(\theta+\widetilde{\delta})\right)(1+C\frac{\widetilde{\delta}\veps_+^{m+1}+\veps_-^{m+4}}{\veps_+^{m+2}-\veps_-^{m+2}})b^{\veps_+,\veps_-}_l(u_l)\geq  D_{NL,l}^{\veps_+^\prime,\veps_-^\prime}(P^* u_l).
		$$

		For (2), we proceed similarly as in the proof of (1) to deduce
		\begin{equation*}
		\begin{split}
		b_l^{\veps_+,\veps_-}(Pf_l)&\le\frac{1+C(\theta+\widetilde{\delta})}{\sigma_{\eta}(\veps_+^{m+2}-\veps_-^{m+2})} \sum_{i} \sum_{j} \int_{U_{il}} \int_{U_{jl}} \left[\eta\left(\frac{\lvert  \tT(x)-\tT(y)\lvert  }{\veps_+}\right)-\eta\left(\frac{\lvert  \tT(x)- \tT(y)\lvert  }{\veps_-}\right)\right]
		\\& \ \ \ \ \ \cdot \lvert  f_l(y)-f_l(x)\lvert  ^{2} d \mu_l(y) d \mu_l(x)\\
		&\leq \frac{1+C(\theta+\widetilde{\delta})}{\sigma_{\eta}(\veps_+^{m+2}-\veps_-^{m+2})} \int_{\mathcal{M}_l}\int_{\mathcal{M}_l}\left[\eta\left(\frac{\left[d(x,y)-\frac{8\veps_+^3}{R^2}-2\widetilde{\delta}\right]_+}{\veps_+}\right)-\eta\left(\frac{d(x,y)+2\widetilde{\delta}}{\veps_-}\right)\right]
		\\& \ \ \ \ \ \cdot \lvert  f_l(y)-f_l(x)\lvert  ^{2} d \mu_l(y) d \mu_l(x)\\
		&\leq \frac{1+C(\theta+\widetilde{\delta})}{\sigma_{\eta}(\veps_+^{m+2}-\veps_-^{m+2})}\left[E_l^{\veps_+''}(f_l)-E_l^{\veps_-''}(f_l)\right]
		\\& \leq \left(1+C(\theta+\widetilde{\delta})\right)\left(1+C\frac{\veps_-^{m+4}+\veps_+^{m+1}\widetilde{\delta}}{\veps_+^{m+2}-\veps_-^{m+2}}\right)D^{\veps_+'',\veps_-''}_{NL,l}(f_l),
		\end{split}
		\end{equation*}
		where in the last line we have used \eqref{eqn:DirichletEnerSpheresAnnular} and Assumptions \ref{assumption: main}.
	\end{proof}

	We are ready to prove Proposition \ref{Proposition: Inequality for Dirichlet energies}.

	\subsection{Proofs of Propositions \ref{Proposition: Inequality for Dirichlet energies} and \ref{Proposition: Discretization and interpolation maps are almost isometries}}

	\begin{proof}[Proof of Proposition \ref{Proposition: Inequality for Dirichlet energies}]

		\textbf{(1):} Let $u \in L^2(\mu^n)$. We write $u$ in coordinates as $u=(u_1, \dots, u_N)$. We combine Lemmas  \ref{lemma: D(lambda) and D } and \ref{lemma: D^NL and b} to obtain for every $l=1, \dots, N$:
		$$
		\begin{aligned}
		&\sigma_\eta  D_l(\I_l u_l) \leq  \left(1+C(\theta+\widetilde{\delta})\right)(1+C\frac{\widetilde{\delta}}{\veps_+}) (1+c_\rho L_{\rho_l}\veps_+)\left[1+Cm_l K_l\veps_+^2\left(1+\frac{\sqrt{1+c_\rho L_\rho}}{\sigma_{\eta}}\right)\right]b_l^{\veps_+,\veps_-}(u_l).
		\end{aligned}
		$$
		From the above we deduce that $ D(\mathcal{I} u  ) = \sum_{l=1}^N w_l^2 D_l( \mathcal{I}_l u_l )$ is smaller than:
		\begin{align*}
		\left(1+C(\theta+\widetilde{\delta})\right)(1+C\frac{\widetilde{\delta}}{\veps_+}) (1+c_\rho L_{\rho}\veps_+)\left[1+Cm K\veps_+^2\left(1+\frac{\sqrt{1+c_\rho L_\rho}}{\sigma_{\eta}}\right)\right] \sum_{l=1}^N w_l^2  b_l^{\veps_+,\veps_-}(u_l).
		\end{align*}
		In turn, Proposition \ref{Proposition: Tail Inequality} implies that with probability at least \blue $1-2N\exp \left(\frac{-2 t^{2}}{n}\right)$ \nc we have
		\[  \sum_{l=1}^N w_l^2  b_l^{\veps_+,\veps_-}(u_l) \leq (1+ t)\sum_{l=1}^N \left( \frac{n_l}{n} \right)^2 b_{l}^{\veps_+, \veps_-}(u_l)  =(1+t) b^{\veps_+,\veps_-}_I(u) \leq (1+t) b^{\veps_+,\veps_-}(u). \]
		Putting together the above inequalities we obtain the desired estimate. Here it is worth highlighting that the last inequality in the above expression comes from the fact that the discrete Dirichlet energy $b^{\veps_+, \veps_-}$ is the sum of  $b_I^{\veps_+, \veps_-}$ and $b_O^{\veps_+, \veps_-}$. As we will see below, in order to obtain a reverse inequality between $b^{\veps_+, \veps_-}$ and $D$ one needs to control $b^{\veps_+, \veps_-}_O(Pf)$ using regularity estimates of $f$ in each of the $\M_l$. We will be able to obtain this control when $f$ is in the span of the eigenfunctions of $\Delta_\M$ smaller than a certain value (which is all we need in the remainder). 
		
		%
		
		%
		
		\textbf{(2):} Similarly to \textbf{(1)}, we may combine Lemma \ref{lemma: D^NL(f)<= D(f)} and Lemma \ref{lemma: D^NL and b}, to deduce:
		\begin{equation}\label{inequ: bound b^I by D(f)}
		\begin{split}
		b_I^{\veps_+,\veps_-}(P f) &\leq \left(1+c_\rho L_{\rho_l}\veps_+''\right)\left(1+Cm_l K_l\veps_+''^2\right)\left(1+C(\frac{\veps_+^{m+4}+\veps_+^{m+1}\widetilde{\delta}}{\veps_+^{m+2}-\veps_-^{m+2}}+\theta+\widetilde{\delta})\right)\sigma_\eta D(f)\\
		&\leq \left(1+C(\veps_+''+\frac{\widetilde{\delta}}{\veps_+}+\theta+\widetilde{\delta}) + \blue t \nc \right)\sigma_\eta D(f).
		\end{split}
		\end{equation}
		
		where the last step we used Assumption \ref{assumption: main}.

		Let $f \in L^2(\M)$ belong to the span of $\Delta_\M$'s eigenfunctions with corresponding eigenvalue less than $\lambda$. Then, $f$ can be written as $f=\sum_{l=1}^N  f^l$ where each $f^l$ has support on $\M_l$, and where, abusing notation slightly, each $f^l$ has the form $f^l =\sum_{q}b_{ql} f_{q}^l $ for an orthonormal basis of eigenfunctions of $\Delta_{\M_l}$, $\{ f_q^l \}$, with corresponding eigenvalues smaller than $\lambda$. It is straightforward to see that: 
		\begin{equation}
		\label{inequ: omega^O bounded by r/epsilon_+}
		\begin{split}
		b_O^{\veps_+,\veps_-}(\widetilde Pf)&=\frac{1}{n^2(\veps_+^{m+2}-\veps_-^{m+2})}\sum_{x_i,x_j\in\X_n }\omega^O
		_{x_i x_j}( \widetilde Pf(x_i)- \widetilde Pf(x_j))^2
		\\ &\leq \frac{2}{\veps_+^{m+2}-\veps_-^{m+2}}\sum_{x_i\in\X_n}\sum_{x_j\in\X_n}\omega^O
		_{x_i x_j} ( \widetilde Pf(x_i) )^2
		\\ &= \frac{2}{\veps_+^{m+2}-\veps_-^{m+2}} \sum_{l=1}^N \sum_{s: s\not = l} \sum_{x_i\in \M_l  }\sum_{x_j\in\M_s}\omega^O
		_{x_i x_j} \lvert \int_{U_{il}}  f^l(x) \widetilde{p}^n_l(x)d\vol_{\M_l}(x) \rvert^2
		\\ &= \frac{2}{\veps_+^{m+2}-\veps_-^{m+2}} \sum_{l=1}^N  \frac{\lVert   f^l \lVert  ^2_{L^\infty(\mathcal{M}_l)}}{n_l^2}\sum_{s: s\not = l} \sum_{x_i\in \M_l  }\sum_{x_j\in\M_s}\omega^O
		_{x_i x_j}
		\\ &= \frac{2N\cdot N_0}{\veps_+^{m+2}-\veps_-^{m+2}} \sum_{l=1}^N  \frac{\lVert f^l \|^2_{L^\infty(\mathcal{M}_l)}}{n_l^2}
		\\&\leq \frac{2N\cdot N_0(1+t)}{w^2_{min} n^2(\veps_+^{m+2}-\veps_-^{m+2})} \sum_{l=1}^N \| f^l \|^2_{L^\infty(\mathcal{M}_l)}. 
		\end{split}
		\end{equation}
		In the above, the last inequality follows with high probability according to  Proposition \eqref{Proposition: Tail Inequality}.  
		To complete the proof we find estimates for each of the terms $\| f^l \|^2_{L^\infty(\mathcal{M}_l)}$ and to do this we adapt the argument in Lemma 3.3 of \citet{Lu2019GraphAT}. From standard higher order elliptic regularity results (e.g. Theorem 2 in  \citet{EvansBook}) it follows that for every $s\in \N$
		\begin{center}
			$\begin{aligned}
			\left\|f^{l}\right\|_{H^{2 s}(\mathcal{M}_l)}^2 & \leq C(\M_l , \rho_l, s)\left(\left\|\Delta_{\M_l}^{s} f^{l}\right\|_{L^{2}(\mathcal{M}_l)}^2+\left\|f^l\right\|^2_{L^{2}(\mathcal{M}_l)}\right) 
			\end{aligned}$,
		\end{center}
		where in the above $H^{2s}(\M_l)$ is the Sobolev space of functions on $\M_l$ with square-integrable derivatives of order up to $2s$;  it is at this stage that we use the smoothness of the manifold $\M_l$ and the density $\rho_l$.  Moreover, by the Sobolev embedding theorem  on compact manifolds (e.g. Theorem 2.20 in \citet{Aubin1998}), we have $H^{2 s}(\mathcal{M}_k) \subset C^{1}(\mathcal{M}_k)$ as long as $2 s>m / 2+1 .$ Choosing $2 s=m / 2 +2$ we obtain:
		\begin{equation}
		\label{eqn:Linfty}
		\left\|f^l\right\|_{L^{\infty}(\mathcal{M}_l)}   \leq   \left\|f^l\right\|_{C^{1}(\mathcal{M}_l)} \leq C\left(\M_l, \rho_l \right)\left\|f^l\right\|_{H^{m / 2+2}(\mathcal{M}_l)} \leq  C\left(\M_l, \rho_l \right) \left(\lambda^{m / 4+1}+1\right)\left\|f^l\right\|_{L^{2}(\mathcal{M}_l)}.\end{equation}
		Recalling that $f^l$ has the form $\sum_{q}b_{ql} f_{q}^l $, where the $f_{q}^l$ are orthonormal in $L^2(\M_l , \rho_l)$ and are eigenfunctions of $\Delta_{\M_l}$ with eigenvalues $\lambda_{q}^l$ smaller than $\lambda$,  we can see that
		\begin{align*}
		\left\|\Delta_{\M_l}^{s} f^{l}\right\|_{L^{2}(\mathcal{M}_l,\rho_l)}^2 & = \left\| \sum_q b_{ql} (\lambda_q^l)^s f_q^{l}\right\|_{L^{2}(\mathcal{M}_l,\rho_l)}^2
		\\& = \sum_q b_{ql}^2 (\lambda_q^l)^{2s} \left\|  f_q^{l}\right\|_{L^{2}(\mathcal{M}_l,\rho_l)}^2 \leq \lambda^{2s}\lVert f^l \rVert_{L^2(\M_l)}^2. 
		\end{align*}
		\nc
		%
		%
		%
		Putting the above estimates together and combining with \eqref{inequ: omega^O bounded by r/epsilon_+} gives us the desired result.

	\end{proof}

	Before proving Proposition \ref{Proposition: Discretization and interpolation maps are almost isometries} we need one last preliminary estimate.

	\begin{lemma}\label{lemma: tau f controlled by f}
		Suppose $\veps_+,\veps_-$ satisfy Assumptions \ref{assumption: main}. Then, there exists a universal constant $C>0$ such that
		$$
		\lVert \Lambda_{\veps_+,\veps_-}f\lVert ^2_{L^2(\mathcal{M},\rho)}\leq (1+C c_\rho L_{\rho}\veps_+)(1+Cm K\veps_+^2)\lVert f\lVert ^2_{L^2(\mathcal{M},\rho)},
		$$
		and 
		$$
		\lVert \Lambda_{\veps_+,\veps_-}f-f\lVert ^2_{L^2(\mathcal{M},\rho)}\leq \frac{Cc_\rho^2\veps_+^2}{\sigma_{\eta}}\sum_{l=1}^Nw_l D^{\veps_+,\veps_-}_{NL,l}(f_l) \leq \frac{Cc_\rho^2\veps_+^2}{\sigma_{\eta} w_{\min}} D(f) .
		$$
		for all $f\in L^2(\mathcal{M},\rho)$. In the above, $w_{\min}:= \min_{l=1, \dots, N} w_l $.
	\end{lemma}
	\begin{proof}
		
		Since $\Lambda_{\veps_+, \veps_-}$ acts on $f$ coordinatewise, we get	
		\begin{align*}
		\begin{split}
		\int_{\M} (\Lambda_{\veps_+, \veps_-} f(x) )^2 d\mu(x)  &= \sum_{l=1}^{N}  w_l \int_{\M_l} (\Lambda^l_{\veps_+, \veps_-} f_l(x) )^2 \rho_l(x) d\vol_{\M_l}(x)\\ &\leq \sum_{l=1}^N  w_l \int_{\mathcal{M}_l}\int_{\mathcal{M}_l}\frac{\K^l_{\veps_+,\veps_-}(x,y)}{\tau_l(x)}(f_l(y))^2\rho_l(x)d\vol_{\M_l}(y)d\vol_{\M_l}(x)  
		\\& \leq (1+ C c_\rho L_\rho \veps_+)(1+ C m K \veps_+^2 )  
		\sum_{l=1}^N  w_l \int_{\M_l} (f_l(y) )^2 \rho_l(y) d\vol_{\M_l}(y) 
		\\&= (1+ C c_\rho L_\rho \veps_+)(1+ C m K \veps_+^2 )  \int_{\M}(f(x))^2 d\mu(x), 
		\end{split}
		\end{align*}	
		where the first inequality follows from Jensen's inequality, and the second inequality follows from Lemma \ref{lemma: restrict for tau} and the properties of the density functions $\rho_l$.

		%
		For the second inequality, we first calculate the difference between $\Lambda_{\veps_+,\veps_-}^l f_l(x)$ and $f_l(x)$:
		$$
		\begin{aligned}
		|\Lambda^{\veps_+,\veps_-}_l f_l (x)-f_l(x)|^2&= \left(\frac{1}{\tau_l(x)}\int_{\mathcal{M}_l}\K^l_{\veps_+,\veps_-}(x,y)(f_l(y)-f_l(x))d\mu_l(y) \right)^2\\
		&\leq \frac{1}{\tau_l(x)^2}\int_{\mathcal{M}_l}\K^l_{\veps_+,\veps_-}(x,y)d\mu_l(y)\int_{\mathcal{M}_l}\K^l_{\veps_+,\veps_-}(x,y)(f_l(x)-f_l(y))^2d\mu_l(y)\\
		&=\frac{1}{\tau_l(x)}\int_{\mathcal{M}_l}\K^l_{\veps_+,\veps_-}(x,y)(f_l(x)-f_l(y))^2d\mu_l(y).
		\end{aligned}
		$$
		Then we integrate with respect to $\rho_l(x)d \vol_{\M_l}(x)$ to get:
		$$
		\begin{aligned}
		\lVert \Lambda^{\veps_+,\veps_-}_l f_l-f_l\lVert ^2_{L^2(\mathcal{M}_l,\rho_l)}
		&\leq (1+Cm_l K_l\veps_+^2)\int_{\mathcal{M}_l}\int_{\mathcal{M}_l}\K^l_{\veps_+,\veps_-}(x,y)(f(x)-f(y))^2d\mu_l(y)d\mu_l(x)\\
		&\leq \frac{(1+Cm_l K_l\veps_+^2)\veps_+^2}{\sigma_{\eta}(\veps_+^{m+2}-\veps_-^{m+2})}\int_{\mathcal{M}_l}\int_{\mathcal{M}_l}\eta\left(\frac{d(x,y)}{\veps_+}\right)(f(x)-f(y))^2d\mu_l(y)d\mu_l(x)\\
		&\leq \frac{Cc_\rho^2\veps_+^2}{\sigma_{\eta}}D^{\veps_+,\veps_-}_{NL,l}(f_l),
		\end{aligned}
		$$
		where the second inequality follows from the fact that $\eta \leq \frac{1}{\sigma_\eta}\psi$ (recall \eqref{eqn:PsiEta}), and the third inequality follows from Lemma \ref{lemma: restrict E by using D}.  Multiplying the above by $ w_l$, adding over $l$, and using Lemma \ref{lemma: D^NL(f)<= D(f)} we get the desired result.
	\end{proof}

	\begin{proof}[Proof of Proposition \ref{Proposition: Discretization and interpolation maps are almost isometries}] 
		\textbf{(1):} 
		For each $l=1, \dots, N$, we use estimates proved in  \citet{calder2019improved} (appearing in Pages 24-25 in the proof of Proposition 4.2) to conclude that there is a constant $C$ for which
		\begin{equation*}
		\begin{aligned}
		&\left| \blue  \lVert f_l\rVert^2_{L^2(\widetilde{\mu}_l^n)} \nc-\lVert f_l\lVert ^2_{L^2(\mu_l)}\right|\leq C(\theta+\widetilde{\delta})\lVert f_l\lVert ^2_{L^2(\mu^l)}\\
		&\left\lVert |\widetilde{P_l}f_l\lVert ^2_{L^2(\mu^n_l)}-\lVert f_l\lVert ^2_{L^2(\mu^l)}\right|\leq C\lVert f_l\lVert _{L^2(\mu^l)}\lVert \widetilde{P_l}^*\widetilde{P_l}f_l-f_l\lVert _{L^2(\widetilde{\mu}^l_n)}+C(\theta+\widetilde{\delta})\lVert f_l\lVert ^2_{L^2(\mu^l)}\\
		&\lVert \widetilde{P_l}^*\widetilde{P_l}f_l-f_l\lVert ^2_{L^2(\widetilde{\mu}^l_n)}\leq C\widetilde{\delta}^2D_l(f_l),
		\end{aligned}
		\end{equation*}
		for every $f_l \in L^2(\mu_l)$; in the above we use $\widetilde{\mu}_n^l$ to denote the measure $\widetilde{\rho}_n^l d\vol_{\M_l}(x)$. Combining the previous inequalities, we deduce that:
		$$
		\begin{aligned}
		\left\lVert |\widetilde{P}f\lVert ^2_{L^2(\mu)}-\lVert f\lVert ^2_{L^2(\mu)}\right|&\leq\sum_{l=1}^Nw_l\left\lVert |\widetilde{P_l}f_l\lVert ^2_{L^2(\mu^n_l)}-\lVert f_l\lVert ^2_{L^2(\mu^l)}\right|\\
		&\leq C\sum_{l=1}^Nw_l\lVert f_l\lVert _{L^2(\mu^l)}\lVert \widetilde{P_l}^*\widetilde{P_l}f_l-f_l\lVert _{L^2(\widetilde{\mu}^l_n)}+C(\theta+\widetilde{\delta})\lVert f\lVert ^2_{L^2(\mu)}.
		\end{aligned}
		$$
		Now, the first term in the last inequality above is controlled by $C\widetilde{\delta}\lVert f\lVert _{L^2(\mu)}\sqrt{D(f)}$. Indeed, this follows from Cauchy-Schwartz inequality:
		$$
		\begin{aligned}
		\left(\sum_{l=1}^Nw_l\lVert f_l\lVert _{L^2(\mu^l)}\lVert \widetilde{P_l}^*\widetilde{P_l}f_l-f_l\lVert _{L^2(\widetilde{\mu}^l_n)}\right)^2
		&\leq \left(\sum_{l=1}^N w_l\lVert f_l\lVert ^2_{L^2(\mu^l)}\right)\left(\sum_{l=1}^Nw_l\lVert \widetilde{P_l}^*\widetilde{P_l}f_l-f_l\lVert ^2_{L^2(\widetilde{\mu}^l_n)}\right)\\
		&\leq C\frac{\widetilde{\delta}^2}{\blue w_{\min} \nc} \lVert f\lVert _{L^2(\mu)}^2D(f).
		\end{aligned}
		$$
		Putting things together we finally deduce
		$$
		\left\lVert |\widetilde{P}f\lVert ^2_{L^2(\mu^n)}-\lVert f\lVert ^2_{L^2(\mu)}\right|\leq C\frac{\widetilde{\delta}}{\blue  \sqrt{w_{\min}} \nc}\lVert f\lVert _{L^2(\mu)}\sqrt{D(f)}+C(\theta+\widetilde{\delta})\lVert f\lVert ^2_{L^2(\mu)}.
		$$
		\textbf{(2):} From the identity $\lVert u_l\lVert _{L^2(\mu^n_l)}=\lVert \widetilde{P_l}^*u_l\lVert _{L^2(\widetilde{\mu}^n_l)}$ (which follows automatically from the fact that the map $\widetilde{T}_l$ is a transport map between $\widetilde{\mu}_l^n$ and $\mu_l^n$) and the triangle inequality we get
		$$
		\begin{aligned}
		\left\lVert |{\I_l}u_l\lVert _{L^2(\widetilde{\mu}^n_l)}-\lVert u_l\lVert _{L^2(\mu^n_l)}\right|&\leq\lVert \Lambda^{\veps_+,\veps_-}_l\widetilde{P}_l^*u_l-\widetilde{P}_l^*u_l\lVert _{L^2(\widetilde{\mu}^n_l)}\\
		&\leq \left(1+ c_\rho\lVert \rho_l-\widetilde{\rho}^n_l\lVert _{L^\infty(\M_l)} \right)\cdot\lVert \Lambda^{\veps_+,\veps_-}_l\widetilde{P}_l^*u_l-\widetilde{P}_l^*u_l\lVert _{L^2(\mu_l)}\\
		&\leq \left(1+ c_\rho\lVert \rho_l-\widetilde{\rho}^n_l\lVert _{L^\infty(\M_l)}\right)\cdot C\veps_+\sqrt{D^{\veps_+,\veps_-}_{NL,l}(\widetilde{P}^*_lu_l)}\\
		& \blue \leq C\veps_+\sqrt{b_l^{\veps_+,\veps_-}(u_l)}, \nc
		\end{aligned}
		$$
		where the third inequality comes from Lemma \ref{lemma: tau f controlled by f} and the last one follows from Lemma \ref{lemma: D^NL and b}. Also, notice that 
		$$
		\lVert \I_l u_l\lVert _{L^2(\widetilde{\mu}_l^n)}=\lVert \Lambda_l^{\veps_+,\veps_-}\widetilde{P}^*_lu_l\lVert _{L^2(\widetilde{\mu}_l^n)}\leq C\lVert \widetilde{P}^*_lu_l\lVert _{L^2(\widetilde{\mu}_l^n)}=C\lVert u_l\lVert _{L^2(\mu^n_l)}.
		$$
		This inequality is also a consequence of Lemma \ref{lemma: tau f controlled by f}. So far we have proved that
		$$
		\begin{aligned}
		\left\lVert |\I_lu_l\lVert ^2_{L^2(\widetilde{\mu}^n_l)}-\lVert u_l\lVert ^2_{L^2(\mu^n_l)}\right|
		\leq  C\veps_+\sqrt{b_l^{\veps_+,\veps_-}(u)}\lVert u_l\lVert _{L^2(\mu^n_l)}.
		\end{aligned}
		$$
		Next, we compare $\lVert \I_l u_l\lVert ^2_{L^2(\widetilde{\mu}^n_l)}$ and $\lVert \I_lu_l\lVert ^2_{L^2(\mu_l)}$, bounding their difference with
		$$
		\begin{aligned}
		\left\lVert |\widetilde{\I_l}u_l\lVert ^2_{L^2(\widetilde{\mu}_l^n)}-\lVert \widetilde{\I_l}u_l\lVert ^2_{L^2(\mu_l)}\right|&\leq C(\theta+\widetilde{\delta})\lVert \widetilde{\I_l}u_l\lVert ^2_{L^2(\widetilde{\mu}_l^n)}\leq C\blue c_\rho \nc (\theta+\widetilde{\delta})\lVert u_l\lVert ^2_{L^2(\mu_l^n)},
		\end{aligned}
		$$
		as it follows from the fact that the difference between $\rho_l$ and $\widetilde{\rho}_{l}^n$ is uniformly controlled with very high probability, i.e.  i) in Corollary \ref{cor:Densities}.
		%
		%
		%
		
		Finally, we obtain
		$$
		\begin{aligned}
		\left\lVert |u_l\lVert ^2_{L^2(\mu^n_l)}-\lVert \widetilde{\I_l}u_l\lVert ^2_{L^2(\mu_l)}\right|&\leq \left\lVert |\widetilde{\I_l}u_l\lVert ^2_{L^2(\widetilde{\mu}_l^n)}-\lVert u_l\lVert ^2_{L^2(\mu^n_l)}\right|+\left\lVert |\widetilde{\I_l}u_l\lVert ^2_{L^2(\widetilde{\mu}_l^n)}-\lVert \widetilde{\I_l}u_l\lVert ^2_{L^2(\mu_l)}\right|\\
		&\leq C\veps_+\sqrt{b_l^{\veps_+,\veps_-}(u_l)}\lVert u_l\lVert _{L^2(\mu^n_l)}+Cc_\rho(\theta+\widetilde{\delta})\lVert u_l\lVert ^2_{L^2(\mu^n_l)}.
		\end{aligned}
		$$
		Adding over all $l=1,\dots,N$ and using Cauchy-Schwarz inequality we obtain the desired estimate:
		$$
		\begin{aligned}
		\left\lVert |u\lVert ^2_{L^2(\mu^n)}-\lVert \widetilde{\I}u\lVert ^2_{L^2(\mu)}\right|
		&= \left|\sum_{l=1}^N\lVert u_l\lVert ^2_{L^2(\mu^n_l)}-\sum_{l=1}^N w_l\lVert \widetilde{\I_l}u_l\lVert ^2_{L^2(\mu_l)}\right|\\
		&\leq C\veps_+\sum_{l=1}^N w_l\sqrt{b_l^{\veps_+,\veps_-}(u)}\lVert u_l\lVert _{L^2(\mu^n_l)}+Cc_\rho (\theta+\widetilde{\delta})\sum_{l=1}^N w_l\lVert u_l\lVert ^2_{L^2(\mu_l^n)}\\
		&\leq C\veps_+\lVert u\lVert _{L^2(\mu^n)}\sqrt{b^{\veps_+,\veps_-}(u)}+Cc_\rho(\theta+\widetilde{\delta})\lVert u\lVert ^2_{L^2(\mu^n)}.
		\end{aligned}
		$$
	\end{proof}

	\subsection{Proof of Theorem \ref{Rate of convergence for eigenvalues}}
	\label{Proof:ThmEigen}
	\begin{proof} [Proof of Theorem \ref{Rate of convergence for eigenvalues}]
		With the aid of Propositions \ref{Proposition: Inequality for Dirichlet energies} and \ref{Proposition: Discretization and interpolation maps are almost isometries} we can now compare $\lambda^{\veps_+,\veps_-}_k$ and $\lambda_k$, the $k$-th eigenvalues of $\mathcal{L}$ and $\Delta_\M$ (listed according to multiplicity) respectively.
		
		First, to find an upper bound for $\lambda_k^{\veps_+, \veps_-}$ in terms of $\lambda_k$, let $f^1,\dots,f^k$ be an orthonormal set (w.r.t. $L^2(\mu)$) consisting of eigenfunctions of $\Delta_{\M}$ corresponding to its first $k$ eigenvalues (and let us label them $\lambda_1 \leq \dots \leq \lambda_k$). Let 
		$$
		v_i:=\widetilde{P}f_i, \forall i=1,\dots,k.
		$$
		Applying Proposition \ref{Proposition: Discretization and interpolation maps are almost isometries} to every $f$ of the form 
		$$
		f:=f_i-f_j,
		$$
		we deduce that
		$$
		|\langle f_i,f_j\rangle_{L^2(\mu)}-\langle v_i,v_j\rangle_{L^2(\mu^n)}|\leq C\widetilde{\delta}\sqrt{\lambda_k}+C(\theta+\widetilde{\delta}) < \frac{1}{k}.
		$$
		We can then conclude that $v_1,\dots, v_k$ are linearly independent and that the subspace $S:=Span\{v_1,\dots, v_k\}$ has dimension $k$. From \eqref{minmax principle for graph laplacian} we deduce that
		$$
		\lambda_k^{\veps_+,\veps_-}\leq \max_{v\in S, \lVert v\rVert_{L^2(\mu^n)}=1}b^{\veps_+,\veps_-}(v).
		$$
		For $v \in S$, written as $v=\sum_{i=1}^ka_iv_i=\sum_{i=1}^ka_i\widetilde{P}f_i$, we can write $v:=\widetilde{P}f$ where $f= \sum_{i=1}^ka_i f_i$. This $f$ satisfies:
		$$
		D(f) = \langle \Delta_\M  f , f \rangle_{L^2(\mu)} \leq \lambda_k\lVert f\lVert ^2_{L^2(\mu)}
		$$
		according to the spectral decomposition of $\Delta_\M$. Applying part (2) of Proposition \eqref{Proposition: Inequality for Dirichlet energies} we obtain: 
		$$
		\begin{aligned}
		&b^{\veps_+,\veps_-}(v)\\
		&\leq \frac{CN N_0}{w_{min}^2n^2(\veps_+^{m+2}-\veps_-^{m+2})}\left(1+\lambda_{k}^{m / 2+2}\right)\left\|f\right\|_{L^{2}(\mathcal{M})}^{2} +\left(1+C(\veps_+''+\frac{\widetilde{\delta}}{\veps_+})\right)\sigma_\eta D(f)\\
		&\leq \frac{CN N_0}{w_{min}^2n^2(\veps_+^{m+2}-\veps_-^{m+2})}\left(1+\lambda_{k}^{m / 2+2}\right)\left\|f\right\|_{L^{2}(\mathcal{M})}^{2} +\left(1+C(\veps_+''+\frac{\widetilde{\delta}}{\veps_+})\right) \lambda_k\sigma_\eta \cdot\lVert f\lVert ^2_{L^2(\mu)}.
		\end{aligned}
		$$
		Finally, from Proposition \ref{Proposition: Discretization and interpolation maps are almost isometries} applied to a $v \in S$ with norm one, we deduce that $v$'s corresponding $f$ satisfies:
		$$
		\lVert f\lVert ^2_{L^2(\mu)}\leq 1+C(\widetilde{\delta}\sqrt{\lambda_k}+\theta+\widetilde{\delta}).
		$$
		From this we conclude that
		$$
		\begin{aligned}
		\frac{1}{\sigma_\eta }\lambda_k^{\veps_+,\veps_-}&\leq \frac{CNN_0}{w_{min}^2n^2(\veps_+^{m+2}-\veps_-^{m+2})}\left( 1+C'(\lambda_k^{m/2+2}+\widetilde{\delta}\sqrt{\lambda_k}+\theta+\widetilde{\delta}) \right)\\
		&+\left(1+C(\veps_+''+\widetilde{\delta}\sqrt{\lambda_k}+\theta+\frac{\widetilde{\delta}}{\veps_+})\right)\lambda_k.
		\end{aligned}
		$$
		This establishes the upper bound for $\lambda_k^{\veps_+,\veps_-}$ in terms of $\lambda_k$.
		
		For the lower bound, we follow completely analogous arguments as the ones above, relating functions $u \in L^2(\mu^n)$ with $f \in L^2(\mu)$ via the map $\mathcal{I}$ and applying Propositions \ref{Proposition: Inequality for Dirichlet energies} and \ref{Proposition: Discretization and interpolation maps are almost isometries}.
		
	\end{proof}
	%
	%

	\subsection{Proof of Theorem \ref{convergence rate for eigenvectors}}
	\label{Proof:Eignvector}

	\begin{proof}[Proof of Theorem \ref{convergence rate for eigenvectors}]
		We use an energy estimate based on Proposition \ref{Proposition: Inequality for Dirichlet energies} to find a relationship between eigenvectors of $\mathcal{L}^{\veps_+,\veps_-}$ and eigenfunctions of $\L_{\M}$. We follow a similar strategy to the one in \citet{trillos2019error}.

		Let $\lambda$ be an eigenvalue of $\Delta_{\M}$ and let $k\in \N$ be the first integer for which $\lambda=\lambda_k$ (here $\lambda_k$ is as in \eqref{minmax principle for laplacian}). Let $l$ be the multiplicity of $\lambda$ so that  $\lambda= \lambda_k = \dots = \lambda_{k+ l-1}  < \lambda_{k+l} $. The gap $\gamma_{\lambda}$ associated to $\lambda$ is given by:
		\begin{equation}
		\gamma_{\lambda}:=\frac{1}{2}\min\{|\lambda-\lambda_{k-1}|,|\lambda-\lambda_{k+l}|\}
		\end{equation}
		if $\lambda>0$, and $\gamma_{\lambda}:= \lambda_{l+1} = \lambda_{N+1}$ otherwise.

		Now, we can pick $\veps_+,\veps_-,\theta,\widetilde{\delta}$ to be small enough so that 
		$$
		e+C\left(\veps_++\theta+\widetilde{\delta}\right)\lambda\leq \gamma_{\lambda}
		$$
		Then, for these choices of parameters, we know from Theorem \ref{Rate of convergence for eigenvalues} that
		\begin{equation}
		|\lambda^{\veps_+,\veps_-}-\sigma_\eta \lambda|\leq \gamma_{\lambda}
		\end{equation}

		Let $S$ be a subspace of $L^2(\mu^n)$ spanned by all eigenvectors of $\mathcal{L}^{\veps_+,\veps_-}$ with eigenvalues \[\lambda^{\veps_+,\veps_-}_k,\cdots,\lambda_{k+l-1}^{\veps_+,\veps_-},\] and let us denote the orthogonal projection onto $S$ as $P_S$, the orthogonal projection onto the span of the eigenvectors of $\mathcal{L}$ with eigenvalue strictly smaller than $\lambda_{k}^{\veps_+, \veps_-}$ as  $P_{S_-}$, and the orthogonal projection onto the span of the eigenvectors of $\mathcal{L}$ with eigenvalue strictly larger than $\lambda_{k+l-1}^{\veps_+, \veps_-}$ as  $P_{S_+}$.

		Let $f$ be a normalized (w.r.t $L^2(\mu)$) eigenfunction of $\L_{\M}$ with eigenvalue $\lambda$ and let $u=\widetilde{P}f$. Notice that we can assume without the loss of generality that $f$ takes the form in \eqref{eqn:FormEigen} for one of the manifolds $\M_k$ (in particular the support of $f$ is $\M_k$). Based on Proposition \ref{Proposition: Inequality for Dirichlet energies} and its proof (specifically the bound \eqref{eqn:Linfty}) we have:
		\begin{equation}\label{inequality with Ps}
		\begin{aligned}
		 e+\sigma_\eta\left[1+C\left(\veps_++\theta+\widetilde{\delta}\right)\right]\lambda&\geq  e+\sigma_\eta\left[1+C\left(\veps_++\theta+\widetilde{\delta}\right)\right]D(f)\\
		&\geq b^{\veps_+,\veps_-}(u)=\langle \mathcal{L}^{\veps_+,\veps_-}u,u\rangle\\
		&\geq \lambda_k^{\veps_+,\veps_-}\lVert P_Su\lVert ^2_{L^2(\mu^n)}+\lambda^{\veps_+,\veps_-}_{k+l}\lVert P_{S_+} u\lVert ^2_{L^2(\mu^n)}\\
		&\geq \lambda_k^{\veps_+,\veps_-}\left(\lVert u\lVert ^2_{L^2(\mu^n)}-\lVert u-P_Su\lVert ^2_{L^2(\mu^n)}\right)
		\\&+\lambda_{k+l}^{\veps_+,\veps_-} \left(\lVert u-P_Su\lVert ^2_{L^2(\mu^n)} - \lVert P_{S_-} u \rVert_{L^2(\mu^n)}^2 \right).
		\end{aligned}
		\end{equation}

		Using the results about $\gamma_{\lambda}$ we obtained above and Proposition \ref{Proposition: Discretization and interpolation maps are almost isometries} we deduce
		\begin{equation*}
		\begin{split}
		&|\sigma_\eta\lambda_2-\lambda_2^{\veps_+,\veps_-}|\leq e+C\sigma_\eta \left(\veps_++\theta+\widetilde{\delta}\right)\lambda\leq \gamma_{\lambda};\quad |1-\lVert u\lVert ^2_{L^2(\mu^n)}|\leq C(\theta+\widetilde{\delta})
		\end{split}
		\end{equation*}

		Here $C$ is some constant may correspond to $\lambda$. Combining the above inequalities with \eqref{inequality with Ps}, we obtain:
		$$
		\begin{aligned}
		&e+\sigma_\eta\left[1+C\sigma_\eta \left(\veps_++\theta+\widetilde{\delta}\right)\right]\lambda\\
		&\geq \sigma_\eta\lambda + (\lambda_2^{\veps_+,\veps_-}-\sigma_\eta\lambda_2)+\lambda_2^{\veps_+,\veps_-}(\lVert u\lVert ^2_{L^2(\mu^n)}-1)+(\lambda_{k+2}^{\veps_+,\veps_-}-\lambda_2^{\veps_+,\veps_-})\lVert u-P_Su\lVert ^2_{L^2(\mu^n)}
		\\& -\lambda_{k+l}^{\veps_+,\veps_-} \lVert P_{S_-} u  \rVert^2_{L^2(\mu^n)}
		\\ &\geq \sigma_\eta\lambda -e-C\left(\veps_+ +\theta+\widetilde{\delta}\right)+2\gamma_{\lambda}\lVert u-P_Su\lVert ^2_{L^2(\mu^n)}-\lambda_{k+l}^{\veps_+,\veps_-} \lVert P_{S_-} u  \rVert^2_{L^2(\mu^n)}.
		\end{aligned}
		$$
		From this and the upper bound for $\lambda_{k+l}^{\veps_-, \veps_+}$ in terms of $\lambda_{k+l}$:
		$$
		\lVert u-P_Su\lVert _{L^2(\mu^n)}\leq \left[\frac{e}{\gamma_{\lambda}}+\frac{C}{\gamma_{\lambda}}(\veps_+ +\theta+\widetilde{\delta})\right]^{1/2}  + \sqrt{\lambda_{k+l}}\lVert P_{S_-} u \rVert_{L^2(\mu^n)}.
		$$
		We now compare the functions $u$ and $f$ at the data points $x_i$. Notice that for every data point $x_i \in \M_k$ we have:
		$$
		|u(x_i)-f(x_i)|= \lvert \widetilde{P}f(x_i)-f(x_i) \rvert\leq n\int_{\widetilde{U}_{ik}}\vert f(x)-f(x_i) \rvert \widetilde{\rho}_n(x)d\vol_{\mathcal{M}_k}(x)\leq \lVert \nabla f\lVert _{L^{\infty}(\mu_k)}\widetilde{\delta}.
		$$
		Also, due to \eqref{eqn:Linfty} we know that $\lVert \nabla f\lVert _{L^{\infty}(\mu_k)} \leq \sqrt{w_k} C(\M_k, \rho_k) (\lambda^{m/4 + 1} + 1) \leq C(\M, \rho) (\lambda^{m/4 + 1} + 1)  $. 
		Thus,
		\[ |u(x_i)-f(x_i)|  \leq C(\M, \rho)(\lambda^{m/4 + 1}+1 )  \widetilde \delta, \quad \forall x_i \in \M_k.  \]
		Notice that on the other hand, $u(x_i)= f(x_i)=0$ for $x_i \in \M \setminus \M_k$ by definition of $\widetilde{P}$ and the fact that $f$ is zero outside of $\M_k$. We conclude that:
		$$
		\lVert u-f\lVert _{L^2(\mu^n)}\leq C_{\mathcal{M}, \rho}(\lambda^{m/4 +1} +1) \widetilde{\delta},
		$$
		and in turn
		\begin{equation}
		\lVert f-P_S\widetilde{P}f\lVert _{L^2(\mu^n)}\leq \left[\frac{e}{\gamma_{\lambda}}+\frac{C}{\gamma_{\lambda}}(\veps_++\theta+\widetilde{\delta})\right]^{1/2}+C_{\mathcal{M}, \rho}(\lambda^{m/4 +1} +1) \widetilde{\delta} + \sqrt{\lambda_{k+l}} \lVert P_{S_-} \widetilde P (f)\rVert_{L^2(\mu^n)} . 
		\label{eqn: AuxEigenvecConv}
		\end{equation}

		From this point on the idea is to use an inductive argument. We describe in detail the base case and outline the inductive step. 
		\textbf{Base Case: }When $\lambda=0$ (and $\lambda_1= \dots = \lambda_N= 0 < \lambda_{N+1}$) we have $\lVert P_{S_-} \rVert_{L^2(\mu^n)} =0$ and thus we can drop the last term in \eqref{eqn: AuxEigenvecConv}. This means that if $f_1,\cdots,f_l$ form an orthonormal basis for the space of eigenfunctions of $\L_{\M}$ with eigenvalue $\lambda$, then we can find an orthonormal set $v_1,\cdots,v_l$ spanning $S$ such that 
		$$
		\lVert f_i-v_i\lVert _{L^2(\mu^n)}\leq \left[\frac{e}{\gamma_{\lambda}}+\frac{C}{\gamma_{\lambda}}(\veps_++\theta+\widetilde{\delta})\right]^{1/2}+C(\M, \rho)\widetilde{\delta}.
		$$
		In turn, this also implies that if $u_1,\cdots,u_l$ form an orthonormal basis of $\mathcal{L}^{\veps_+,\veps_-}$ with corresponding eigenvalues $\lambda_2^{\veps_+,\veps_-},\cdots,\lambda_{l+1}^{\veps_+,\veps_-}$, then there exists an orthonormal set $\widetilde{f}_1,\cdots,\widetilde{f}_l$ for $\L_{\rho_l}$ with eigenvalue $\lambda$ satisfying the same inequality above with $f_i$ replaced with $\tilde f_i$ and $v_i$ replaced with $u_i$. 
		
		\textbf{Inductive step:} having found the desired relationship for the eigenvectors and eigenfunctions associated to the first portion of the spectrum of $\Delta_\M$, we return to \eqref{eqn: AuxEigenvecConv} and notice that by Proposition \ref{Proposition: Discretization and interpolation maps are almost isometries} we can conclude that the term $\lVert P_{S_-} \widetilde{P} f \rVert_{L^2(\mu^n)}$ is smaller than
		\[ C(\M, \rho) (  (\lambda^{(1/4)}+1)  \sqrt{\widetilde{\delta}} +\sqrt{\theta} ). \]  
		We can plug this estimate in \eqref{eqn: AuxEigenvecConv} and then proceed as in the base case to obtain the desired result. 
		
	\end{proof}
	
	\subsection{Different dimensions: Proof of Theorem \ref{thm:MixedDimensions}}
	\label{sec:ProofSiffDim}

	We start by writing the discrete Dirichlet form $b^{\veps_+, \veps_-}$ \eqref{eqn:GraphDirichlet} as the sum of three terms:
	\[ b^{\veps_+, \veps_-} (u) =  b_{max} (u_{max}) + b_{S}(u_S) +  b_O(u),  \]
	where
	\[b_{max}(v):=\frac{1}{n^2(\veps_+^{m+2}-\veps_-^{m+2})}\sum_{x_i,x_j\in\X_n \cap \M_{max} }\omega
	_{x_i x_j}(v(x_i)-v(x_j))^2, \quad v \in L^2(\X_n \cap \M_{max}), \]
	\[b_{S}(u_{S}):=\frac{1}{n^2(\veps_+^{m+2}-\veps_-^{m+2})} \sum_{k=N_{max}+1}^N \sum_{x_i,x_j\in\X_n \cap \M_{k} }\omega
	_{x_i x_j}(u_k(x_i)-u_k(x_j))^2, \quad u_S=(u_{N_{max}+1}, \dots, u_{N}),  \]
	and lastly,
	\[   b_{O}(u):= b^{\veps_+,\veps_-}(u) - b_{max}(u_{max})- b_{S}(u_S). \]
	Notice that $b_{max}$ captures all interactions between points that belong to the manifolds with the maximum dimension $m$.  For this energy we can use all the results presented in section \ref{sec:SameDim} and in particular relate it to the Dirichlet form:
	\begin{equation*}\label{Equ: dirichlet energyMax}
	D_{max}(f):= \begin{cases} \sum_{i=1}^{N_{max}} w_i^2\int_{\M_i} |\nabla f_i(x)|^2 \rho^2_i(x) d \vol_{\M_i}(x), \quad \text{ if } f \in H^1(\M_{max}) \\  +\infty, \quad \text{ if } f \in L^2(\M_{max}) \setminus H^1(\M_{max}). \end{cases}
	\end{equation*}
	The energy $b_S$, on the other hand, captures the interactions between points that are on the same manifold when this manifold is not one of the ones with the largest dimension $m$. Using \eqref{eqn:DirichletIndivManifold}, we can write $b_S$ as:
	\[ b_S(u_S) =  \sum_{k=N_{max}+1}^N \left( \frac{n_k}{n} \right)^2\cdot \left(\frac{\veps_+^{m_k+2}-\veps_-^{m_k+2}}{\veps_+^{m+2}-\veps_-^{m+2}}\right) \cdot  b_{k}(u_k).\]
	Finally, the term $b_O(u)$ accounts for all interactions between points in two different manifolds when the two manifolds are among the ones with dimension smaller than $m$, or when one of them has dimension $m$ and the other one does not. In short, $b_O$ accounts for all interactions not accounted for by the terms $b_{max}$ and $b_{S}$ and is thus a non-negative term.

	We let $\widetilde{\I}_{max}: L^2(\X_n \cap \M_{max}) \rightarrow L^2(\M_{max})$ and $P_{max}: L^2(\M_{max}) \rightarrow L^2(\X_n \cap \M_{max})$ be the maps constructed in section \ref{set up} applied to the data set $\X_n \cap \M_{max}$ and $\M_{max}$, i.e. the union of manifolds with the same dimension $m$. We also consider the following maps:
	\[  \mathcal{I}' : L^2(\X_n) \rightarrow  L^2(\X_n \cap \M_{max}) \]
	\[  \I':  u  \longmapsto  u_{max}, \]
	\[  P' : L^2(\X_n \cap \M_{max}) \rightarrow  L^2(\X_n ) \]
	\[  P':  v  \longmapsto  u=(v,0), \]
	where by $u=(v,0)$ we mean that $u$ coincides with $v$ for data points in $\M_{max}$ and $u=0$ for data points in $\M \setminus \M_{max}$.
	
	It will be convenient to introduce the norms:
	\[ \lVert u_k \rVert^2_{L^2(\X_n \cap \M_k)}:= \frac{1}{n}\sum_{x_i \in \X_n \cap \M_k} (u_k(x_i))^2, \quad  u_k \in L^2(\X_n \cap \M_k), \]
	and
	\[ \lVert v \rVert^2_{L^2(\X_n \cap \M_{max})}:= \frac{1}{n}\sum_{x_i \in \X_n \cap \M_{max}} (v(x_i))^2, \quad  v \in L^2(\X_n \cap \M_{max}),\]
	as well as the discrete Laplacians:
	\[\mathcal{L}_k u_k(x):=\frac{1}{n_k(\veps_+^{m_{k}+2}-\veps_-^{m_k+2})}\sum_{y\in  \X_n \cap \M_k }\omega_{xy}(u(x)-u(y)), \quad x \in X\cap \M_{k} , \quad  u : X \rightarrow \R. \]
	We use $\lambda_{2,k}^{\veps_+, \veps_-}$ to denote the second eigenvalue of $\mathcal{L}_k$.
	\begin{proof}[Proof of Theorem \ref{thm:MixedDimensions}]

		Following the structure of the proofs of Theorems \ref{Rate of convergence for eigenvalues} and \ref{convergence rate for eigenvectors} we see that we can obtain our desired estimates if we can obtain similar inequalities to the ones in Propositions \ref{Proposition: Inequality for Dirichlet energies} and \ref{Proposition: Discretization and interpolation maps are almost isometries} where now we use the maps $\I_{max} \circ \I'$ and $P' \circ P_{max}$ as interpolation and discretization maps respectively. There is only one small caveat in the almost isometry property of $\I_{max} \circ\I' $ as we explain below.
		
		We start by noticing that from the above definitions we have:
		\[    b_{max}(\I' u ) \leq b^{\veps_+, \veps_-}(u), \quad \forall u \in L^2(\X_n), \]
		and by Proposition \ref{Proposition: Inequality for Dirichlet energies}
		\[\sigma_\eta D_{max}({\mathcal{I}_{max}} \circ \I' u) \leq\left(1+C(\veps_+ +\frac{\widetilde{\delta}}{\veps_+}+\theta+\widetilde{\delta})\right)b_{max}(\I'u),\]
		so that
		\begin{equation}
		\sigma_\eta D_{max}({\mathcal{I}_{max}} \circ \I' u) \leq  \left(1+C(\veps_+ +\frac{\widetilde{\delta}}{\veps_+}+\theta+\widetilde{\delta} )\right)b^{\veps_+, \veps_-}(u), \quad \forall u \in L^2(\X_n). 
		\label{eqn:MixedDimen1}
		\end{equation}
		The above occurs with probability at least $1-\sum_{l=1}^N (nw_l+t) \exp \left(-\mathrm{C}(nw_l-t) \theta^{2} \widetilde{\delta}^{m}\right)-2N\exp \left(\frac{-2 t^{2}}{n}\right)-C_1(n)$ \nc.
		
		On the other hand, for arbitrary $f \in L^2(\M_{max})$ we have  
		\begin{align}
		\begin{split}
		b^{\veps_+, \veps_-} (P'\circ P_{max} f) &= b_{max}(P_{max} f) + b_{O}( P'\circ P_{max} f ) \\&  \leq  \left(1+C(\veps_+ +\frac{\widetilde{\delta} }{\veps_+}+\theta+\widetilde{\delta}\nc)\right)D_{max}(f) + b_{O}( P'\circ  {P}_{max} f ),
		\end{split}
		\label{eqn:MixedDimen2}
		\end{align}
		whereas 
		\begin{equation}
		b_{O}( P' \circ \widetilde{P}_{max} f ) \leq  \frac{CN N_0}{w_{min}^2 n^2(\veps_+^{m+2}-\veps_-^{m+2})}\left(1+\lambda^{m / 2+2}\right)\left\|f\right\|_{L^{2}(\mathcal{M}_{max})}^{2},
		\label{eqn:MixedDimen3}
		\end{equation}
		for $f$ an element in the span of $\Delta_{\M_{max}}$'s eigenfunctions with corresponding eigenvalue less than $\lambda$, as it follows from a completely analogous computation to the one in \eqref{inequ: omega^O bounded by r/epsilon_+};  this holds in the same event of very high probability where \eqref{eqn:MixedDimen1} holds.

		We consider now the norm distortion of the maps $\I\circ \I'$ and $P'\circ P_{max}$. First, notice that by definition, for $v \in L^2(\X_n \cap \M_{max})$ we have
		\begin{equation*}
		\lVert  P' v  \rVert^2_{L^2(\X_n)} = \lVert v \rVert^2 _{L^2(\X_n \cap \M_{max})},  
		\end{equation*}
		and thus combining with 1) in Proposition \eqref{Proposition: Discretization and interpolation maps are almost isometries} we obtain:
		\begin{equation}
		\left| \lVert   P' \circ P_{max} f  \rVert^2_{L^2(\X_n)} - \lVert f \rVert^2 _{L^2(\M_{max})}   \right| \leq   C \widetilde{\delta}\|f\|_{L^{2}({\M_{max}})}     \sqrt{D_{max}(f)}+C(\theta+\widetilde{\delta})\|f\|_{L^{2}(\M_{max})}^{2}.
		\label{eqn:MixedDimen4}
		\end{equation}


		Now, for a given $u \in L^2(\X_n)$ we have: 
		\[  \left \lvert  \lVert  \I' u  \rVert^2_{L^2(\X_n \cap \M_{max})} - \lVert u \rVert^2_{L^2(\X_n)} \right \rvert = \sum_{k=N_{max}+1}^N \lVert u_k   \rVert_{L^2(\X_n \cap \M_k)}^2. \]
		Also, if we let $\overline{u}_k$ represent the average of $u_k$ in $\M_k \cap \X_n$ we see that
		\[ \lVert  u_k - \overline{u}_k  \rVert^2_{L^2(\X_n \cap \M_k)} \leq \frac{1}{\lambda^{\veps_+, \veps_-}_{2,k}}   \langle \mathcal{L}_{k} u_k , u_k \rangle_{L^2(\M_k \cap \X_n)}  =  \frac{1}{\lambda^{\veps_+, \veps_-}_{2,k}} \frac{n_k^2}{n^2} b_{k}(u_k) \leq C(\M_k, w_k, \rho_k) \veps_+^{m- m_k} b(u),  \]
		for all  $k=N_{max}+1, \dots, N$, where the last inequality holds with very high probability. Indeed, notice that by Theorem \ref{Rate of convergence for eigenvalues} applied to a single manifold $\M_k$ we can find a lower bound for $\lambda_{2,k}^{\veps_+, \veps_-}$ in terms of the first non-trivial eigenvalue for $w_k \Delta_{\M_k}$. 
		We have also used the fact that $b_k(u_k) \leq (n/n_k)^2 \veps_+^{m-m_k} b(u)  $. This means that
		\begin{equation*}
		\left \lvert  \lVert  \I' u  \rVert^2_{L^2(\X_n \cap \M_{max})} - \lVert u \rVert^2_{L^2(\X_n)} \right \rvert \leq C(\M,\mu)\veps_+^{m-m_{N_{max}+1}} b^{\veps_+, \veps_-}(u) + \sum_{k=N_{max}+1}^{N} (\overline{u}_k)^2  .
		\end{equation*}
		
		Combining with Proposition \ref{Proposition: Discretization and interpolation maps are almost isometries} and using the triangle inequality we deduce that 
		\begin{align}
		\begin{split}
		\left \lvert  \lVert  \I_{max} \circ\I' u  \rVert^2_{L^2(\M_{max})} - \lVert u \rVert^2_{L^2(\X_n)} \right \rvert & \leq  C \veps_+\|u\|_{L^{2}\left(\mu^n\right)} \sqrt{b^{\veps_+,\veps_-}(u)}
		\\& +C(\theta+\widetilde{\delta})\|u\|_{L^{2}\left(\mu^n\right)}^{2} + C(\M,\mu)\sum_{k=N_{max}+1}^n\veps_+^{m-m_k} b^{\veps_+, \veps_-}(u)
		\\& + \sum_{k=N_{max}+1}^{N} (\overline{u}_k)^2.
		\end{split}
		\label{eqn:MixedDimen5}
		\end{align}
		\nc
		
		Notice that the right hand side in the above expression is small for a $u$ with low Dirichlet energy only when $u$ is close to the orthogonal complement of $\text{Span}\{ \mathds{1}_{\M_{N_{max}+1}}, \dots, \mathds{1}_{\M_{N}} \}$ (i.e. the $\overline{u}_k$ are small). Because of this, we will only be able to proceed as in the proofs of Theorems \ref{Rate of convergence for eigenvalues} and \ref{convergence rate for eigenvectors} to obtain all our estimates if first we show that the top $N$ eigenvectors of $\mathcal{L}$ are close to the indicator functions of $\M_{1}\cap \X_n, \dots,\M_N \cap \X_n $. However, this is straightforward from the following observations:
		
		\begin{enumerate}
			\item We can obtain an upper bound for the first $N$ eigenvalues of $\mathcal{L}$ following the representation \eqref{minmax principle for graph laplacian} and computing the graph Dirichlet energy of the indicator functions of the sets $\M_k \cap \X_n$. Namely, we have:
			\[ \lambda_k^{\veps_+,\veps_-} \leq \frac{CNN_0}{w_{min}^2 n^2(\veps_+^{m+2}-\veps_-^{m+2})}, \quad k=1, \dots, N. \]
			\item  Using the alternative representation:
			\[ \lambda_{N+1}^{\veps_+, \veps_-} = \max_{S\in \mathcal{G}_{N}}\min_{u\in S^\perp \backslash \{0\}}\frac{b^{\veps_+,\veps_-}(u)}{\lVert u\lVert^2_{L^2(\mu^n)}} \]
			we can obtain the lower bound
			\[ \lambda_{N+1}^{\veps_+, \veps_-} \geq  \frac{1}{2}\sigma_\eta \lambda_{N+1},  \]
			with very high probability. Indeed, taking $S= \text{Span} \{ \mathds{1}_{\M_1 \cap \X_n}, \dots, \mathds{1}_{\M_N \cap \X_n} \}$ and a unit norm $u\in S$ (in particular $\overline{u}_k=0 $ for all $k=N_{max}+1, \dots, N$) we see from \eqref{eqn:MixedDimen1} and \eqref{eqn:MixedDimen5} that 
			\[b^{\veps_+, \veps_-}(u) \geq  \sigma_\eta \lambda_{N+1}   \left(1-C(\veps_+ +\veps_+\sqrt{\lambda_{N+1}} + \veps_+^{m-m_{N_{max}+1}}\lambda_{N+1} +\theta+\frac{\widetilde{\delta}}{\veps_+})\right)\geq \frac{\sigma_\eta }{2}\lambda_{N+1}.\]

			\item Combining the previous steps we get an order one lower bound for the gap between $\lambda^{\veps_+, \veps_-}_{N}$ and $\lambda^{\veps_+, \veps_-}_{N+1}$. We can then follow the proof of Theorem \ref{convergence rate for eigenvectors} to show that there exists an orthonormal set $v^1, \dots, v^N$ consisting of eigenvectors of $\mathcal{L}$ corresponding to $\mathcal{L}$'s first $N$ eigenvalues such that 
			\[  \lVert \sqrt{\frac{n}{n_k}} \mathds{1}_{\M_k \cap \X_n} - v^k \rVert_{L^2(\X_n)}^2 \leq \frac{C(\M, \mu)N_0}{n^2(\veps_+^{m+2}-\veps_-^{m+2})}.   \]
			We deduce that if $u$ belongs to the orthogonal complement of $\text{Span} \{ v^1, \dots, v^N \}$, then 
			\[ (\overline{u}_k)^2 \leq  \frac{C(\M, \mu)N_0}{n^2(\veps_+^{m+2}-\veps_-^{m+2})},\]
			with very high probability.
		\end{enumerate}


		As discussed above, with the above estimates in hand we can now proceed as in the proofs of Theorems \ref{Rate of convergence for eigenvalues} and \ref{convergence rate for eigenvectors}. 
		
	\end{proof}

\end{document}